\documentclass[twoside]{article}

\usepackage[accepted]{aistats2023}

\usepackage[round]{natbib}

\usepackage[utf8]{inputenc} %
\usepackage[T1]{fontenc}    %
\usepackage{hyperref}       %
\usepackage{url}            %
\usepackage{booktabs}       %
\usepackage{amsfonts}       %
\usepackage{nicefrac}       %
\usepackage{microtype}      %

\usepackage{relsize}
\usepackage{scalerel}
\usepackage{float}

\usepackage{amsmath,amsthm,amssymb}

\usepackage{graphicx}
\usepackage{array}
\usepackage{tabularx}
\usepackage{makecell}
\usepackage{subcaption}
\usepackage{xcolor}
\usepackage{tikz}
\usepackage{verbatim}
\usepackage{footmisc}
\newcommand{\cmmnt}[1]{}
\usepackage{enumitem}
\usepackage{xspace} %

\usepackage{natbib}
\setcitestyle{round}

\ifx\lemma\undefined
\newtheorem{theorem}{Theorem} 
\newtheorem{lemma}[theorem]{Lemma} 
\newtheorem{proposition}[theorem]{Proposition} 

\newtheorem{corollary}[theorem]{Corollary}
\newtheorem{definition}[theorem]{Definition}

\else
\fi

\begin{document}

\twocolumn[

\aistatstitle{From Shapley Values to Generalized Additive Models and back}

\aistatsauthor{ Sebastian Bordt \And Ulrike von Luxburg }

\aistatsaddress{ Department of Computer Science \\University of Tübingen \And Department of Computer Science and Tübingen AI Center\\
University of Tübingen} ]

\begin{abstract}
In explainable machine learning, local post-hoc explanation algorithms and inherently interpretable models are often seen as competing approaches. This work offers a partial reconciliation between the two by establishing a correspondence between Shapley Values and Generalized Additive Models (GAMs). We introduce $n$-Shapley Values, a parametric family of local post-hoc explanation algorithms that explain individual predictions with interaction terms up to order $n$. By varying the parameter $n$, we obtain a sequence of explanations that covers the entire range from Shapley Values up to a uniquely determined decomposition of the function we want to explain. The relationship between $n$-Shapley Values and this decomposition offers a functionally-grounded characterization of Shapley Values, which highlights their limitations. We then show that $n$-Shapley Values, as well as the Shapley Taylor- and Faith-Shap interaction indices, recover GAMs with interaction terms up to order $n$. This implies that the original Shapely Values recover GAMs without variable interactions. Taken together, our results provide a precise characterization of Shapley Values as they are being used in explainable machine learning. They also offer a principled interpretation of partial dependence plots of Shapley Values in terms of the underlying functional decomposition. A package for the estimation of different interaction indices is available at \url{https://github.com/tml-tuebingen/nshap}.
\end{abstract}

\section{INTRODUCTION}

Local post-hoc explanation algorithms and inherently interpretable models are two of the most prominent approaches in explainable machine learning \citep{Molnar20,holzinger2022xxai}. Despite a number of arguments about their relative benefits, the differences and similarities between these two approaches remain largely unresolved \citet{Rudin19}. In the current literature, post-hoc explanations and inherently interpretable models are often framed as different concepts, with research papers,  book chapters, and tutorials divided along these lines \citep{lundberg2020local,Molnar20,neurips2022tutorial}.
We take a different perspective and highlight the similarities between post-hoc explanations and interpretable models. We do so for the particular case of Shapley Values, a prominent feature attribution method, and  GAMs, a popular class of interpretable models.
 
{\bf Post-hoc explanations with Shapley Values.} The seminal work by \citet{lundberg2017unified} introduced the SHAP feature attributions. These are based on the literature on Shapley Values in game theory.
The authors showed that for linear functions $f(x)=w^T x$ and statistically independent features, the SHAP attributions take the form  $\Phi_i=w_i (x_i-\mathbb{E}(x_i))$, thus establishing a link between the post-hoc explanation method and a very simple type of interpretable model. This work has inspired a whole branch of literature on explainable machine learning. Most relevant to us are Shapley Interaction Values \citep{lundberg2020local}, which extend Shapley Values with local interaction effects between pairs of features.

An important building block of our work is the generalization of Shapley Interaction Values towards  {\bf $n$-Shapley Values}, a novel type of Shapley-based post-hoc explanation that is able to incorporate arbitrarily many variable interactions. Similarly to the Shapley Taylor- \citep{sundararajan2020shapley} 
and the Faith-Shap interaction index \citep{tsai2022faith}, $n$-Shapley Values are a parametric family of local post-hoc explanation algorithms that explain individual predictions with interaction terms up to order $n$. As $n$ increases, the explanations become more complex and expressive and are able to faithfully explain more complex models. %

{\bf Generalized Additive Models} (GAMs hereafter) are a popular class of interpretable models with a restricted form of non-linearity \citep{hastie1990generalized,caruana2015intelligible,agarwal2021neural}. Traditionally, GAMs are allowed to exhibit (arbitrary) non-linearity in individual features, but no interaction between features is allowed. G$\text{A}^2$Ms \citep{lou2012intelligible} relax this restriction and allow for interaction between pairs of features. Conceptually, it is straightforward to extend GAMs with interaction effects of any desired order $n$ (this comes, however, at the cost of human interpretability). Important to us, the model class of GAMs suffers from an identification problem. As soon as we introduce variable interactions, the way in which a given function can be written as a GAM is no longer uniquely determined \cite{lengerich2020purifying}.

{\bf Shapley-based explanations faithfully explain GAMs.}  In this work, we show that different kinds of Shapley-based post-hoc explanations \citep{lundberg2017unified,lundberg2020local,sundararajan2020shapley,tsai2022faith} 
are completely faithful to GAMs: if the function to be explained is a GAM,  then the explanations recover its individual non-linear component functions. We link the order of the GAM -- the maximum degree of variable interaction that is present in a function -- with the order of an explanation that we use to explain that function. If the order of the explanation is at least as large as the maximum variable interaction that is (locally) present in the model, then the explanations are guaranteed to recover a faithful representation of the function as a GAM. This result applies to the newly proposed $n$-Shapley Values, as well as to the Shapley Taylor- and Faith-Shap interaction indices. As a special case, our results imply that the interventional SHAP feature attributions \citep{lundberg2017unified,janzing2020feature} are perfectly faithful to GAMs without variable interactions, even if the features are arbitrarily dependent.  

What is more, we show that Shapley-based post-hoc explanations of {\bf any function} implicitly solve the problem of representing the function as a GAM (potentially with variable interactions of very high order). This means that our results provide insights into the mechanics of Shapley Values not only if the function to be explained is a lower-order GAM, but any (learned) function, for example a neural network. Concretely, we identify a necessary and sufficient regularity condition -- subset compliance -- under which a value function gives rise to a well-defined functional decomposition of the function that we attempt to explain. 
Because this decomposition connects Shapley Values with GAMs, we term it the Shapley-GAM.

Taken together, our results offer a precise {\bf functionally-grounded analysis} of Shapley Values, one of the most widely used approaches in explainable machine learning \citep{doshi2017towards}. They also highlight the peculiar properties of these explanations, and the way in which they are different from other feature attribution methods \citep{covert2021explaining,krishna2022disagreement}. For example, contrary to popular belief, Shapley Values only depend on the coordinates of the point that we attempt to explain, but not on the local neighbourhood of that point. This in turn implies that the explanations are unrelated to the gradient and do not perform any kind of local function approximation \citep{han2022explanation}.

We consider $n$-Shapley Values to be a useful tool for practitioners who want to debug black-box models. Moreover, we introduce a novel method to plot feature attributions of higher order that is consistent with the underlying theory (depicted, for example, in Figure \ref{fig:n_shap}). We also introduce a way to estimate the amount of variable interaction that is necessary to represent a given function. Finally, we study the link between accuracy and the average degree of variable interaction present in different standard classifiers \mbox{(Section \ref{sec:trade-off}).}

\section{RELATED WORK}

{\bf Shapley Values.} The seminal paper by \citet{lundberg2017unified} has led to a line of work that investigates the usage of Shapley Values in explainable machine learning \citep{chen2020true,heskes2020causal,SlackEtal20,albinicshap}. Shapley Values originate in a literature on economic game theory \citep{shapley1953value}, and our work builds on a particular paper from this literature, namely the seminal work by \citet{grabisch1997k} on additive set functions. The idea to extend Shapley Interaction Values towards $n$-Shapley Values is closely related to other approaches that also extend the Shapley Value \citep{grabisch1997k,lundberg2020local,sundararajan2020shapley,tsai2022faith}. The efficient computation of Shapley Values is a topic of ongoing research interest \citep{lundberg2020local,jethani2021fastshap}. Our results also relate to the debate about the choice of value function \citep{sundararajan2020many,janzing2020feature}.
Shapley Values have been explored in various tasks with human decision makers, a topic about which there is much debate \citep{kumar2020problems}. %

{\bf Generalized Additive Models.} Generalized additive models originate in statistics \citep{hastie1990generalized} and have recently become popular in combination with trees \citep{lou2012intelligible,lou2013accurate} and neural networks \citep{agarwal2021neural}. On tabular data sets, interpretable GAMs with few interactions \citep{caruana2015intelligible} can often achieve competitive accuracy, which has led to an active line of research on these models \citep{wang2022interpretability,lengerich2022death}. From a statistical perspective, the decomposition of a function as a GAM is underdetermined, which has led to the development of additional uniqueness criteria such as functional ANOVA \citep{hooker2007generalized,lengerich2020purifying}.

{\bf Explainable Machine Learning.} Shapley Values are one of many different feature attribution methods \citep{RibeiroEtal16,sundararajan2017axiomatic,kommiya2021towards} about which there is a large literature  \citep{LeeEtal19,GarLux20,SlackEtal21a,covert2021explaining,krishna2022disagreement,han2022explanation} and much debate \citep{lipton2018mythos,Rudin19,bordt2022post}. Considerable debate also exists around the question whether there is an accuracy-explainability trade-off or a cost of using interpretable models \citep{Rudin19,moshkovitz2020explainable}. Apart from GAMs, there are many other interpretable models such as rule lists \citep{wang2015falling} and sparse decision trees \citep{lin2020generalized}. Since our work is exclusively focused on Shapley Values and GAMs, we do not offer a comprehensive review of the literature on explainable machine learning. This can be found in many other places \citep{Molnar20,samek2021explaining,holzinger2022xxai,rudin2022interpretable}.

\section{BACKGROUND AND NOTATION}

\begin{figure*}[t]
    \centering
    \includegraphics[width=\textwidth]{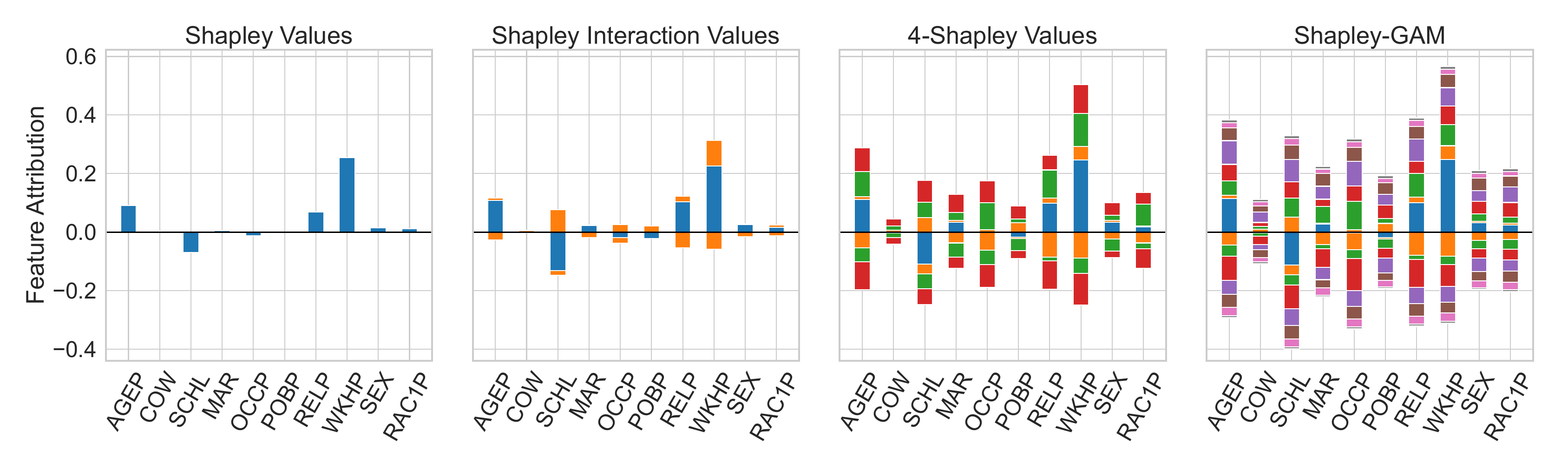}
    \includegraphics[width=0.9\textwidth]{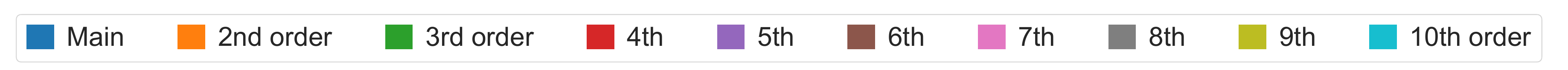}
    \caption{$n$-Shapley Values generate a sequence of explanations of increasing complexity, ranging from the original Shapley Values to the Shapley-GAM. From left to right: Shapley Values ($n=1$), Shapley Interaction Values ($n=2$), 4-Shapley Values ($n=4$) and the Shapley-GAM ($n=d$). In each plot, we distributed the higher-order interaction effects uniformly onto all involved features (as justified by Theorem \ref{thm:thm_shapley_values_from_gam}). Taking into account the signs of the attributions, the different contributions to each of the bars sum to the Shapley Value of that feature (Equation \eqref{prop:n_shapley_recursion}). Taking the overall sum over all  bars for all features recovers the prediction $f(x)$. See Appendix Section \ref{apx:visualization} for more details regarding this visualization. In this example, the function $f$ is a random forest on the Folktables Income classification task, the data point is the first observation in our test set, and we used the value function of interventional SHAP.}
    \label{fig:n_shap}
\end{figure*}

We consider data points $x\in\mathbb{R}^d$ with $d$ features, and a function $f:\mathbb{R}^d\to\mathbb{R}$ whose behavior we want to explain. %
We consider the  {\bf local post-hoc explanation} setting with {\bf feature attributions}: For a point \mbox{$x \in \mathbb{R}^d$}, our goal is to explain which input features (or combinations thereof) were most influential in determining the ``decision'' $f(x)$. In order to do so, we assign real numbers to input features and their combinations. The higher the absolute value of this number, the more influential the feature is considered to be (for an illustration consider Figure \ref{fig:n_shap}).

In what follows, we denote $[n]=\{1,\dots,n\}$ and use subsets of coordinates $S=\{s_1,\dots,s_n\}\subset[d]$ to index both data points $x_S=(x_{s_1},\dots,x_{s_n})$ and collections of functions $f_S(x_S)=f_{x_{s_1},\dots,x_{s_n}}(x_{s_1},\dots,x_{s_n})$ where we assume the ordering $s_1<\dots<s_n$.

\subsection{Value Functions and Shapley Values}

For a data point $x\in\mathbb{R}^d$, a subset of coordinates $S\subset[d]$, and a function $f$, 
the {\it value function }$v(x,S)$ is supposed to quantify how much the features that are present in $S$ contribute towards the prediction $f(x)$. 
Two important value functions are the {\it observational SHAP} value function \citet{lundberg2017unified} \begin{equation}
\label{eq:SHAP}
    v(x,S)=\mathbb{E}_{z\sim\mathcal{D}}\left[f(z)\,|\,x_S \right]%
\end{equation}
and the {\it interventional SHAP} value function \citep{chen2020true,janzing2020feature}
\begin{equation}
\label{eq:interventional_SHAP}
    v(x,S)=\mathbb{E}_{z\sim\mathcal{D}}\left[f(z)\,|\,do(x_S)\right].
\end{equation}
Shapley Values, denoted by $\Phi_i(x)$, are obtained from the value function via the well-known Shapley formula \citep{shapley1953value}. %
We first introduce the Shapley Interaction Index \citep{grabisch1999axiomatic}, given by $\Delta_S(x)=$
\begin{equation}
\label{eq:delta_S}
\sum_{T\subset[d]\setminus S}\frac{(d-|T|-|S|)!|T|!}{(d-|S|+1)!}\sum_{L\subset S}(-1)^{|S|-|L|}v(x,L\cup T).
\end{equation}
The Shapley Value $\Phi_i(x)$ of feature $i$ at $x$ is then simply given by $\Delta_{i}(x)$. Importantly, different value functions give rise to different Shapley Values, so that there effectively exists a multiplicity of possible Shapley Values, depending on our choice of value function \citep{sundararajan2020many}. The popular KernelSHAP algorithm \citep{lundberg2017unified} approximates  Shapley Values with respect to the interventional SHAP value function. The corresponding attributions are also known as the {\it SHAP feature attributions}. The following regularity condition, satisfied by both \eqref{eq:SHAP} and \eqref{eq:interventional_SHAP}, will guarantee that the value function gives rise to a well-defined functional decomposition of the function that we attempt to explain.

\begin{definition}[Subset-Compliant Value Function]
\label{def:subset-compliant-vfunc}
We say that $v(x,S)$ is a subset-compliant value function for $f:\mathbb{R}^d\to\mathbb{R}$ if $v(x,[d])=f(x)$ and if the value $v(x,S)$ depends only on those coordinates of $x$ that are indexed by $S$. For a subset-compliant value function, we also write $v(x,S)=v(x_S,S)$.
\end{definition}

\subsection{Generalized Additive Models}

We employ the following definition of a generalized additive model (GAM) of order $n$.

\begin{definition}[Generalized Additive Model of order $n$]
\label{def:gam}
We say that $f:\mathbb{R}^d\to \mathbb{R}$ is a generalized additive model of order $n$ if $f$ can be written in the form
\begin{equation}
\label{eq:def_gam}
    f(x)=\sum_{S\subset[d],\,|S|\leq n}f_S(x_S)
\end{equation}
\end{definition}
In words, the function $f$ can be described as a simple sum with interaction terms of at most $n$ variables at a time. The individual functions $f_S$ are called component functions of $f$. 
GAMs with few interactions ($n=1,2,3$) are often considered interpretable and called Glassbox-GAMs \citep{lou2012intelligible,caruana2015intelligible}. The reason for this is that the feature-wise shape functions $f_1,\dots,f_d$
can be easily visualized, see for example Figure \ref{fig:recovery}.%

If we allow for interactions of arbitrary order, that is $n=d$, then every function can be written as a GAM. 
However, it is a well-known fact that representing an arbitrary function according to \eqref{eq:def_gam} is under-determined: Many such representations might be possible for the same function. Any such representation is called a functional decomposition of $f$. This non-identifiability has led to the development of additional criteria on the decomposition, such as functional ANOVA, that resolve the identification problem \citep{hooker2007generalized,lengerich2020purifying}.

\section{FROM SHAPLEY VALUES TO GENERALIZED ADDITIVE MODELS}

We now introduce $n$-Shapley Values, a parametric family of local-post hoc explanation algorithms that extends Shapley Values \citep{lundberg2017unified} and Shapley Interaction Values \citep{lundberg2020local}. We then show that every subset-compliant value function implicitly provides a functional decomposition of the function that we attempt to explain. Due to its connection with Shapley Values, we denominate this decompositions the Shapley-GAM. We then show that for $n=d$, $n$-Shapley Values are equal to this decomposition.

\subsection{$n$-Shapley Values}

\begin{figure*}[t]
    \centering    
    \includegraphics[width=\textwidth]{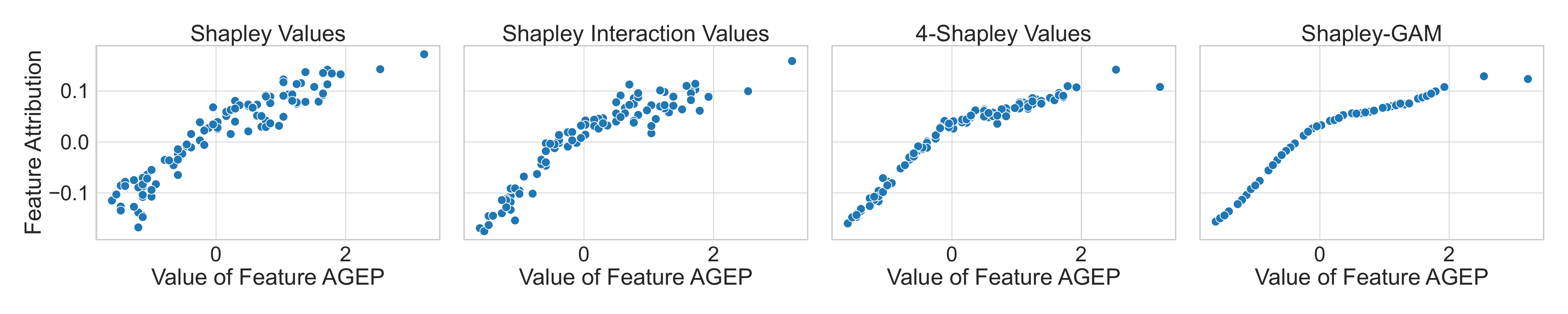}
    \caption{As $n\to d$, the $n$-Shapley Values provide increasingly precise representations of the component functions $f_S$ of the Shapley-GAM. This figure depicts partial dependence plots of $\Phi_\texttt{AGEP}^1$ (Shapley Values, $n=1$), $\Phi_\texttt{AGEP}^2$ (Shapley Interaction Values, $n=2$), $\Phi_\texttt{AGEP}^4$ (4-Shapley Values, $n=4$) and $\Phi_\texttt{AGEP}^{10}$ (Shapley-GAM, $n=d$). The leftmost partial dependence plot is the usual plot that is often used in order to visualize Shapley Values \citep{lundberg2020local} (the plot depicts the original Shapley Values for the observations in the test set). It takes the often observed form where the Shapley Values are scattered around an overall functional relationship. Theorem \ref{thm:shapley_gam} and Theorem \ref{thm:thm_shapley_values_from_gam} make this intuition precise by specifying how the Shapley Values are related to the component functions of the Shapley-GAM. The middle and right plots illustrate that as we move towards higher-order explanations, interaction effects can be appropriately represented. As a consequence, the partial dependence plots of individual feature attributions approach the component functions of the Shapley-GAM. In this example, the function $f$ is a kNN classifier on the Folktables Income classification task. Appendix Figure \ref{apx:knn_dependence} depicts the partial dependence plots of all other features.}
    \label{fig:partial_dependece_plot}
\end{figure*}

The definition of $n$-Shapley Values relates to the function $f$ that we want to explain implicitly via the value function. 

\begin{definition}[$n$-Shapley Values] 
\label{def:n_shapley_values} Fix a function \mbox{$f: \mathbb{R}^d \to \mathbb{R}$.} Let $v(x,S)$ be a value function for $f$. \mbox{$n$-Shapley Values} $\Phi_S^n$ provide an an attribution to all groups of at most $n$ features at a time, that is for all sets $S\subset[d]$ with $|S|\leq n$. We define them recursively, starting from the original Shapley Values at $n=1$ up to $n=d$, by
\begin{equation}
\label{eq:n_shapley_definition}
\Phi_S^n=\begin{cases}
          \qquad\Delta_S  
& \text{if }|S|=n \\
\\
\Phi_S^{n-1}+B_{n-|S|}\mathlarger{\sum}_{\substack{K\subset [d]\setminus S\\|K|+|S|=n}}\Delta_{S\cup K}
 & \text{if } |S|<n.
    \end{cases}
\end{equation}
The coefficients $B_n$ that balance the different terms are the Bernoulli numbers (see Appendix \ref{apx:n_shapley}). All terms except the Bernoulli numbers additionaly depend on the point $x$. %
\end{definition}

While this definition might seem rather abstract, $n$-Shapley Values are actually a straightforward extension of  Shapley Interaction Values \citep{lundberg2020local}. These correspond to the case $n=2$. The original Shapley Values correspond to the case $n=1$. Similar to the original Shapley Values, $n$-Shapley Values are additive and always sum to the function value $f(x)$ (when summed over all subsets $S \subset [d])$ of size $\leq n$).\footnote{The proof of Proposition \ref{prop:efficiency} in the Appendix shows that the Bernoulli numbers are exactly the coefficients that balance equation \eqref{eq:n_shapley_definition} in this way.} The overall intuition behind the recursive definition of $n$-Shapley Values is that starting from the original Shapley Values at $n=1$, we successively add higher-order variable interactions to the explanations.

$n$-Shapley Values give rise to a sequence of explanations of increasing complexity, ranging from the original Shapley Values up to a functional decomposition of the function that we attempt to explain (see Theorem \ref{thm:shapley_gam} below). Figure \ref{fig:n_shap} depicts such a sequence of explanations for a random forest on the Folktables Income classification task \citep{ding2021retiring}. To visualize the $n$-Shapley Values, we evenly distribute all higher-order interactions onto the involved features. As we detail in Appendix \ref{apx:visualization}, this technique is justified by the recursive relationship between $n$-Shapley Values of different order. Note that $n$-Shapley Values of higher order are different from those of lower order only if the function that we attempt to explain actually contains higher-order variable interactions (this intuition will be made precise in Section \ref{sec:recovery}). For this reason, $n$-Shapley Values can be used as a tool to assess the amount of variable interaction that is present in a given black-box predictor. For the random forest, we can see from the rightmost part of Figure \ref{fig:n_shap} that it relies on very high degrees of variable interaction (for a quantitative analysis, see Section \ref{sec:trade-off}).

\subsection{The Shapley-GAM}

The following Theorem \ref{thm:shapley_gam} shows two things. First, a subset-compliant value function gives rise to a well-defined functional decomposition. Second, $d$-Shapley Values are equal to this decomposition. The transformation of the value function that defines the decomposition is well-known as the Harsanyi Dividend \citep{harsanyi1982simplified} or Möbius transform.

\begin{theorem}[$d$-Shapley Values provide a functional decomposition of $f$]
\label{thm:shapley_gam}
Fix a function $f: \mathbb{R}^d \to \mathbb{R}$. Let $v(x,S)$ be a subset-compliant value function for $f$. Then the $d$-Shapley Values represent the function $f$ as a specific GAM that we denominate the Shapley-GAM. It is given by \\[1pt]
\begin{equation}
    f(x)=\sum_{S\subset[d]}f_S(x_S)
\end{equation}
with component functions
\begin{equation}
    f_{\emptyset}=v(\emptyset)\quad\text{and}\quad f_S(x_{S})=\Phi_S^d(x)
\end{equation}
where
\begin{equation}
    \label{eq:moebius_transform}
    \Phi_S^d(x)=\sum_{L\subset S}(-1)^{|S|-|L|}v(x_{L},L).
\end{equation}
\end{theorem}

For intuition about Theorem \ref{thm:shapley_gam}, consider Figure \ref{fig:partial_dependece_plot}. It is a well-known fact that the Shapley Value of feature $i$ not only depends on the value of that feature, but also on the values of the other features of $x$ (compare the leftmost partial dependence plot in Figure \ref{fig:partial_dependece_plot}). The reason for this is that Shapley Values subsume higher-order variable interactions into the attributions of individual features (according to formula \eqref{eq:shapley_values_from_gam_formula}, as we will see below). Now, as we successively increase $n$, more and more variable interactions are appropriately represented in the explanations. This means that they no longer have to be subsumed into lower-order effects, which implies in turn that the lower-order components of the explanations become more distinct (middle parts of Figure \ref{fig:partial_dependece_plot}). For $n=d$, all possible variable interactions can be represented in the explanations, which implies that $d$-Shapley Values become well-defined functions of the respective features (rightmost plot in Figure \ref{fig:partial_dependece_plot}).

$n$-Shapley Values depend on the value function, and so does the associated functional decomposition. For the observational and interventional SHAP value functions, the functional decompositions are given as follows.

\begin{corollary}[Observational and Interventional SHAP]
For the observational SHAP value function \eqref{eq:SHAP}, the component functions of the Shapley-GAM are given by $f_{\emptyset}=\mathbb{E}[f]$,
\begin{equation}
\begin{split}
f_i(x_i)&=\mathbb{E}[f|x_i]-\mathbb{E}[f]\\[6pt]
f_{i,j}(x)&=\mathbb{E}[f|x_i,x_j]-\mathbb{E}[f|x_i]-\mathbb{E}[f|x_j]+\mathbb{E}[f]\\[6pt]
f_{S}(x_{S})&=\sum_{L\subset S}(-1)^{|S|-|L|}\mathbb{E}[f|x_L].\\
\end{split}
\end{equation}
For the interventional SHAP value function, the component functions are given by the same expression, but with the conditional expectations replaced by the causal \textit{do}-operator.
\end{corollary}
As will see below (Theorem \ref{thm:value_function_from_gam}), there is actually a one-to-one relationship between subset-compliant value functions and different functional decompositions of~$f$.

\section{FROM GENERALIZED ADDITIVE MODELS TO SHAPLEY VALUES}

\begin{figure*}[t]
    \centering
     \includegraphics[width=\textwidth]{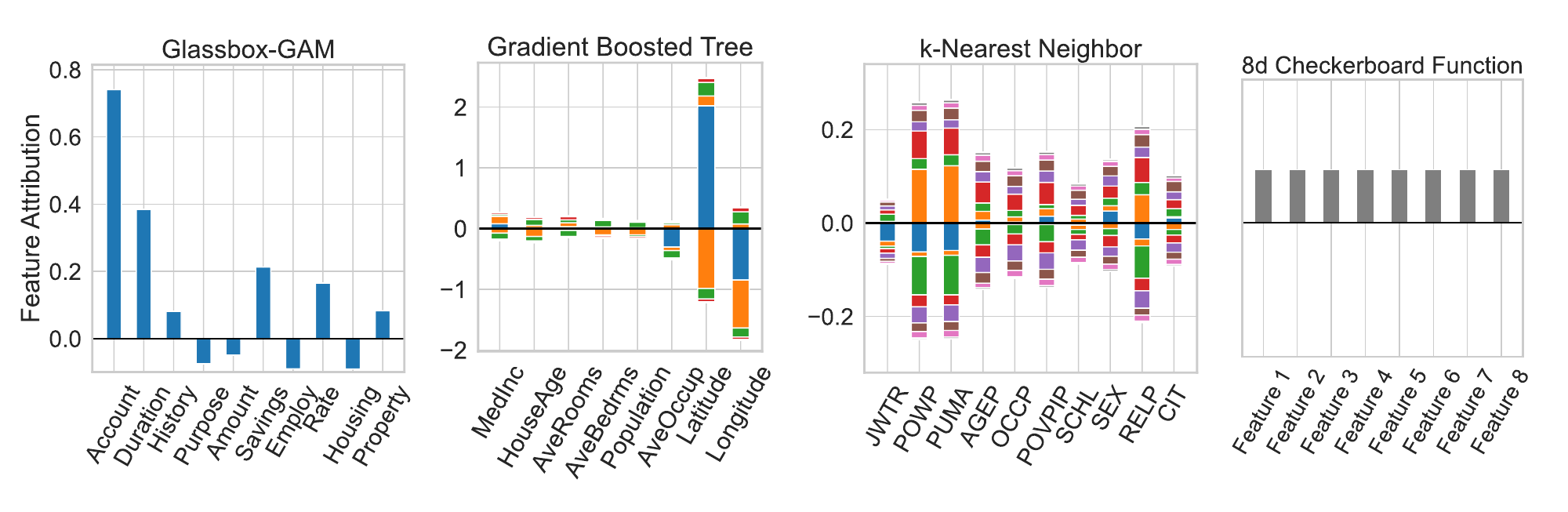}  
     \includegraphics[width=0.9\textwidth]{figures/main_paper/legend.pdf} 
    \caption{Visualizing the Shapley-GAM of interventional SHAP. Figures depict $d$-Shapley Values, visualized as in Figure \ref{fig:n_shap}. Different functions on different data sets require a different degree of variable interaction. (Left) A GAM without variable interactions on the German Credit data set.  (Middle Left) A gradient boosted tree on the California Housing data set. (Middle Right) A kNN classifier on the Folktables Travel data set. (Right) The 8-dimensional checkerboard function \eqref{eq:checkerboard_fn}. Additional figures for more data points and classifiers can be found in Appendix \ref{apx:figures}.}
    \label{fig:shapley_gam_explanations}
\end{figure*}

In the previous section, we have seen that Shapley Values give rise to a functional decomposition of the original function (via the associated value function). In this section, we show that the original Shapley Values as well as $n$-Shapley Values of any order are linear combinations of the component functions of this decomposition. This provides a novel motivation for Shapley Values that does not require value functions or the Shapley formula. This alternative motivation of Shapley Values is equivalent to the original motivation via value functions: For every functional decomposition of $f$, there is a corresponding subset-compliant value function $v$ such that the Shapley Values derived from the decomposition and $v$ are equal (and vice-versa).

\subsection{Shapley Values from the Shapley-GAM}

Theorem \ref{thm:thm_shapley_values_from_gam} specifies the way in which the different component functions of the Shapley-GAM give rise to $n$-Shapley Values.

\begin{theorem}[$n$-Shapley Values from the Shapley-GAM]
\label{thm:thm_shapley_values_from_gam} Let $f(x)=\sum_{S\subset [d]}f_{S}(x_S)$ be the decomposition of $f$ provided by the Shapley-GAM, and let $\Phi_{S}^n(x)$ be the $n$-Shapley Values of $f$. Then, it holds that
\begin{equation}
\label{eq:n_shapley_values_from_gam_formula}
    \Phi_{S}^n=f_S+\sum_{\substack{K\subset [d]\setminus S\\n+1\leq |S|+|K|}}C_{n-|S|,|K|}\,f_{S\cup K}
\end{equation}
with coefficients $C_{n,m}=\sum_{k=0}^n\binom{n}{k}\frac{B_k}{1+m-k}$. Specifically, the Shapley Value of feature $i$ is given by
\begin{equation}
\label{eq:shapley_values_from_gam_formula}
    \Phi^1_i=f_i+\cdots+\frac{1}{k+1}\sum_{S\subset [d]\setminus \{i\},|S|=k} f_{S\cup \{i\}}+\cdots+\frac{1}{d}\,f_{[d]}
\end{equation}
where all terms additionally depend on the point $x$.
\end{theorem}
Theorem \ref{thm:thm_shapley_values_from_gam} specifies how higher-order variable interactions that are present in $f$ are subsumed into lower-order explanations. In the case of the original Shapley Values, this is particularly intuitive: Higher-order effects are evenly distributed among the involved features.\footnote{For individual value functions, equation \eqref{eq:shapley_values_from_gam_formula} is known in the literature on economic game theory \citep{grabisch1997k}[Theorem 1]. Variants of it were independently re-discovered in \citet{keevers2020power}, \citet{herren2022statistical} and \citet{causal_shapley_gam}.} Theorem \ref{thm:thm_shapley_values_from_gam} also specifies what information about the function $f$ is and is not contained in Shapley Values. We see that different functions $f$ can give rise to the same  $n$-Shapley Values as long as $n<d$ \citep{grabisch2016bases}. We also see that it is impossible to tell from individual Shapley Values whether the model consists of main effects or complex variable interactions. Furthermore, a feature can have zero attribution although it appears in multiple interaction effects with different signs.

For a bit more intuition about the Shapley-GAM, Figure \ref{fig:shapley_gam_explanations} illustrates the Shapley-GAM of interventional SHAP for different functions. A main point is that different predictors require a different degree of variable interaction in order to be represented as a GAM. By definition, a Glassbox-GAM (leftmost part of Figure \ref{fig:shapley_gam_explanations}) does not require any variable interaction. The other extreme is the $k$-dimensional checkerboard function \eqref{eq:checkerboard_fn} (rightmost part of Figure \ref{fig:shapley_gam_explanations}), which only consists of interaction terms of order $k$. Many learned functions such gradient boosted trees (Figure \ref{fig:shapley_gam_explanations}, middle left) and the k-Nearest Neighbor (kNN) classifier (Figure \ref{fig:shapley_gam_explanations}, middle right) lie in between. Overall, there is a significant amount of variation between different methods and problems. This is also illustrated in many additional figures in Appendix \ref{apx:figures}. For a quantitative analysis, see Section \ref{sec:trade-off}.

\subsection{From Functional Decompositions to Subset-Compliant Value Functions}

We have show that every subset-compliant value function corresponds to a functional decomposition of $f$. We now show that the reverse is also true, that is every functional decomposition of $f$ corresponds to a subset-compliant value function. The transformation that defines the value function is also known as the Zeta transform. 

\begin{theorem}[From Generalized Additive Models to Value Functions]
\label{thm:value_function_from_gam}
Let $f(x)=\sum_{S\subset[d]}g_S(x)$ be any functional decomposition of $f$. Define the subset-compliant value function
\begin{equation}
\label{eq:gam_value_function}
    v(x,S)=\sum_{L\subset S}g_L(x).
\end{equation}
Then the functional decomposition $g_S$ is the Shapley-GAM with respect to the value function \eqref{eq:gam_value_function}. 
\end{theorem}
Taken together, Theorem \ref{thm:shapley_gam} and Theorem \ref{thm:value_function_from_gam} establish a bijection between subset-compliant value functions and functional decompositions of $f$. In a sense, this implies that every functional decomposition implicitly corresponds to a notion of feature attribution via its associated value function and the Shapley formula (or, more directly, via equation \eqref{eq:shapley_values_from_gam_formula} which is just the same).

\section{RECOVERY}
\label{sec:recovery}

In this section, we connect Shapley Values with interpretable models by showing that $n$-Shapley Values, as well as the Shapley Taylor- and Faith-Shap interaction indices, recover GAMs. In order for this to be the case, the order of the explanation has to be at least as large as the order of \mbox{the GAM.}

\begin{theorem}[Shapley-based Explanations Recover GAMs] 
\label{thm:thm_recovery}
Let $f$ be a generalized additive model of order $n$. Assume that either
\begin{itemize}
    \item[(a)] the value function is given by observational SHAP and the individual features are independent random variables, or
    \item[(b)] the value function is given by interventional SHAP.
\end{itemize}
Then, $n$-Shapley Values, as well as the Shapley Taylor- and Faith-Shap interaction indices of order $n$, recover a representation of $f$ as a GAM. In fact, all the interaction indices are equal to each other and given by
\begin{equation*}
    \Phi_S^n(x)=f_S(x_S)
\end{equation*}
where $f_S$ are the component functions of the Shapley-GAM. 
\end{theorem}

Theorem \ref{thm:thm_recovery} implies that the SHAP feature attributions recover GAMs without variable interactions and that Shapley Interaction Values recover GAMs with interactions of at most two variables at a time. Unlike our previous results, Theorem \ref{thm:thm_recovery} depends on the choice of the value function. This is because the recovery property holds if (1) the interaction index can be written like in equation \eqref{eq:n_shapley_values_from_gam_formula}, and (2) the Shapley-GAM is a GAM or order $n$ --- and the second point depends on the value function. 

As it turns out, the independence assumption in part (a) of Theorem \ref{thm:thm_recovery} is indeed necessary (see Appendix \ref{apx:observational_shap}). This is interesting insofar as it establishes the usefulness of the interventional SHAP value function from a purely statistical perspective, that is without any causal motivation (for a discussion about the differences between observational and interventional SHAP, see also \citet{chen2020true}).

Figure \ref{fig:recovery} (Top) illustrates the recovery result for a GAM without variable interactions. For this example, we explicitly resort to the default implementation of the Kernel SHAP algorithm, in order to see whether there is any significant approximation error (Kernel SHAP approximates the Shapley Values of the interventional SHAP value function). The top part of Figure \ref{fig:recovery} depicts the shape curve of the feature POWPUMA in the GAM (blue curve), as well as the associated Kernel SHAP values (red dots). The Kernel SHAP values lie almost exactly on the shape curve of the GAM, which means that the recovery property holds fairly precisely, at least in this simple example. %

\section{IS THERE AN ACCURACY-\\COMPLEXITY TRADE-OFF?}
\label{sec:trade-off}

\begin{figure}
    \centering   
    \includegraphics[width=0.49\textwidth]{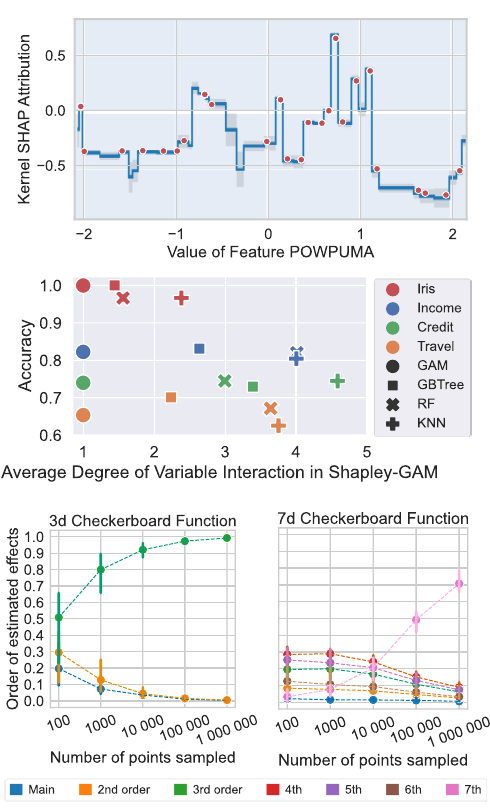}
    \caption{\textbf{Top}: Shapley Values recover GAMs without variable interactions (Theorem \ref{thm:thm_recovery}). To create this figure, we first trained a GAM on the Folktables Travel data set using the InterpretML package \citep{nori2019interpretml}. We then computed the Kernel SHAP values for the decision function of the GAM using the {\tt shap} package \citep{lundberg2017unified}. For the feature POWPUMA, the Figure depicts the ground-truth variable effect in the GAM in blue, and the associated Kernel SHAP values for data points from the test set as red dots. We see that the red dots lie on the blue line, that is Kernel SHAP recovers the component function of the GAM. \textbf{Middle}: The average degree of variable interaction \eqref{eq:def_variable_interaction} in the Shapley-GAM of interventional SHAP for various standard classifiers. The figure depicts predictive accuracy versus the average degree of variable interaction. \textbf{Bottom}: Estimating higher-order variable interactions requires precise evaluations of the value function. A simple way to study this is by estimating the $k$-dimensional checkerboard function \eqref{eq:checkerboard_fn}. \textit{Left:} 3-way variable interactions can be precisely estimated. \textit{Right:} 7-way variable interactions can be reliably detected, but precise estimation requires prohibitively many samples.}
    \label{fig:recovery}
\end{figure}

In the previous sections, we have outlined the connections between Shapley Values and GAMs on a theoretical level. In this section, as well as in the next section, we turn to more practical concerns. In this section, we investigate the number of variable interactions that are present in various standard classifiers.  In order to do so, we rely on a number of low-dimensional data sets on which we can reliably estimate the Shapley-GAM decompositions of the different learned predictors (compare Section \ref{sec:estimating}). It is interesting to compare this against the accuracy: Because models with more variable interactions can represent strictly more functions than models with less variable interactions, it is natural to suspect that more accurate classifiers might exhibit higher degrees of variable interaction \citep{dziugaite2020enforcing}.

We suggest to measure the extent of variable interaction that is present in a given classifier with the following quantity
\begin{equation}
\label{eq:def_variable_interaction}
    \mathop{\mathbb{E}}_{x\sim\mathcal{D}}\left[\sum_{S\subset[d]}|S|\cdot|f_S(x_S)|\right] \Big/ \mathop{\mathbb{E}}_{x\sim\mathcal{D}}\left[\sum_{S\subset[d]}|f_S(x_S)|\right].
\end{equation}
where $f_S$ are the component functions of the Shapley-GAM decomposition of $f$, using interventional SHAP. 

Figure \ref{fig:recovery} (Middle) illustrates the relationship between the predictive accuracy and our measure \eqref{eq:def_variable_interaction} for different predictors $f$. The figure depicts four different kinds of classifiers: A Glassbox-GAM without variable interactions \citep{nori2019interpretml}, a gradient boosted tree \citep{xgboost}, a random forest, and a kNN classifier \citep{scikit-learn}. We compare these classifiers on four different data sets: Folktables Travel and Income \citep{ding2021retiring}, Iris, and German Credit. Details on the data sets and training procedures are in Appendix \ref{apx:datasets_models}. 

As far as accuracy is concerned, we see from Figure \ref{fig:recovery} that GAMs without variable interactions perform fairly well against the more complicated classifiers --- a fact that has often been observed in the literature \citep{caruana2015intelligible,agarwal2021neural}. On the more complex data sets, however, there is usually a model with variable interactions and slightly better accuracy\footnote{The InterpretML package \citep{nori2019interpretml} allows to include interactions between pairs of variables which reportedly allows to be on par with black-box models on many data sets. Compare also \citep{lou2012intelligible}.} As far as the degree of variable interaction is concerned, we see that there is a large amount of variation in between the different classifier. 

Especially interesting is the kNN classifier. It tends to perform worse in terms of accuracy than the interpretable GAM, but exhibits very high degree of variable interaction. Observe that the kNN classifier can also be considered interpretable (by explaining the workings of the classifier and providing the $k$ data points that are responsible for the classification). Therefore, this example shows that a high degree of variable interaction in the Shapley-GAM does not imply that a function is hard to explain per se. %

This simple empirical investigation suggests that the relation between accuracy and the average degree of variable interaction in the Shapley-GAM is nuanced: While some degree of interaction seems necessary in order to achieve competitive accuracy, some classifiers seem to exhibit more interaction than that. In some cases, the correlation might even be negative (as for the kNN classifier).

\section{COMPUTATION AND ESTIMATION}
\label{sec:estimating}

We now turn to the practical question of computing $n$-Shapley Values. In this work, we take the trivial approach and simply evaluate the value function for all possible subsets $S\subset[d]$, then combine the respective terms according to Definition \ref{def:n_shapley_values}. A Python package to compute $n$-Shapley Values, as well as the Shapley Taylor- and Faith-Shap interaction indices,  is available \url{https://github.com/tml-tuebingen/nshap}. Even for the original Shapley Values, it is well-known that the number of required evaluations of the value function grows exponentially in the number of features. For this reason, there exist efficient approximations such as Kernel SHAP, as well as efficient implementations for certain function classes such as tree based models \citep{lundberg2017unified}. We hold that such computationally efficient approximations are also be possible for $n$-Shapley Values.

Instead of focusing on the well-known computational aspect of the problem, we want to focus on the estimation aspect which seems much less studied. Note that $n$-Shapley Values are a statistic that is subject to sampling variation. The same is true for our visualizations (as in Figure \ref{fig:n_shap}), which are summary statistics of $n$-Shapley Values. This is because both the observational and the interventional SHAP value function require to estimate an expectation.

We now asses with a simple empirical analysis up to which order interaction effects can be estimated in practice. We consider the $k$-dimensional checkerboard function $B_k:[0,1]^d\to\{0,1\}$ given by
\begin{equation}
\label{eq:checkerboard_fn}
    B_k(x_1,\dots,x_d)=\begin{cases}
    0 \quad\text{if } \sum_{i=0}^k\lfloor(\lambda\cdot \,x_i)\rfloor \mod 2 = 0\\
    1 \quad\text{otherwise}\\
    \end{cases}
\end{equation}
where $\lambda>2$ parameterizes the number of checkers along each axis. If data points are uniformly distributed in the unit cube $[0,1]^d$, then the Shapely-GAM of interventional SHAP of $B_k$ is given by the single $k$-th order interaction effect $f_{x_1,\dots,x_k}(x_1,\dots,x_k)=B_k(x_1,\dots,x_k,0,\dots,0)$. The question now is how precisely we have to estimate the expectation $\mathbb{E}_{z\sim\mathcal{D}}\left[f(z)\,|\,do(x_S)\right]$ if we want to precisely estimate a $k$th-order interaction effect. 

The bottom part of Figure \ref{fig:recovery} depicts the result of estimating $10$-Shapley Values when the underlying function is the 3- or 7-dimensional checkerboard function, respectively. The x-axis depicts the number of samples used to estimate the value function, ranging from 100 to 1 000 000. The y-axis depicts the order of the estimated effects, with confidence bands that account for 5 randomly sampled data sets. From the figure, we observe that if the number of samples is small in relation to the magnitude of the interaction effect, then the estimation results in spurious lower-order effects. For $k=3$, these effects vanish with sufficiently many samples, which means that the checkerboard function is precisely estimated. For $k=7$, the presence of the higher-order interaction effect can be reliably detected, but not precisely estimated given reasonably many samples. 

In this simple analysis, we see that interaction effects of order larger that 2 can be precisely estimated given sufficiently many samples. We also see that functions with high-order interactions are difficult to estimate and can result in artifacts. Figures for all interaction orders $k=2,\dots,10$ and a discussion of the precision of the depicted visualizations of $n$-Shapley Values can be found in Appendix \ref{apx:estimation}.

\section{DISCUSSION}

This work provides a functionally-grounded characterization of Shapley Values as they are being used in explainable machine learning \citep{doshi2017towards}. Explainable machine learning is often believed to be an important component in societal applications of machine learning \citep{WachterEtal17,KamUrb21,Kaestner21}. At the same time, current approaches face a lot of criticism, for example because they are non-robust or unable to provide the desired level of model understanding (as well as for a variety of other concerns) \citep{lipton2018mythos,kumar2020problems,SlackEtal20,bordt2022post}. In this situation, we believe that a precise understanding of the mechanics of popular explainability methods, such as the one presented in this work, is a good first step toward an informed discussion of what we can and cannot achieve.

Some of our results stand in contrast to conventional wisdom around Shapley Values, and offer a novel perspective on local-post hoc explanation algorithms.
For example, we have seen that Shapley Values depend on the coordinates of the point that we attempt to explain, but not on the local neighbourhood of that point --- the recovery example with the step function in Figure \ref{fig:recovery} suggests that this is also the case for the approximations of the Shapley Value that are used in practice. We have further seen that the original Shapley Values are able to faithfully explain non-linear functions, as long as the non-linearity is restricted to the specific form permitted by GAMs. As such, our results highlight the differences between Shapley Values and other feature attribution methods, for example those that are related to the gradient  \citep{GarLux20,agarwal2021towards}, and those that perform local function approximation \citep{han2022explanation}.

The demonstrated connections between value functions and functional decompositions effectively link the literature on feature attributions with the tools developed in the statistics literature on functional decompositions \citep{hooker2007generalized,lengerich2020purifying}. This raises the question of whether criteria for functional decompositions can be useful to understand feature attributions. Here, two concurrent works made significant contributions: \citet{causal_shapley_gam} show that the value function of interventional SHAP can be motivated with a causal assumption on the associated functional decomposition. \citet{herren2022statistical} outline connections between observational SHAP and functional ANOVA. 

While our work gives a functionally-grounded analysis of Shapley Values for any function, as well as recovery guarantees for Shapley Values and GAMs, we do not claim that Shapley Values are an appropriate post-hoc explanation method for any function \citep{kumar2021shapley,tan2022cautionary}. Instead, the purpose of our work is to highlight the connections between a post-hoc explanation method and a class of interpretable models. Overall, however, we believe that many properties of Shapley Values have the potential to be more clearly understood using our perspective of functional decompositions.

\newpage
\section*{Acknowledgements}

This work was done in part while Sebastian was visiting the Simons Institute for the Theory of Computing. Sebastian would like to thank Rich Caruana, Gyorgy Turan, Michal Moshkovitz and Tosca Lechner for many fruitful discussions about variable interactions. The authors would also like to thank Markus Scheuer and René Gy for linking Lemma \ref*{bernoulli_lemma_2} to the literature on Bernoulli numbers, and the anonymous reviewers whose comments helped to improve this paper. This work has been supported by the German Research Foundation through the Cluster of Excellence “Machine Learning – New Perspectives for Science" (EXC 2064/1 number 390727645), the BMBF Tübingen AI Center (FKZ: 01IS18039A), and the International Max Planck Research School for Intelligent Systems (IMPRS-IS).

\bibliography{literature}

\onecolumn
\clearpage
\appendix

\renewcommand\thefigure{\thesection.\arabic{figure}}    
\setcounter{figure}{0}    
\renewcommand\thetable{\thesection.\arabic{table}}  
\setcounter{table}{0}    
\setcounter{footnote}{0}    

\section{$n$-Shapley Values}
\label{apx:n_shapley}

This section details the properties of $n$-Shapley Values. 

\subsection{Bernoulli numbers}

The Bernoulli numbers\footnote{An introduction and discussion about Bernoulli numbers can be found, for example, in the corresponding Wikipedia article at  \url{https://en.wikipedia.org/wiki/Bernoulli_number}.} $B_n$ are defined by $B_0=1$ and 
\begin{equation}
\label{eq:bernoulli_numbers}
    \sum_{k=0}^{n}\binom{n+1}{k}B_{k}=0\quad\forall n\geq 1.
\end{equation}
In this paper, the Bernoulli numbers arise as the coefficients that make $n$-Shapley Values sum to the prediction (Proposition
\ref{prop:efficiency}). In fact, equation \eqref{eq:bernoulli_numbers} arises directly from the proof of Proposition
\ref{prop:efficiency}. 
The Bernoulli numbers can be computed recursively by re-writing into
 \eqref{eq:bernoulli_numbers} 
\begin{equation}
\label{eq:bernoulli_recursive}
B_n=\frac{-1}{n+1}\sum_{k=0}^{n-1}B_k\binom{n+1}{k}\quad\forall n\geq 1.
\end{equation}
In a certain sense, the entire combinatorics around $n$-Shapley Values relies on the properties of the Bernoulli numbers. In particular, the proofs of Theorem \ref{thm:shapley_gam} and Theorem \ref{thm:thm_shapley_values_from_gam} rely on the following two Lemmas.

\begin{lemma}
\label{bernoulli_lemma_1}
For all $n\geq 1$, it holds that \begin{equation}
    \sum_{k=1}^n \frac{B_k}{n-k+1}\binom{n}{k}=\frac{-1}{n+1}.
\end{equation}
\end{lemma}
\begin{proof}
We re-arrange the sum to get
\begin{equation}
    \sum_{k=1}^n \frac{B_k}{n-k+1}\binom{n}{k}=\frac{1}{n+1}\sum_{k=0}^{n}\binom{n+1}{k}B_k-\frac{B_0}{n+1}=\frac{-1}{n+1}
\end{equation}
where the second equality follows from \eqref{eq:bernoulli_numbers}.
\end{proof}

\begin{lemma}
\label{bernoulli_lemma_2}
For all $n,m\geq 0$, it holds that
\begin{equation}
    \sum_{k=0}^n\sum_{l=0}^m\binom{n}{k}\binom{m}{l}\frac{(n-k)!(m-l)!}{(n+m-k-l+1)!}(-1)^l B_{k+l}=\begin{cases}
 1 & \quad\text{if } n=0\\[8pt] 
 0 & \quad\text{otherwise.}\\
\end{cases}
\end{equation}
\end{lemma}

Lemma \ref{bernoulli_lemma_2} follows from standard results for the Bernoulli numbers \citep{gould2014bernoulli}[Theorem 2]. A proof is contained in Appendix \ref{apx:bernoulli_lemma}.

\subsection{Additivity and Efficiency}

From the recursive definition of the $n$-Shapley Values in Definition \ref{def:n_shapley_values}, a straightforward calculation shows that
\begin{equation}
\label{eq:apx_nshapley_explicit}
    \Phi_S^n(x)=\sum_{k=0}^{n-|S|}\sum_{K\subset [d]\setminus S,\,|K|=k} B_k\,\Delta_{S\cup K}(x)
\end{equation}
which is an alternative non-recursive definition of $n$-Shapley Values. 

\newpage

\begin{proposition}[Additivity]
For all $1\leq n\leq d$ and all $f,g:\mathbb{R}^n\to\mathbb{R}$, we have
\begin{equation}
    \Phi^{n}_S(x;f+g)=\Phi^{n}_S(x;f)+\Phi^{n}_S(x;g).
\end{equation}
\end{proposition}
\begin{proof}
By definition, $\Phi^{n}_S$ is linear in $\Delta_S$, and $\Delta_S$ is linear in the value function $v$. Therefore, the linearity of $\Phi^{n}_S$ in $f$ follows from the linearity of $v$ in $f$, i.e. from the fact that $v_{f+g}(x,S)=v_{f}(x,S)+v_g(x,S)$.
\end{proof}

\begin{proposition}[Efficiency]
\label{prop:efficiency}
For all $1\leq n\leq d$, it holds that
\begin{equation}
\sum_{\substack{S\subset[d]\\1\leq |S|\leq n}} \Phi^{n}_S(x)=v([d])-v(\emptyset).    
\end{equation}
\end{proposition}
\begin{proof}
For $n=1$, the statement follows from the efficiency of the original Shapley Values. We assume that the statement holds for $n-1$ and re-arrange the sum
\begin{equation}
\begin{split}
\sum_{\substack{S\subset[d]\\1\leq |S|\leq n}} \Phi^{n}_S(x)&=\sum_{\substack{S\subset[d]\\1\leq |S|< n}} \Phi^{n}_S(x)+\sum_{\substack{S\subset[d]\\|S|=n}} \Phi^{n}_S(x)\\
             &=\sum_{\substack{S\subset[d]\\1\leq|S|< n}}\left(\Phi_S^{n-1}(x)+B_{n-|S|}\,\mathlarger{\sum}_{\substack{K\subset [d]\setminus S\\|K|+|S|=n}}\Delta_{S\cup K}(x)\right)+\sum_{\substack{S\subset[d]\\|S|=n}} \Delta_{S}(x)\\[10pt]
             &=\sum_{\substack{S\subset[d]\\1\leq|S|\leq n-1}}\Phi_S^{n-1}(x)+\sum_{\substack{S\subset[d]\\1\leq|S|< n}}\,\mathlarger{\sum}_{\substack{K\subset [d]\setminus S\\|K|+|S|=n}}B_{n-|S|}\,\Delta_{S\cup K}(x)+\sum_{\substack{S\subset[d]\\|S|=n}} \Delta_{S}(x).\\[4pt]
\end{split}
\end{equation}
Notice that the first term is equivalent to $v([d])-v(\emptyset)$ by the induction hypothesis. It remains to show that 
\begin{equation}
\label{eq:proof_efficiency_1}
    \sum_{\substack{S\subset[d]\\1\leq|S|< n}}\,\mathlarger{\sum}_{\substack{K\subset [d]\setminus S\\|K|+|S|=n}}B_{n-|S|}\,\Delta_{S\cup K}(x)+\sum_{\substack{S\subset[d]\\|S|=n}} \Delta_{S}(x)=0.
\end{equation}
Notice that both sums are over sets of length $n$. In the first sum, each sets occurs multiple times. In the second sum, each set occurs exactly once. By counting the occurrences of each set in the first sum we see that \eqref{eq:proof_efficiency_1} holds if
\begin{equation}
    \sum_{s=1}^{n-1}B_{n-s}\binom{n}{s}+1=0.
\end{equation}
If we set $B_0=1$, this holds if and only if 
\begin{equation}
    \sum_{k=0}^{n-1}B_{k}\binom{n}{k}=0,
\end{equation}
which is the defining property of the Bernoulli numbers \eqref{eq:bernoulli_numbers}. In summary, we see that the Bernoulli numbers are the coefficients that balance the terms in the first sum in equation \eqref{eq:proof_efficiency_1}.
\end{proof}

\begin{table}[t]
\begin{tabular}{lcccccccccccccccccccc}
\hline
n              & 0 & 1 & 2 & 3 & 4 & 5 & 6 & 7 & 8 & 9 & 10 & 11 & 12 & 13 & 14 & 15 & 16 & 17 & 18 & 19\\
\hline\\[-6pt]
$B_n$ & 1 & $\frac{-1}{2}$ & $\frac{1}{6}$ & 0 & $\frac{-1}{30}$ & 0 & $\frac{1}{42}$ & 0 & $\frac{-1}{30}$ & 0 & $\frac{5}{66}$ & 0 & $\frac{-691}{2730}$ & 0 & $\frac{7}{6}$ & 0 & $\frac{-3617}{510}$ & 0 & $\frac{43867}{798}$ & 0  \\[2pt]
\hline
\end{tabular}
    \caption{The first 20 Bernoulli numbers.}
    \label{tab:alpha_k}
\end{table}

\subsection{Relationship Between $n$-Shapley Values of Different Order}

The following proposition is a straightforward extension of Theorem \ref{thm:thm_shapley_values_from_gam}.

\begin{proposition}
[Relationship Between $n$-Shapley Values of Different Order]
\label{prop:n_shapley_recursion}
For $m\leq n$, let $\Phi^{m}_S$ and $\Phi^{n}_S$ be the $m$- and $n$-Shapley Values, respectively. Then, the $m$-Shapley Values can be computed from the $n$-Shapley Values by
\begin{equation}
    \Phi_S^m(x)=\Phi_S^n+\sum_{\substack{K\subset[d]\setminus S,\\[2pt]m-|S|<|K|\leq n-|S|}}\beta_{m-|S|,|K|}\,\Phi_{S\cup K}^n(x).
\end{equation}
Specifically, it holds that
\begin{equation}
    \Phi_i^1=\Phi^n_i+\frac{1}{2}\sum_{j\neq i}\Phi^n_{i,j}+\dots+\frac{1}{n}\sum_{\substack{K\subset [d]\setminus\{i\}\\ |K|=n-1}}\Phi^n_{K\cup {i}}
\end{equation}
which is the basis for the visualizations in the paper. 
\end{proposition}
\begin{proof}
The proposition follows from the counting argument used in the proof of Theorem \ref{thm:thm_shapley_values_from_gam}.
\end{proof}

\section{Visualizing $n$-Shapley Values}
\label{apx:visualization}

\begin{figure}[t]
     \centering
     \begin{subfigure}[b]{0.18\textwidth}
         \centering
         \includegraphics[width=\textwidth]{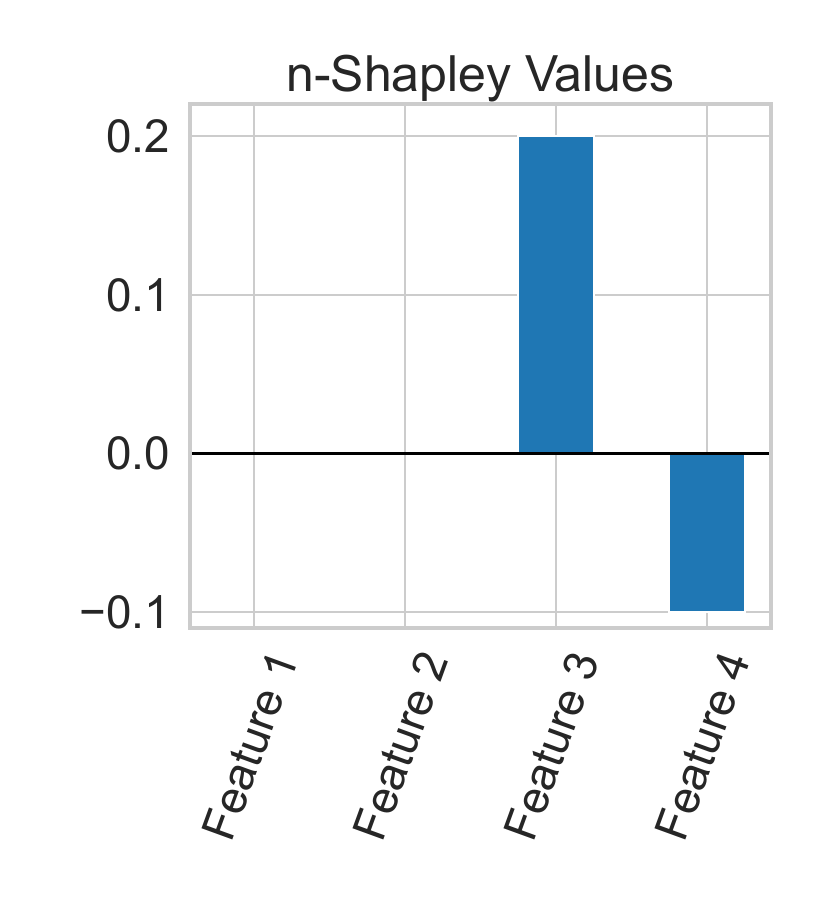}
         \caption{Example 1}
         \label{fig:e1}
     \end{subfigure}
     \hfill
     \begin{subfigure}[b]{0.18\textwidth}
         \centering     
         \includegraphics[width=\textwidth]{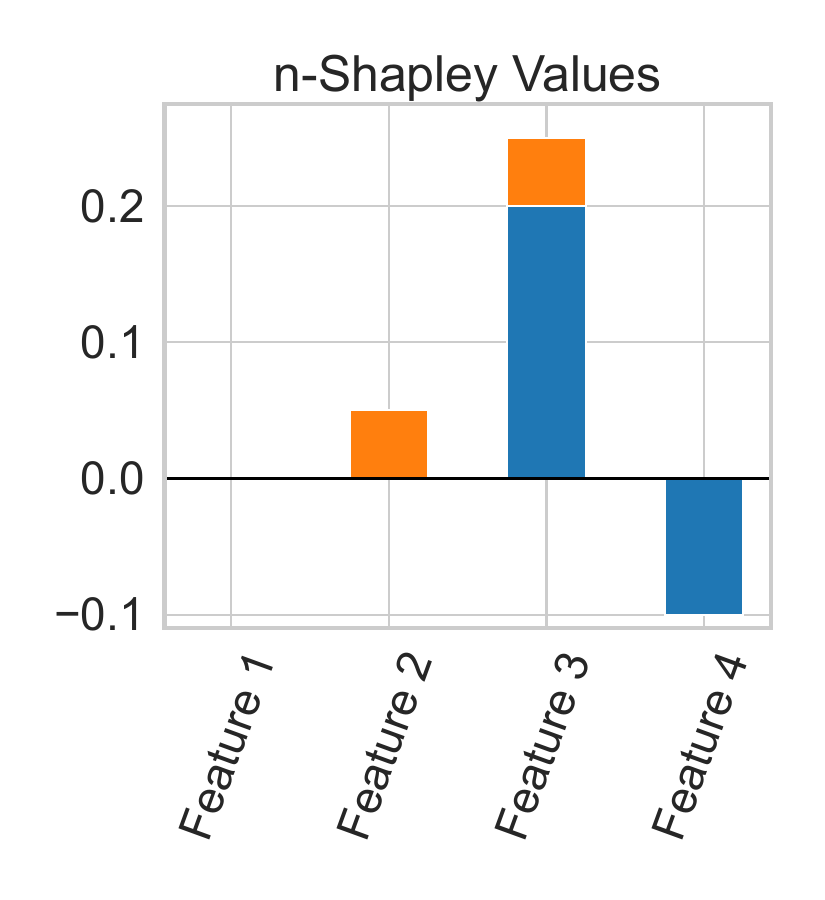}
        \caption{Example 2}
         \label{fig:e2}
     \end{subfigure}
     \hfill     
     \begin{subfigure}[b]{0.18\textwidth}
         \centering
         \includegraphics[width=\textwidth]{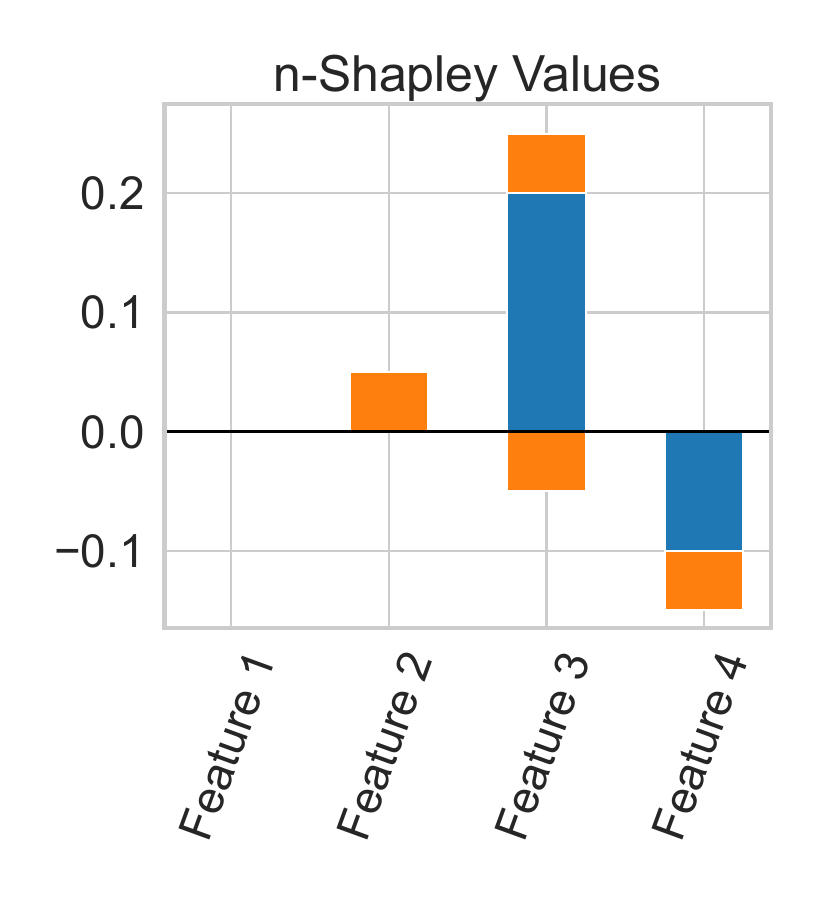}
        \caption{Example 3}
         \label{fig:e3}
     \end{subfigure}
     \hfill     \begin{subfigure}[b]{0.18\textwidth}
         \centering
         \includegraphics[width=\textwidth]{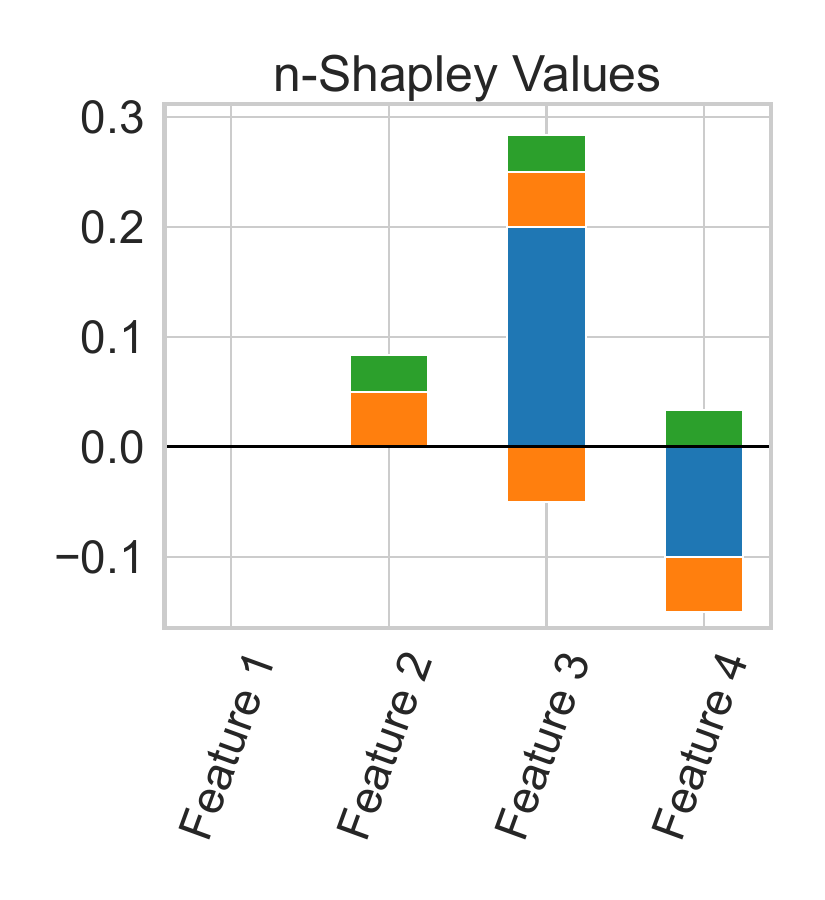}
        \caption{Example 4}
         \label{fig:e4}
     \end{subfigure}
     \hfill     \begin{subfigure}[b]{0.18\textwidth}
         \centering
         \includegraphics[width=\textwidth]{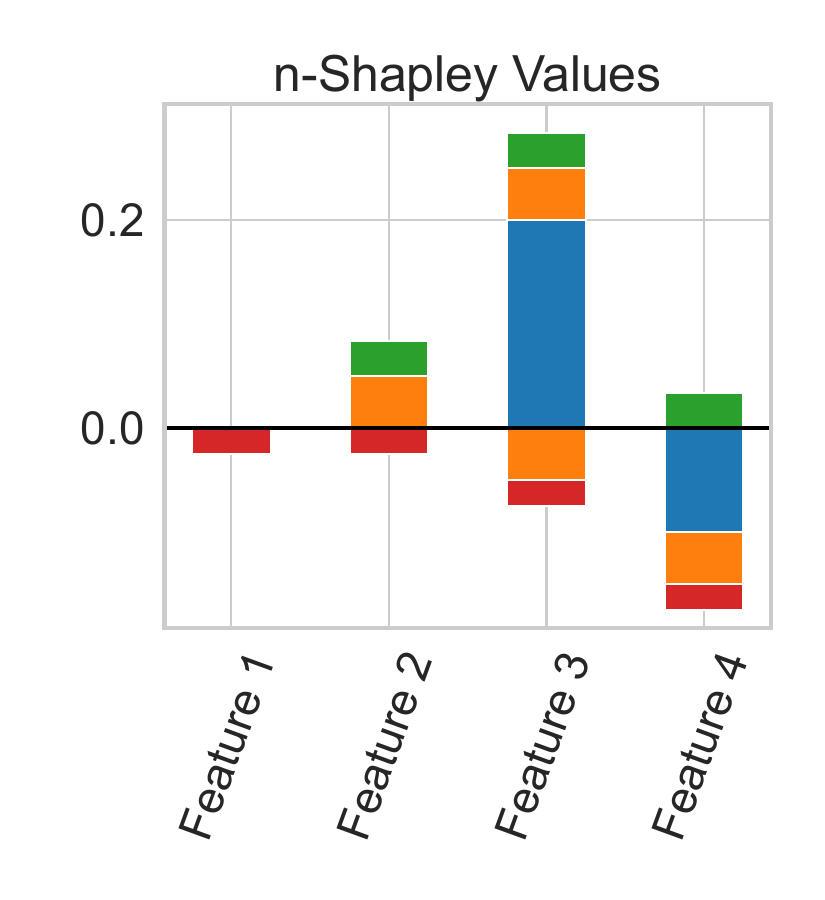}
        \caption{Example 5}
         \label{fig:e5}
     \end{subfigure}
     \hfill
     \caption{Examples that illustrate the proposed visualization technique for $n$-Shapley Values.}
\end{figure}

\begin{figure}[t]
    \centering
    \includegraphics[width=0.275\textwidth]{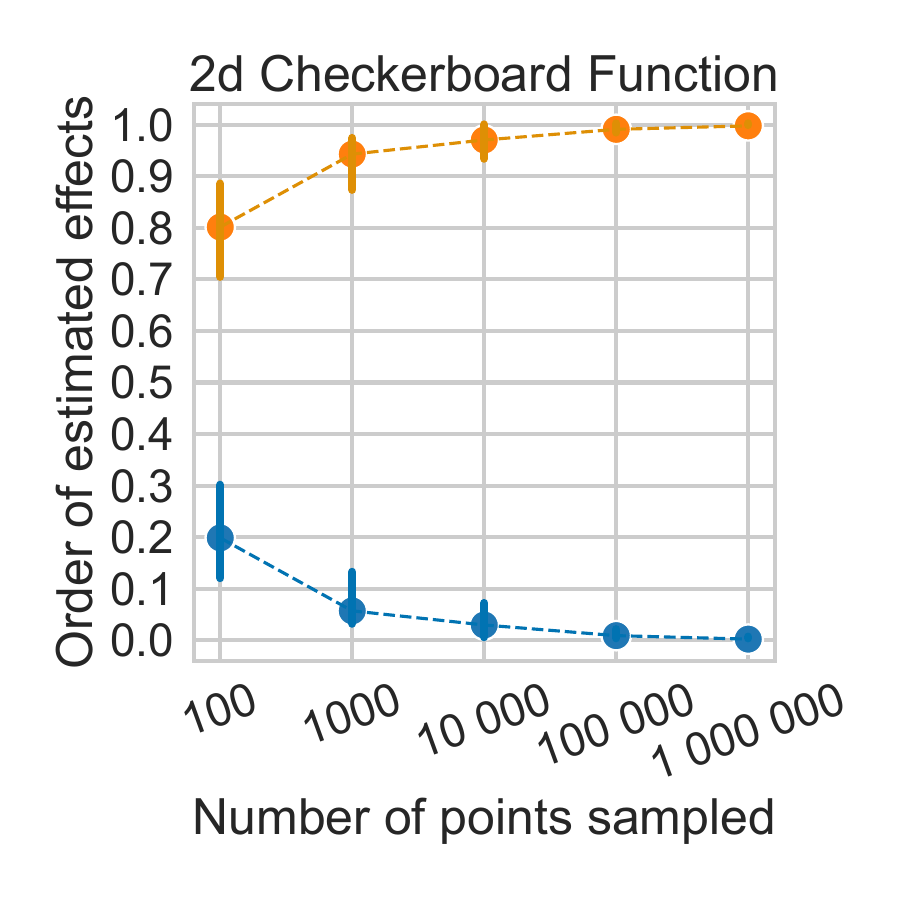}
    \includegraphics[width=0.275\textwidth]{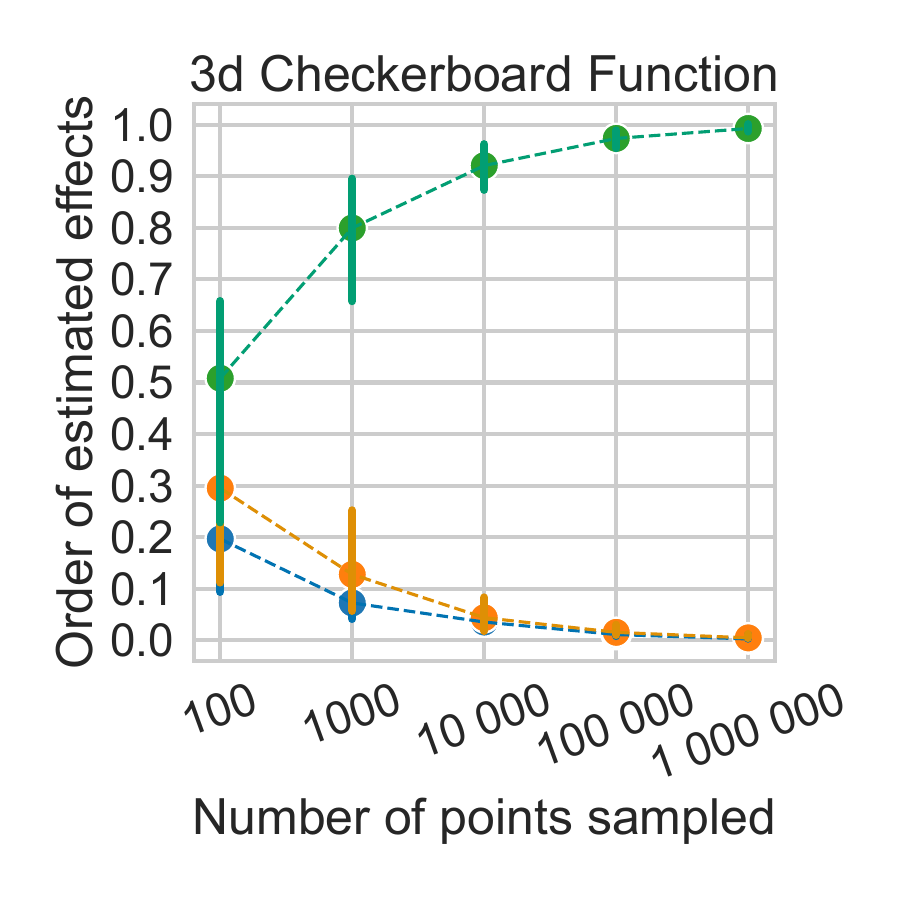}
    \includegraphics[width=0.275\textwidth]{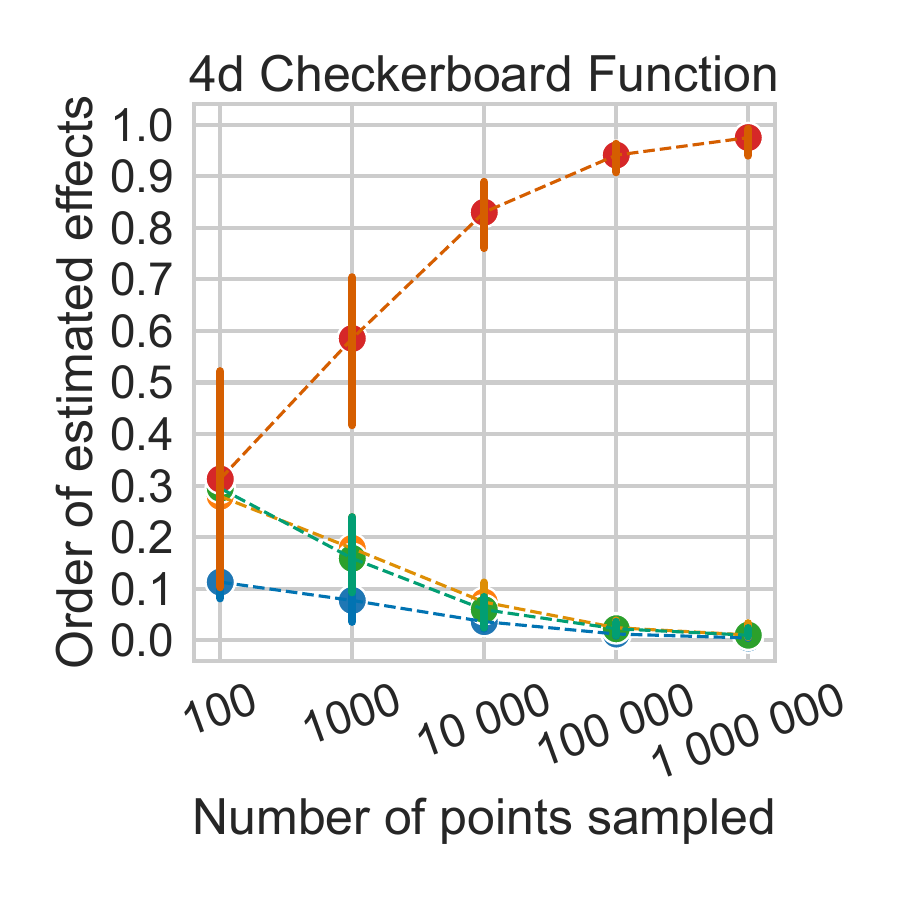}
    \includegraphics[width=0.275\textwidth]{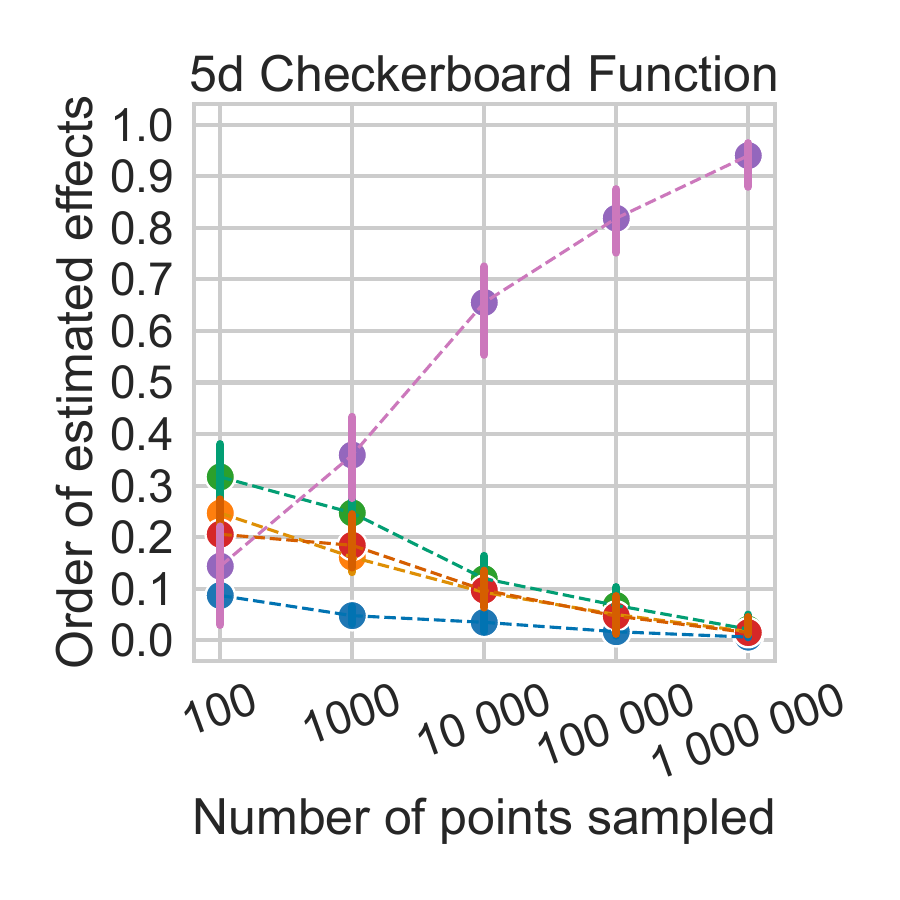}
    \includegraphics[width=0.275\textwidth]{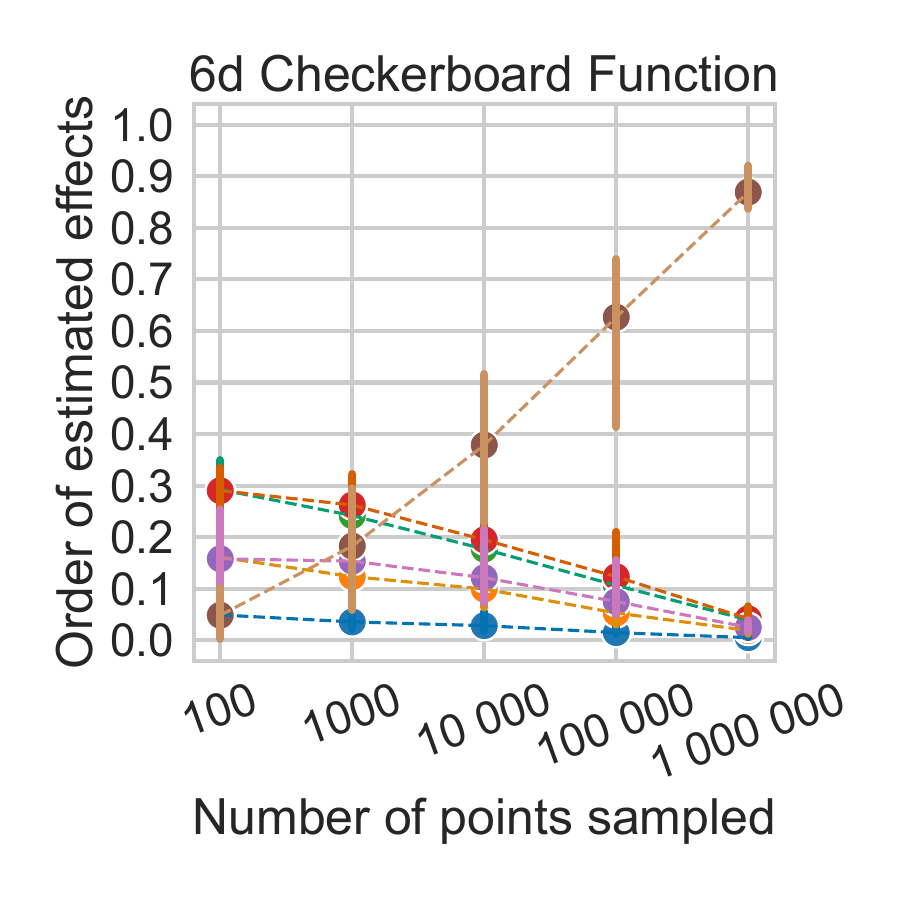}
    \includegraphics[width=0.275\textwidth]{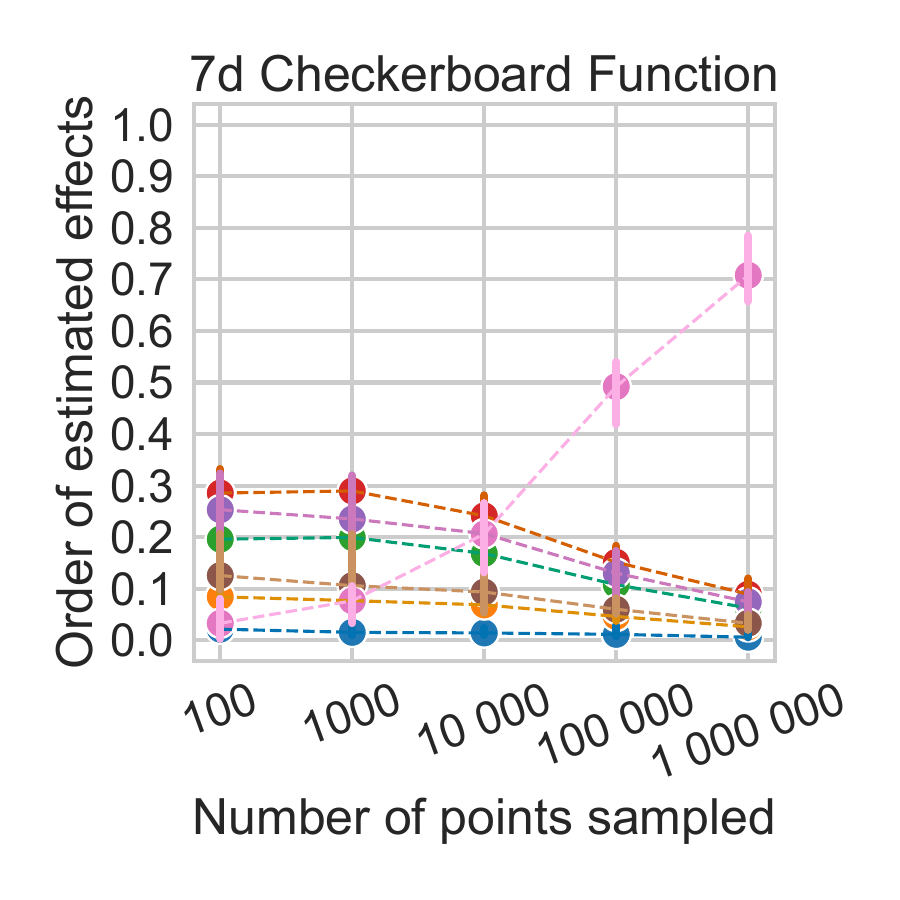}
    \includegraphics[width=0.275\textwidth]{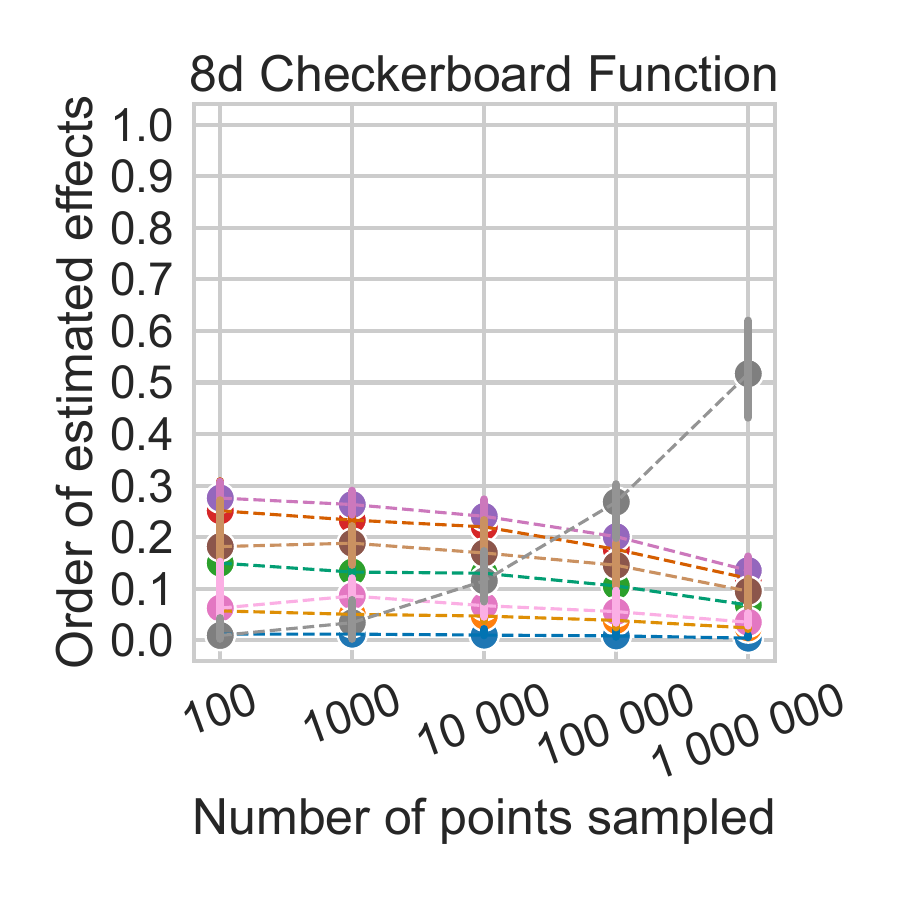}
    \includegraphics[width=0.275\textwidth]{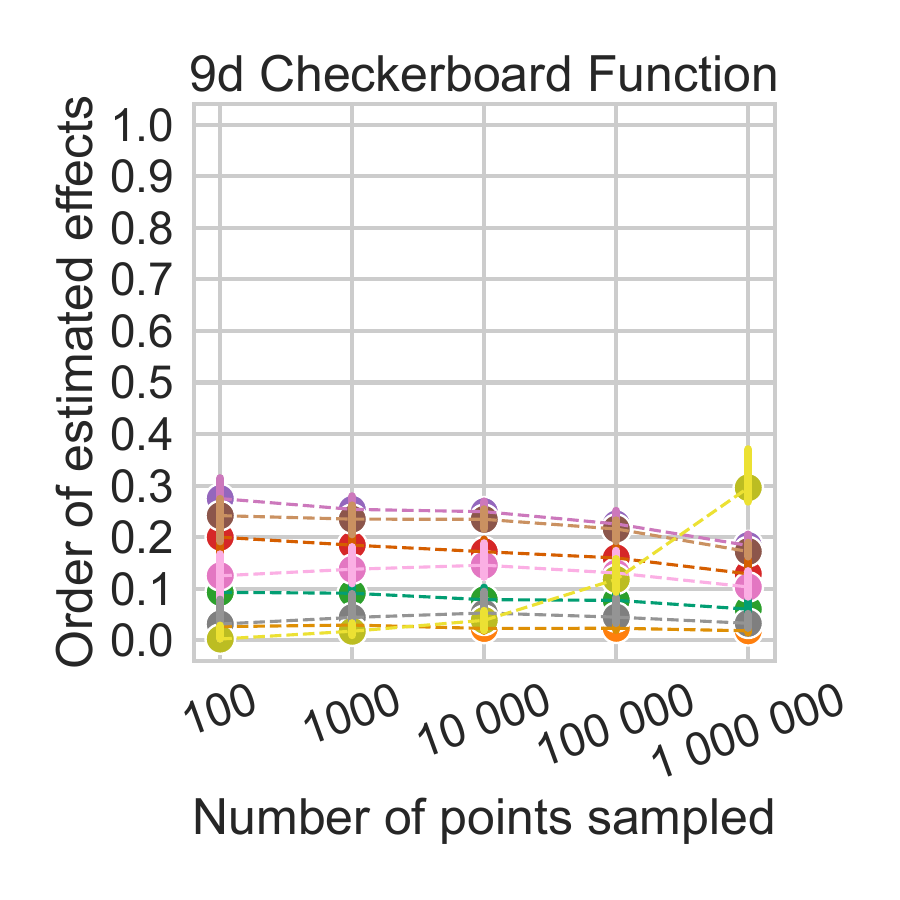}
    \includegraphics[width=0.275\textwidth]{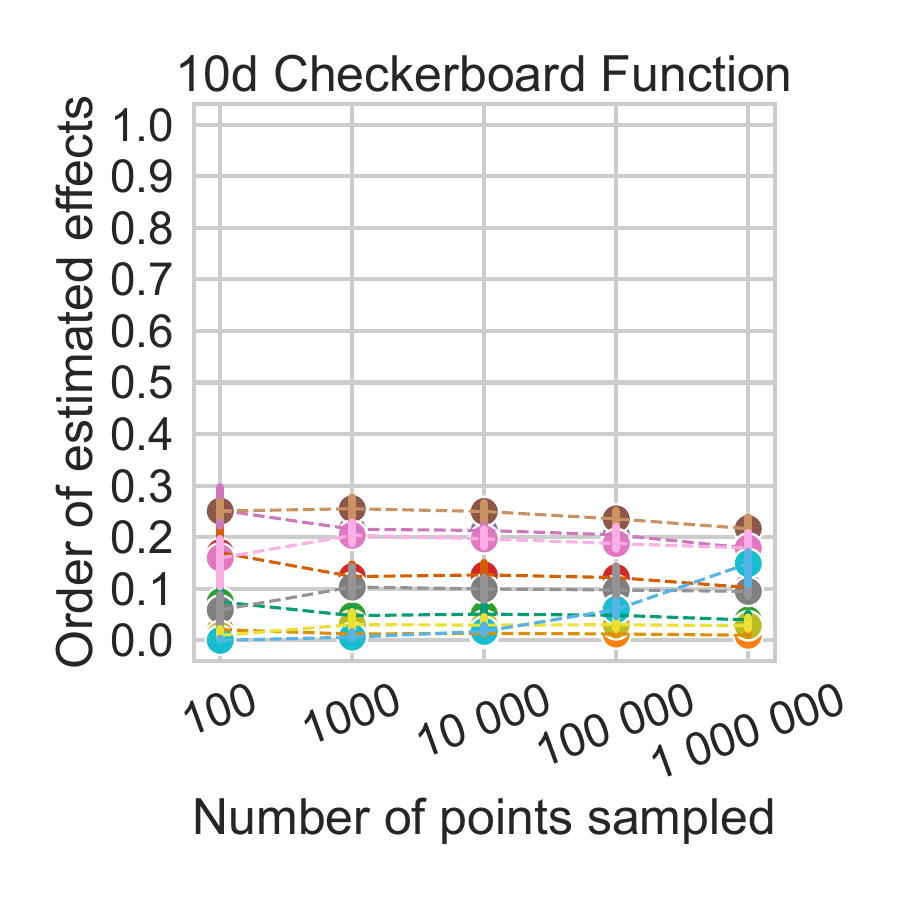}
    \includegraphics[width=0.75\textwidth]{figures/main_paper/legend.pdf}
    \caption{Estimating higher-order variable interactions requires precise evaluations of the value function. A simple way to study this is by estimating the $k$-dimensional checkerboard function \eqref{eq:checkerboard_fn}. Compare Figure \ref{fig:recovery} in the main paper.}
    \label{fig:apx_estiation}
\end{figure}

Due to the large number of terms involved in $n$-Shapley Values of higher order, visualizing these explanations is difficult.
However, Proposition \ref{prop:n_shapley_recursion} (which is really a variant of Theorem \ref{thm:thm_shapley_values_from_gam}) states that higher-order variable interactions in $n$-Shapley Values are related to the original Shapley Values via a simple lump-sum formula. This gives rise to the idea of simply visualizing, for each feature, the respective components of the sum. 

To illustrate this idea, let us consider a simple example. Let us begin with four different features and the usual Shapley Values. Say the first two features have attribution zero, the third feature has attribution $0.2$, and the fourth feature has attribution $-0.1$. These Shapley Values can be visualized as usual, depicted in Figure \ref{fig:e1}. Now, let us add a second-order interaction effect, say $\Phi_{2,3}^2=0.1$. Because this interaction effect would ultimately be added to the attributions of feature 2 and feature 3 with a factor of $\frac{1}{2}$, let us simply add two corresponding bars to the attributions of these features, with the color indicating that it is a second-order effect. From the resulting Figure \ref{fig:e2}, it can then be seen that we have two main effects and a single positive interaction effect between features 2 and 3. If there were another interaction effect, say $\Phi_{3,4}^2=-0.1$, we would proceed in the same way, taking care of the sign. From the resulting Figure \ref{fig:e3}, it can be seen that there are two main effects and a number of second-order interactions. With higher-order interactions we proceed accordingly, as illustrated for $\Phi_{2,3,4}^3=0.1$ (Figure \ref{fig:e4}) and $\Phi_{1,2,3,4}^4=-0.1$ (Figure \ref{fig:e4}).

Note that while this form of visualization faithfully depicts the relative magnitude of the different variable interactions, it is in general not possible to tell from the figures which variables interact with each other, for example when there are a number of different second-order effects.

\section{Estimating $n$-Shapley Values}
\label{apx:estimation}

Here we collect some additional details regarding the estimation of $n$-Shapley Values. We note that the discussion here is not exhaustive. Our objective is to (1) raise awareness for the fact that computing $n$-Shapley Values incurs an estimation problem, and (2) ensure that the results presented in the main paper are precisely estimated and not statistical artifacts. 

Figure \ref{fig:apx_estiation} depicts the result of estimating the $k$-dimensional checkerboard function \eqref{eq:checkerboard_fn} for all values $k=2,\dots,10$ (compare Section \ref{sec:estimating} in the main paper). As already discussed in the main paper, we can see from the figure that estimation becomes gradually harder as we increase the order of interaction. 

In Figure \ref{fig:apx_estiation_visualization}, we assess the degree up to which our visualizations are effected by the presence of spurious interaction effect of intermediate order, as observed when estimating a checkerboard function with too few samples. The figure visualizes the Shapley-GAM decomposition of a kNN classifier on the Folktables Travel data set, estimated with 500, 5000 and 133549 samples per evaluation of the value function, respectively. By comparing the left and middle part of Figure \ref{fig:apx_estiation_visualization} (estimation with 500 and 5000 samples, respectively), we see that 500 samples are to few and result in the presence of spurious interaction effects, for example of of order 4 and 5. This can be seen from the fact that some of these effects vanish as we increase the number of samples. By comparing the middle and right part of Figure \ref{fig:apx_estiation_visualization} (estimation with 5000 and 133549 samples, respectively), we see that estimation with 5000 samples is already quite precise for this kNN classifier. This can be seen from the fact that significantly increasing the number of samples does not have any significant effect on the visualization.\footnote{This could of course be discussed much more rigorously.}

Table \ref{tab:estimation_table} depicts the individual terms that underlie the visualization in Figure \ref{fig:apx_estiation_visualization}. From Table \ref{tab:estimation_table}, we see that main effects are precisely estimated even with 500 samples. However, many relatively small higher-order coefficients are not very precisely estimated even for $N=5000$. Note that the latter point is not in contrast to the fact that Figure \ref{fig:apx_estiation_visualization} is precisely estimated for $N=5000$. Figure \ref{fig:apx_estiation_visualization} depicts summary statistics that are more precisely estimated than the individual components. 

\begin{figure}[t]
    \centering
    \includegraphics[width=0.28\textwidth]{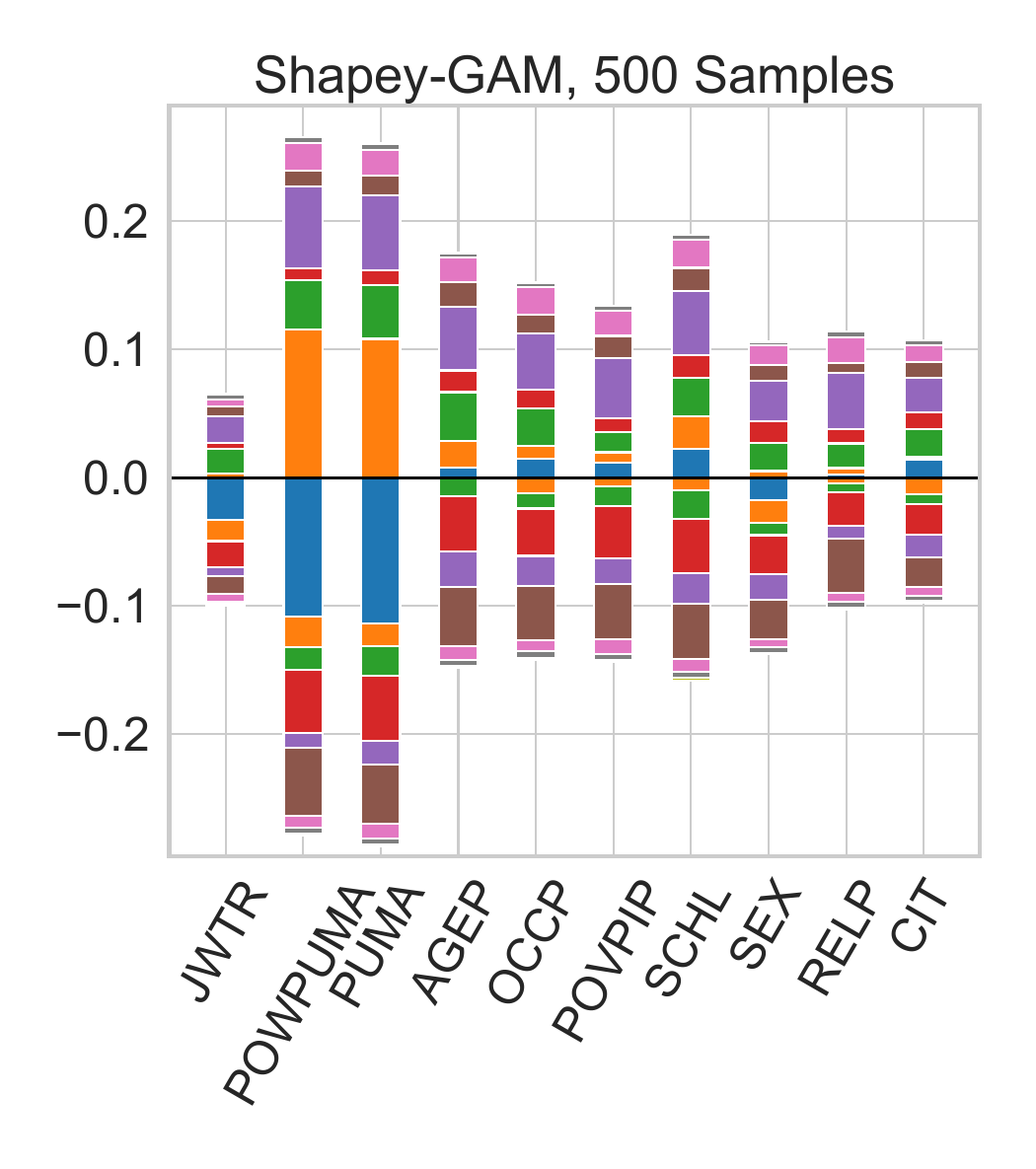}
    \includegraphics[width=0.28\textwidth]{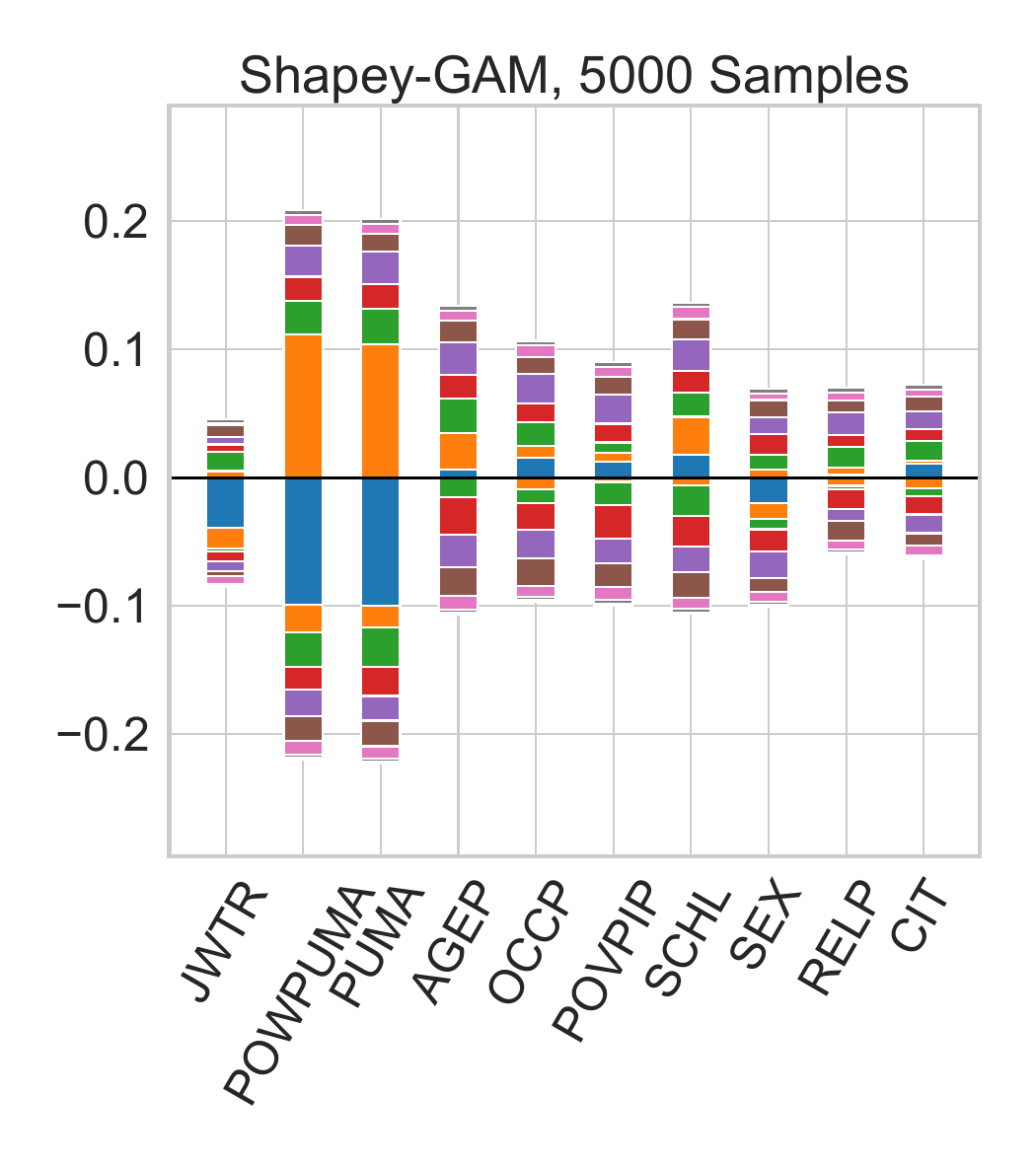}
    \includegraphics[width=0.28\textwidth]{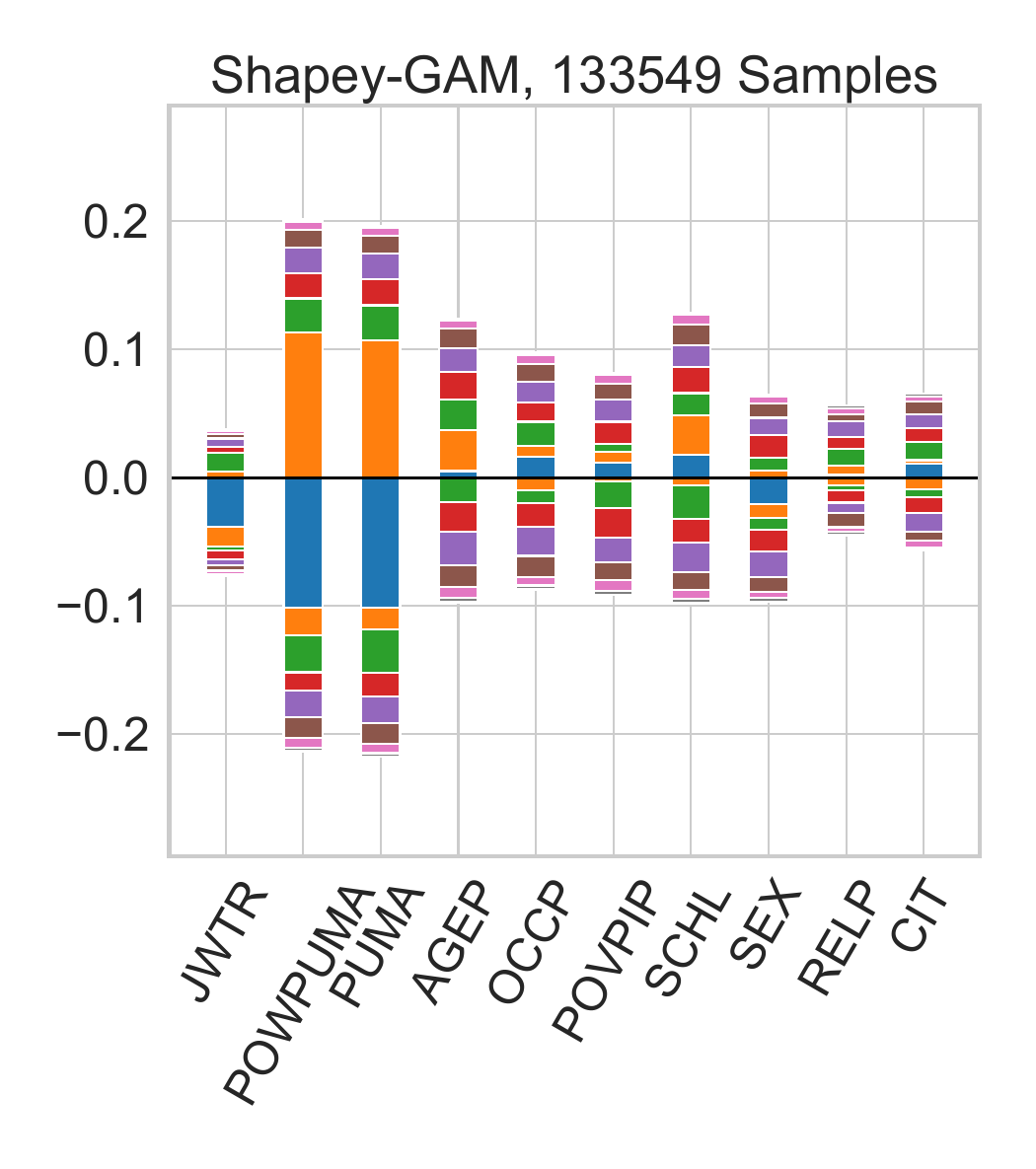}
    \includegraphics[width=0.8\textwidth]{figures/main_paper/legend.pdf}
    \caption{Estimating higher-order interactions with too few samples can result in spurious interaction effects of intermediate order. These effects are also visible in our visualizations. \textbf{Left:} Estimation with 500 samples per evaluation of the value function results in spurious interaction effects. \textbf{Middle:} This can be seen from the fact that parts of the estimated effects vanish if we increase the number of samples to 5000 per evaluation of the value function. \textbf{Right:} Using all 133549 observations in the training data per evaluation of the value function, we get almost the same visualization as for 5000 samples. The function in this example is a kNN classifier and the data set is the Folktables Travel data set.}
    \label{fig:apx_estiation_visualization}
\end{figure}

\section{The Statistical Independence Assumption for Observational SHAP is Necessary}
\label{apx:observational_shap}

In this section we give a simple example to demonstrate that the assumption of independent random variables for the observational SHAP value function in Theorem \ref{thm:thm_recovery} is indeed necessary. 

Consider the GAM of order 1
\begin{equation*}
f(x_1,x_2)=x_1+x_2.
\end{equation*}
Assume that $x_1$ and $x_2$ are correlated normal random variables
\begin{equation*}
    \begin{pmatrix}
  x_1 \\
  x_2
 \end{pmatrix}\sim\mathcal{N}\left(\begin{pmatrix}
  0 \\
  0
 \end{pmatrix},\begin{pmatrix}
  1,\,\, \rho \\
  \rho,\,\, 1
 \end{pmatrix}\right)
\end{equation*}
with $0\leq\rho\leq 1$. We have
\begin{equation*}
    \mathbb{E}[x_2|x_1]=\rho x_1.
\end{equation*}
A simple calculation shows that the Shapley-GAM of observational SHAP is given by
\begin{equation*}
    f_\emptyset=0,\quad f_1(x_1)=(1+\rho)x_1,\quad f_2(x_2)=(1+\rho)x_2,\quad f(x_1,x_2)=-\rho(x_1+x_2).
\end{equation*}
According to Theorem \ref{thm:thm_shapley_values_from_gam}, the observational SHAP values are then given by
\begin{equation*}
    \Phi_1=(1+\frac{\rho}{2})x_1-\frac{\rho}{2} x_2,\qquad \Phi_2=(1+\frac{\rho}{2})x_2-\frac{\rho}{2} x_1.
\end{equation*}
Clearly, recovery does not hold: Despite the fact that the underlying function is a GAM of order 1, the Shapley-GAM is a GAM of order 2. The Shapley Values also depend on both coordinates -- hence they are not well-defined functions of the individual coordinates. 

In contrast, the Shapley-GAM of the interventional SHAP value function is given by
\begin{equation*}
    f_\emptyset=0,\quad f_1(x_1)=x_1,\quad f_2(x_2)=x_2.
\end{equation*}
Moreover, the interventional SHAP values are given by
\begin{equation*}
    \Phi_1=x_1,\qquad \Phi_2=x_2,
\end{equation*}
that is recovery holds with the interventional SHAP value function (as guaranteed by Theorem \ref{thm:thm_recovery}).

\section{Proof of Theorem \ref{thm:shapley_gam}}

\begin{proof}[Proof of Theorem \ref{thm:shapley_gam}]
We are going to show that
\begin{equation}
\label{eq:proof_thm3_0}
    \Phi_S^d(x)=\sum_{L\subset S}(-1)^{|S|-|L|}v(x_{L},L).
\end{equation}
Note that the RHS evaluates the value function $v$ only for sets $L\subset S$. From the assumption that the value function is subset-compliant, it follows that the RHS is a well-defined function of $x_S$. According to Proposition \ref{prop:efficiency} (efficiency), the $d$-Shapley Values sum to $v(x)-v(\emptyset)$ which implies the Theorem.

To show \eqref{eq:proof_thm3_0}, we consider the non-recursive definition of $n$-Shapley Values \ref{eq:apx_nshapley_explicit} and then substitute the definition of $\Delta_S(x)$ from Definition \ref{def:n_shapley_values}.
\begin{equation}
\label{eq:proof_thm3_1}
\begin{split}
    \Phi_S^d(x)&=\sum_{k=0}^{d-|S|}\sum_{K\subset [d]\setminus S,\,|K|=k} B_k\,\Delta_{S\cup K}(x)  \\
               &=\sum_{k=0}^{d-|S|}\sum_{K\subset [d]\setminus S,\,|K|=k} B_k\sum_{T\subset[d]\setminus (S\cup K)}\frac{(d-|T|-|S|-|K|)!|T|!}{(d-|S|-|K|+1)!}\sum_{L\subset S\cup K}(-1)^{|S|+|K|-|L|}v(x,\,L\cup T).  \\[6pt]
               &=\sum_{K\subset [d]\setminus S}\,\,\,\sum_{T\subset[d]\setminus (S\cup K)}B_{|K|}\,\frac{(d-|T|-|S|-|K|)!|T|!}{(d-|S|-|K|+1)!}\sum_{L\subset S\cup K}(-1)^{|S|+|K|-|L|}v(x,\,L\cup T).  \\
\end{split}
\end{equation}
Where the last equation follows from the realization that we are summing over all possible subsets of $[d]\setminus S$.

In equation \eqref{eq:proof_thm3_1}, we are summing over the value of the same sets multiple times. Let us fix a set $M=L\cup T$ and count how often it occurs in the sum. 
First note that $v(x,M)$ occurs exactly once for every set $K$, namely by choosing $T=M\setminus(S\cup K)$ and $L=M\cap (S\cup K)$. Since the coefficients do not only depend on the size of $K$, but also on $|T|$ and $|L|$, let us partition the set $K=K_1\cup K_2=\{K\cap M\}\cup\{K \setminus M\}$. Let $n_1=|M\setminus S|$ and $n_2=|[d]\setminus(S\cup M)|$ denote the maximum sizes of both partitions. With this counting argument, we arrive at
\begin{equation}
\label{eq:proof_thm3_2}
     (-1)^{|S|-|M|}\,\sum_{K_1\subset M\setminus S}\,\,\sum_{K_2\subset [d]\setminus (S\cup M)}B_{|K_1|+|K_2|}\,\frac{(n_2-|K_2|)!(n_1-|K_1|)!}{(n_1+n_2-|K_1|-|K_2|+1)!}(-1)^{|K_2|}
\end{equation}
occurrences of the term $v(x,M)$. Notice that equation \eqref{eq:proof_thm3_2} is equal to
\begin{equation}
\label{eq:proof_thm3_3}
(-1)^{|S|-|M|}\sum_{k_1=0}^{n_1}\sum_{k_2=0}^{n_2}\binom{n_1}{k_1}\binom{n_2}{k_2}\frac{(n_2-k_2)!(n_1-k1)!}{(n_1+n_2-k_1-k_2+1)!}(-1)^{k_2}B_{k_1+k_2}
\end{equation}
The desired result now follows from the properties of the Bernoulli numbers. In particular, since $M\subset S \iff n_1=0$, we see from Lemma \ref{bernoulli_lemma_2} that \eqref{eq:proof_thm3_3} equals $(-1)^{|S|-|M|}$ if $M\subset S$ and $0$ otherwise. Comparing the terms for all possible sets $M\subset [d]$, we see that  \eqref{eq:proof_thm3_1} equals \eqref{eq:proof_thm3_0}.

Note that if we fix the point $x$, then the Shapley-GAM at $x$ is equivalent to the Moebious transform of the measure $v(x,\cdot)$. From this perspective, Theorem \ref{thm:shapley_gam} can be seen as an application of Theorem 2 in \citet{grabisch1997k}.
\end{proof}

\section{Proof of Theorem \ref{thm:thm_shapley_values_from_gam}}

\begin{proof}[Proof of Theorem \ref{thm:thm_shapley_values_from_gam}]
According to Theorem \ref{thm:shapley_gam}, the $d$-Shapley Values can be written as
\begin{equation}
    \Phi_S^d(x)=f_S(x)
\end{equation}
where $f_S(x)$
are the component functions of the Shapley-GAM. Hence, the $d$-Shapley Values are a linear combination of the component functions of the Shapley-GAM. 
From the recursive definition of the $n$-Shapley Values, we see that
\begin{equation}
\label{eq:proof_thm4_eq_1}
\Phi_S^{n}(x)=\Phi_S^{n+1}(x)-B_{1+n-|S|}\sum_{K\subset[d]\setminus S,|K|+|S|=n+1}\Phi_{S\cup K}^{n+1}(x)
\end{equation}
that is the $n$-Shapley Values are a linear combination of the terms involved in the $n+1$-Shapley Values. By induction, we see that the $n$-Shapley Values are linear combinations of the component functions of the Shapley-GAM.

It remains to determine the coefficients $C_{n,m}$. We present a counting argument that is based on the recurrence relation \eqref{eq:proof_thm4_eq_1}. In this counting argument, we first determine the coefficients $D_{n,m}$ where the first index corresponds to the distance between $|S|$ and the order of the Shapley Values, and the second index corresponds to the different between the size of the interaction effect and the order of the Shapley Values. Suppose that we are computing $n$-Shapley Values. If we use equation \eqref{eq:proof_thm4_eq_1} to proceed recursively from $d$-Shapley Values to $n$-Shapley Values, then the first time that the component function $f_{S\cup K}$ is being added to $\Phi_S^{m}$ is during the computation of the $(|S|+|K|-1)$-Shapley Values. According to equation \eqref{eq:proof_thm4_eq_1}, the linear coefficient will simply be
$D_{|K|-1,1}=-B_{|K|}$. 
The second time that the component function $f_{S\cup K}$ is being added to $\Phi_S^{m}$ is during the computation of the $(|S|+|K|-2)$-Shapley Values. This is because we have previously added $-B_1 f_{S\cup K}$ to all the terms of order $|S|+|K|-1$ that are a subset of $S\cup K$. There are $\binom{|K|}{1}$ such terms, and we are now adding all of them to $f_{S}$, using the coefficient $-B_{|K|-1}$. This means that we arrive at a total coefficient of
\begin{equation}
     D_{|K|-2,2}=-B_{|K|}+B_{|K|-1}\binom{|K|}{1}B_1.
\end{equation}
By a similar argument we arrive at a coefficient of 
\begin{equation}
     D_{|K|-3,3}=-B_{|K|}+B_{|K|-1}\binom{|K|}{1}B_1-B_{|K|-2}\binom{|K|}{2}B_2-B_{|K|-2}\binom{|K|}{2}B_1 \binom{2}{1}B_1.
\end{equation}
for the $(|S|+|K|-3)$-Shapley Values.
In general, that is when we compute $n$-Shapley Values, the component function $f_{S\cup K}$ is being added to $\Phi_S^{n}$ once for every possible pathway that goes from a set of order $n+1$ to the set $S\cup K$ by successively adding different numbers of elements. For $k\geq 1$, let 
\begin{equation}
    P_k=\left\{(p_1,\dots,p_k)\in\mathbb{N}_{\geq 0}^k\,\,\bigg|\,\,\sum_{i=1}^k p_i=k\quad\text{and}\quad p_{i}=0\implies (p_{j}=0\,\forall j>i)\right\}
\end{equation}
be the set of pathways of length $k$. This means that we have $P_1 = \big\{(1)\big\}$,
\begin{equation}
\begin{split}
    P_2 &= \big\{(2,0),(1,1)\big\},\\[2pt]
    P_3 &= \big\{(3,0,0),(2,1,0),(1,2,0),(1,1,1)\big\},\\[2pt]
    P_4 &= \big\{(4,0,0,0),(3,1,0,0),(2,2,0,0),(2,1,1,0),\\
    &\quad\quad(1,3,0,0),(1,2,1,0),(1,1,2,0),(1,1,1,1)\big\}\\
\end{split}
\end{equation}
and so on. 
By accounting for the coefficients $B_k$ and the signs along each path, the coefficients can be written as
\begin{equation}
\label{eq:beta_k_m}
    D_{n,m}=\sum_{(p_1,\dots,p_{m})\in P_{m}}(-1)^{\sum_{i=1}^m \text{sign}(p_i)}\binom{n+m}{n+p_1}B_{n+p_1}\prod_{i=2}^m B_{p_i}\binom{m-\sum_{j=1}^{i-1}p_j}{p_i}
\end{equation}
From this, we derive the special case
\begin{equation}
\begin{split}
    D_{0,m}&=\sum_{(p_1,\dots,p_{m})\in P_{m}}(-1)^{\sum_{i=1}^m \text{sign}(p_i)}\binom{m}{i_1}B_{p_1}\prod_{i=2}^m B_{p_i}\binom{m-\sum_{j=1}^{i-1}p_j}{p_i}\\
                &=\sum_{(p_1,\dots,p_{m})\in P_{m}}(-1)^{\sum_{i=1}^m \text{sign}(p_i)}\prod_{i=1}^m B_{p_i}\binom{m-\sum_{j=1}^{i-1}p_j}{p_i}\\
               &=-B_{m}-\sum_{p_1=1}^{m-1}a_{p_1}\binom{m}{p_1}\sum_{(\hat p_1,\dots,\hat p_{m-p_1})\in P_{m-p_1}}\,\,(-1)^{\sum_{i=1}^{m-p_1} \text{sign}(p_i)}\prod_{j=1}^{m-p_1}B_{\hat p_j}\binom{m-i_1-\sum_{s=1}^{j-1}\hat p_s}{\hat p_j}\\
               &=-B_{m}- \sum_{p_1=1}^{m-1}a_{p_1}\binom{m}{p_1}\beta_{0,m-p_1}\\
               &=-B_m-\sum_{p_1=1}^{m-1}a_{p_1}\binom{m}{p_1}\frac{1}{m-p_1+1}\\
               &=-\sum_{k=1}^{m}\frac{B_k}{m-k+1}\binom{m}{k}\\
               &=\frac{1}{m+1}\\
\end{split}
\end{equation}
where the last equality is due to Lemma \ref{bernoulli_lemma_1}. Now, this implies that
\begin{equation}
\label{eq:proof_thm_4_3}
    \Delta_S(x) = \Phi_S^{|S|}(x)
                = f_S(x)+\sum_{K\subset[d]\setminus S,\,\, |K|\geq 1}D_{0,|K|}\,f_{S\cup K}(x)
                =\sum_{K\subset[d]\setminus S} \frac{1}{1+|K|}f_{S\cup K}(x)
\end{equation}
which is a version of Theorem 1 in \citet{grabisch1997k}. Using \eqref{eq:proof_thm_4_3} and the explicit formula for $n$-Shapley Values \eqref{eq:apx_nshapley_explicit}, we get
\begin{equation}
\begin{split}
      \Phi_S^n(x)&=\sum_{k=0}^{n-|S|}\sum_{K\subset [d]\setminus S,\,|K|=k} B_k\,\Delta_{S\cup K}(x)\\
    &=\sum_{k=0}^{n-|S|}\sum_{K\subset [d]\setminus S,\,|K|=k} B_k\,\sum_{T\subset[d]\setminus (S\cup K)} \frac{1}{1+|T|}f_{S\cup K\cup T}(x)
\end{split}
\end{equation}
From which we see that the component function $f_{S\cup \tilde K}$ is being added to $\Phi_S^n(x)$ exactly
\begin{equation}
C_{n-|S|,|\tilde K|}=\sum_{k=0}^{n-|S|}\binom{n-|S|}{k}\frac{B_k}{1+|\tilde K|-k}
\end{equation}
times which concludes the proof.
\end{proof}

\section{Proof of Theorem \ref{thm:value_function_from_gam}}

\begin{proof}[Proof of Theorem \ref{thm:value_function_from_gam}]
According to Theorem \ref{thm:shapley_gam}, the Shapley-GAM decomposition is given by
\begin{equation}
    f_S(x)=\sum_{L\subset S}(-1)^{|S|-|L|}v(x_L,L).
\end{equation}
By substituting the definition of the value function \eqref{eq:gam_value_function}
\begin{equation}
\begin{split}
    f_S(x)&=\sum_{L\subset S}(-1)^{|S|-|L|}v(x_L,L)\\
          &=\sum_{L\subset S}(-1)^{|S|-|L|}\sum_{T\subset L}g_T(x)\\
          &=\sum_{L\subset S}\sum_{T\subset L}g_T(x)(-1)^{|S|-|L|}\\
          &=\sum_{T\subset S}g_T(x)\sum_{L\subset S\setminus T}(-1)^{|S|-|L|-|T|}\\
          &=g_S(x)\\
\end{split}
\end{equation}
Where we have re-arranged the sum to count the number of occurrences of the set $T$, and then used the fact that inner sum averages to zero except for $T=S$.
\end{proof}

\section{Proof of Theorem \ref{thm:thm_recovery}}

We show a slightly more general result than what is stated in the main paper. In fact, we show that recovery holds for all interaction indices that can be written as 
\begin{equation}
    I_S^n(x)=f_S(x)+\sum_{\substack{K\subset[d]\setminus S\\n+1\leq |S|+|K|}} C_{n,|S|,|K|}\,f_{S\cup K}(x)  \qquad \forall S\subseteq[d],|S|\leq n
\end{equation}
where $f_S(x)$ are the component functions of the Shapley-GAM and $C_{n,|S|,|K|}\in\mathbb{R}$ are coefficients that depend on the  interaction index. $n$-Shapley Values can be written like this according to Theorem \ref{thm:thm_shapley_values_from_gam}. For the Faith-Shap interaction index, this representation is given in Theorem 19 in \citet{tsai2022faith}
\begin{equation}
\label{eq:faith_shap}
    \text{Faith-Shap}^n_S(x)=f_S(x)+\sum_{\substack{K\subset[d]\setminus S\\n+1\leq |S|+|K|}}(-1)^{n-|S|}\frac{|S|}{n+|S|}\frac{\binom{n}{|S|}\binom{|S|+|K|-1}{n}}{\binom{|S|+|K|+n-1}{n+|S|}}  \,f_{S\cup K}(x)\qquad\forall |S|\leq n.
\end{equation}
Also the Shapley Taylor interaction index \citep{sundararajan2020shapley} can, due to its symmetry, be written as
\begin{equation}
\label{eq:shapley_taylor}
    \text{Shapley-Taylor}^n_S(x)=\begin{cases}
    f_S(x)\qquad\qquad\qquad\qquad\qquad\qquad&\text{if }|S|<n\\[4pt]
f_S(x)+\sum_{\substack{K\subset[d]\setminus S\\n+1\leq |S|+|K|}} \frac{1}{\binom{|S|+|K|}{|K|}}\,f_{S\cup K}(x)  \qquad &\text{if }|S|=n.
    \end{cases}
\end{equation}

\begin{proof}[Proof of Theorem \ref{thm:thm_recovery}]
We assume that the function $f$ can be written as a GAM of order $n$, that is
\begin{equation}
    f(x)=\sum_{S\subset[d],\,|S|\leq n}g_S(x_S).
\end{equation}
Notice that this GAM is not necessarily the Shapley-GAM, but just some way to write the function $f$ as a GAM. Let $f_S$ be the component functions of the Shapley-GAM. Now,
$n$-Shapley Values, the Faith-Shap interaction index, as well as the Shapley Taylor interaction index, can be written as a linear combination of the component functions of the Shapley-GAM
\begin{equation}
\label{eq:proof_recovery_0}
    I_{S}^n(x)=f_S(x_S)+\sum_{K\subset [d]\setminus S,\,\, |S|+|K|>n}\,\,C_{\scaleto{n-|S|,|K|}{8pt}}\,f_{S\cup K}(x_{S\cup K})
\end{equation}
where the specific linear coefficients $C_{n,m}$ depend on the interaction index (Theorem \ref{thm:thm_shapley_values_from_gam}, equation \eqref{eq:faith_shap}, equation \eqref{eq:shapley_taylor}). According to equation \eqref{eq:proof_recovery_0}, 
the interaction index equals $f_S(x_S)$ plus some weighted components of the Shapley-GAM of order greater than $n$. As a consequence, it remains to show is that the Shapley-GAM is a GAM of order $n$ (then the second sum vanishes and we arrive at $I_S^n(x)=f_S(x_S)$ which is what we want to show).

It remains to show that the Shapley-GAM is a GAM of order $n$. According to Theorem \ref{thm:shapley_gam}, the component functions of the Shapley-GAM are given by
\begin{equation}
    f_S(x)=\sum_{L\subset S}(-1)^{|S|-|L|}v(x_{L},L).
\end{equation}
We want to show that the component functions of degree greater than $n$ vanish. Let us first consider observational SHAP. Here we have
\begin{equation}
\begin{split}
    \sum_{L\subset S}(-1)^{|S|-|L|}v(x_{L},L)&=\sum_{L\subset S}(-1)^{|S|-|L|}\mathbb{E}[f(x)|x_L]\\
                                             &=\sum_{L\subset S}(-1)^{|S|-|L|}\mathbb{E}\left[\sum_{T\subset[d],\,|T|\leq n}g_T(x_T)\Big|x_L\right]\\
                                             &=\sum_{L\subset S}(-1)^{|S|-|L|}\sum_{T\subset[d],\,|T|\leq n}\mathbb{E}\left[g_T(x_T)|x_L\right]\\
                                             &=\sum_{T\subset[d],\,|T|\leq n}\,\,\sum_{L\subset S}(-1)^{|S|-|L|}\mathbb{E}\left[g_T(x_T)|x_L\right]\\        
\end{split}
\end{equation}
Consider the inner sum. If $|S|>n$, we can always pick an element $i\in S\setminus T$ and write
\begin{equation}
\label{eq:proof_recovery_1}
    \sum_{L\subset S\setminus\{i\}}(-1)^{|S|-|L|}\Big(\mathbb{E}\left[g_T(x_T)|x_{L}\right]-\mathbb{E}\left[g_T(x_T)|x_{L\cup\{i\}}\right]\Big)
\end{equation}
If the input features are independent, then $g_T(x_T)$ and $x_i$ are independent, from which we get by the properties of the conditional expectation that
\begin{equation}
    \mathbb{E}\left[g_T(x_T)|x_{L\cup\{i\}}\right]=\mathbb{E}\left[g_T(x_T)|x_{L}\right]
\end{equation}
It follows that the inner sum is zero for all sets $T$, and that the component functions of the Shapley-GAM of degree greater than $n$ are equal to zero, too. 

Let us now consider interventional SHAP. Just as for observational SHAP, we arrive at equation \eqref{eq:proof_recovery_1} using the linearity of the expectation operator. Hence, we require that
\begin{equation}
    \mathbb{E}\left[g_T(x_T)|do(x_{L\cup\{i\}})\right]=\mathbb{E}\left[g_T(x_T)|do(x_{L})\right]
\end{equation}
which follows from the properties of the causal do-operator. Intuitively, since $g_T$ does not depend on the value of feature $i$, intervening on that feature has no effect.
\end{proof}

\section{Proof of Lemma \ref{bernoulli_lemma_2}}
\label{apx:bernoulli_lemma}

\begin{proof}%
Let us first consider the case $n=0$. For $n=0$ and $m=0$, we have\begin{equation}
    \binom{0}{0}\binom{0}{0}\frac{(0-0)!(0-0)!}{(0+0-0-0+1)!}(-1)^0 B_0=1.
\end{equation}For $n=0$ and $m\geq 1$, we have 
\begin{equation}
\begin{split}
    \sum_{l=0}^{m}\binom{m}{l}\frac{1}{(m-l+1)}(-1)^l B_l&=\frac{1}{m+1}\sum_{l=0}^{m}\binom{m+1}{l}(-1)^l B_l\\
       &=\frac{-2}{m+1}\binom{m+1}{1}B_1+\sum_{l=0}^{m}\binom{m+1}{l}\\
       &=-2B_1+0=1.
\end{split}
\end{equation}
where we used \eqref{eq:bernoulli_numbers} and the fact that the odd Bernoulli numbers vanish except for $n=1$. For $m=0$ and $n\geq 1$, we also have from \eqref{eq:bernoulli_numbers}
\begin{equation}
\begin{split}
    \sum_{k=0}^{n}\binom{n}{k}\frac{1}{(n-k+1)}(-1)^0 B_k=\frac{1}{n+1}\sum_{k=0}^{n}\binom{n+1}{k}B_k=0.
\end{split}    
\end{equation}It remains to show the general case $n,m\geq 1$.
According to a derivation by \citet{mo_rene}, the problem in this case is equivalent to 
\begin{equation}
\label{eq:rene}
(-1)^n \sum_{l=
0}^m \frac{B_{n+l+1}}{n+l+1}{m \choose l}+(-1)^m \sum_{k=
0}^n \frac{B_{m+k+1}}{m+k+1}{n \choose k} = - \frac{1}{(n+m+1){n+m\choose m}}
\end{equation}
Now, Theorem 2 in \citet{gould2014bernoulli} with $s=1$ states that for any sequence of numbers $(a_n)_{n\geq 0}$, it holds that
\begin{equation}
\label{eq:gould_theorem}
    \sum_{k=0}^m\binom{m}{k}\frac{a_{n+k+1}}{n+k+1}=\sum_{k=0}^n(-1)^{n-k}\binom{n}{k}\frac{b_{m+k+1}}{m+k+1}+\frac{(-1)^{n+1}a_0}{(m+n+1)\binom{m+n}{n}}
\end{equation}
where the sequence $(b_n)_{n\geq 0}$ is the binomial transform of the sequence $(a_n)_{n\geq 0}$, given by
\begin{equation}
    b_n=\sum_{k=0}^n\binom{n}{k}a_k.
\end{equation}
Setting $a_n=B_n$, we have from \eqref{eq:bernoulli_numbers} that the binomial transform of the Bernoulli numbers is simply
\begin{equation}
    b_n=\sum_{k=0}^n\binom{n}{k}B_k=(-1)^n B_n
\end{equation}
where the factor $(-1)^n$ takes care of the special case $n=1$. Using \eqref{eq:gould_theorem} with $a_n=B_n$ and $b_n=(-1)^n B_n$, we get
\begin{equation}
\begin{split}
(-1)^n\sum_{k=0}^m\binom{m}{k}\frac{B_{n+k+1}}{n+k+1}&=-\sum_{k=0}^n(-1)^{m}\binom{n}{k}\frac{B_{m+k+1}}{m+k+1}-\frac{1}{(m+n+1)\binom{m+n}{n}}
\end{split}
\end{equation}
where we multiplied both sides with $(-1)^n$. This is the same as \eqref{eq:rene} which concludes the proof.
\end{proof}

\newpage
\section{Datasets and Models}
\label{apx:datasets_models}

In our experiments, we use the following data sets and models.

\subsection{Datasets}

{\bf Folktables Income.} Folktables is a Python package that provides access to data sets derived from recent US Censuses \url{https://github.com/zykls/folktables}. We used this package to obtain the data from the 2016 Census in California. The machine learning problem is the ACSIncome prediction task, that is to predict whether an individual's income is above \$50,000, based on 10 personal characteristics \citep{ding2021retiring}. The data set contains of 152 149 observations.

{\bf Folktables Travel Time.} Folktables is a Python package that provides access to data sets derived from recent US Censuses \url{https://github.com/zykls/folktables}. We used this package to obtain the data from the 2016 Census in California. The machine learning problem is the ACSTravelTime prediction task, that is to predict whether an individual has to commute to work longer than 20 minutes, based on 10 personal characteristics \citep{ding2021retiring}. The data set contains 133 549 observations.

{\bf German Credit.} The German Credit Data set is a  data set with 20 different features on individual's credit history and personal characteristic. The machine learning problem is to predict credit risk in binary form. 
We obtained the data set from the UCI machine learning repository and reduced the number of features to 10 without any observed drop in accuracy. The data set contains 1000 observations.

{\bf California Housing.} The California Housing data set was derived from the 1990 U.S. census. The regression problem is to predict the median house value, based on 8 characteristics. We obtained the data set form the \texttt{scikit-learn} library. The data set contains 20 640 observations.

{\bf Iris.} The Iris data set is a simple flower data set. The machine learning problem is to classify whether the flower is of a particular kind or not, based on 4 different features. We obtained the data set form the \texttt{scikit-learn} library. The data set contains 150 observations.

\subsection{Models}

{\bf Glassbox-GAM.} We train the Glassbox-GAMs with the \texttt{interpretML} library \citep{nori2019interpretml} and default parameters (no interactions).

{\bf Gradient Boosted Tree.} We use the \texttt{xgboost} library \citep{xgboost} and train with 100 trees per model. This setting allows to achieve competitive accuracy for gradient boosted trees.

{\bf Random Forest.} We use the \texttt{scikit-learn} library \citep{scikit-learn} and train with 100 trees per forest. This setting allows to achieve competitive accuracy for random forests.

{\bf k-Nearest Neighbor.} We use the \texttt{scikit-learn} library \citep{scikit-learn}. The hyperparameter $k$ was chosen with cross-validation to be 30, 80, 25, 10, 1 for the data sets as listed above.

\newpage
\section{Additional Plots and Figures}
\label{apx:figures}

\subsection{Folktables Income}

\subsubsection{Glassbox-GAM}
\begin{figure}[H]
    \centering
   \includegraphics[width=0.95\textwidth]{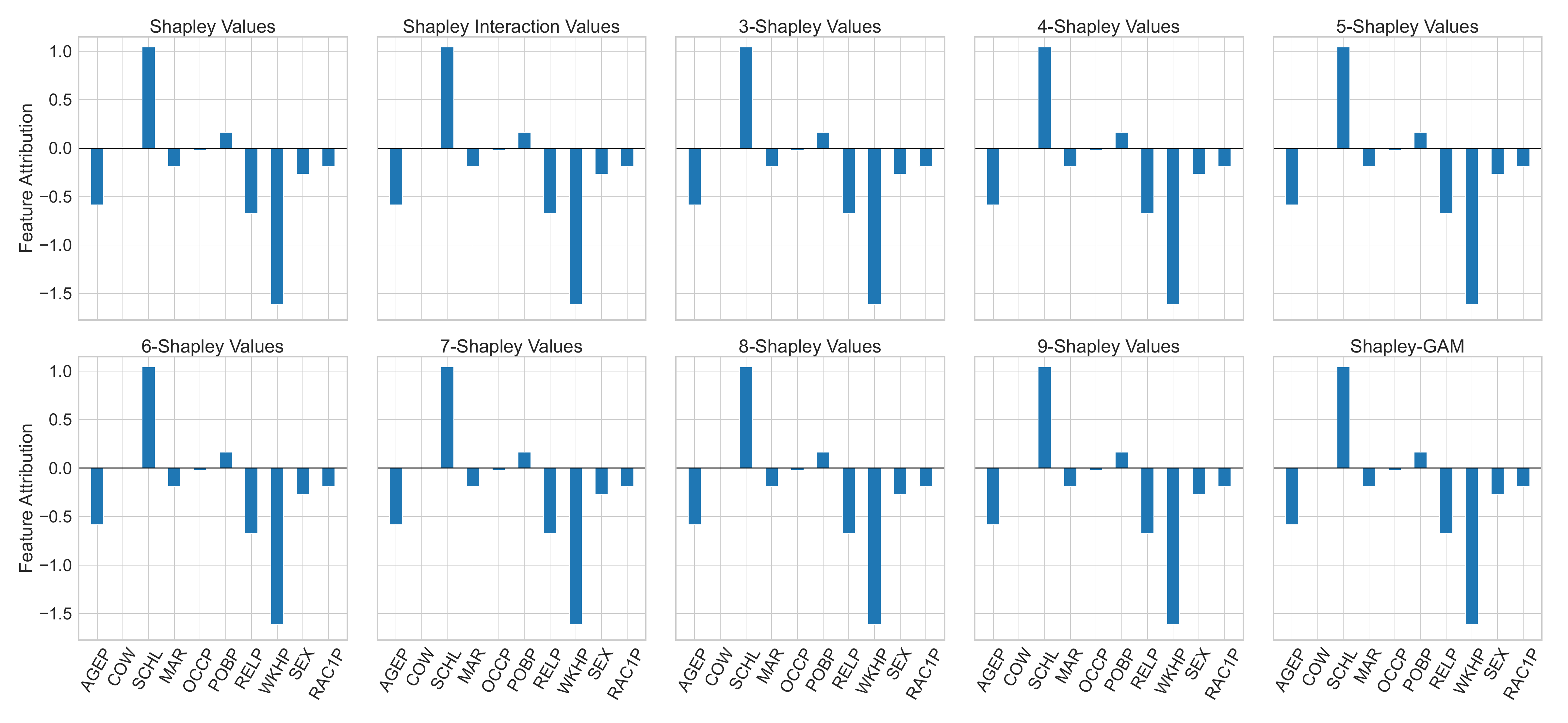}
   \includegraphics[width=0.85\textwidth]{figures/main_paper/legend.pdf}
    \caption{$n$-Shapley Values for a Glassbox-GAM and the first observation in our test set of the Folktables Income data set.}
\end{figure}

\subsubsection{Gradient Boosted Tree}
\begin{figure}[H]
    \centering
    \includegraphics[width=0.95\textwidth]{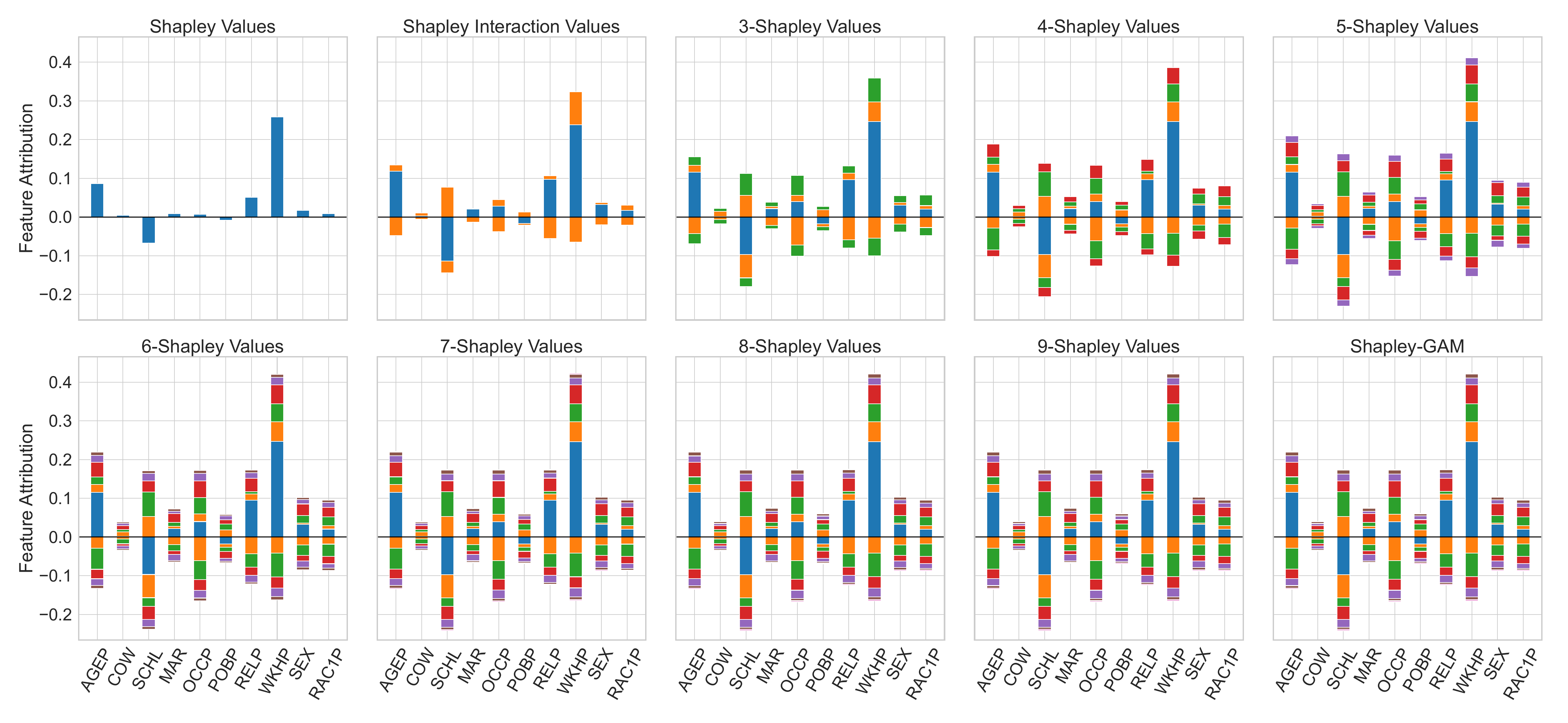}
    \includegraphics[width=0.85\textwidth]{figures/main_paper/legend.pdf}
    \caption{$n$-Shapley Values for a Gradient Boosted Tree and the first observation in our test set of the Folktables Income data set.}
\end{figure}

\newpage
\subsubsection{Random Forest}
\begin{figure}[H]
    \centering
    \includegraphics[width=\textwidth]{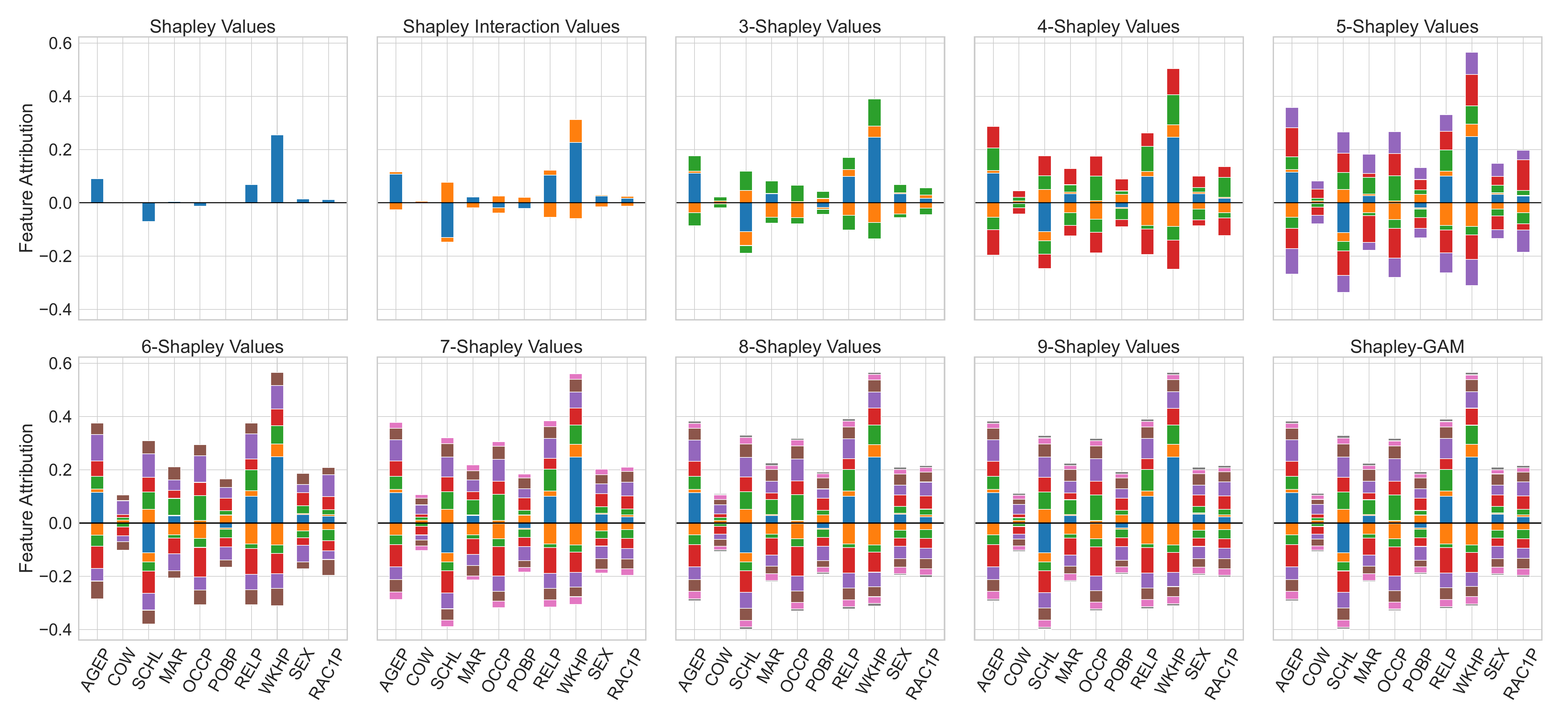}
    \includegraphics[width=0.85\textwidth]{figures/main_paper/legend.pdf}
    \caption{$n$-Shapley Values for a Random Forest and the first observation in our test set of the Folktables Income data set.}
\end{figure}

\subsubsection{k-Nearest Neighbor}
\begin{figure}[H]
    \centering
    \includegraphics[width=\textwidth]{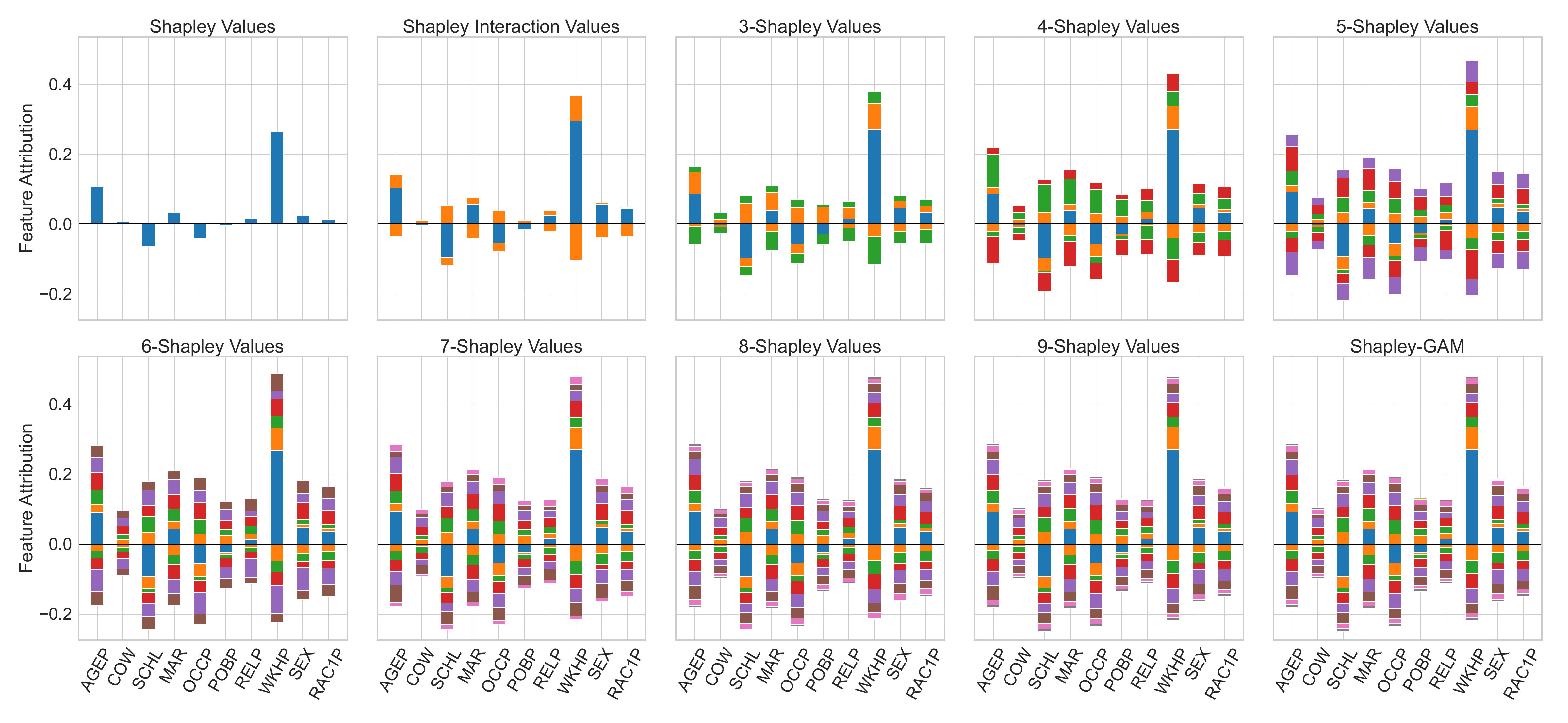}
    \includegraphics[width=0.85\textwidth]{figures/main_paper/legend.pdf}
    \caption{$n$-Shapley Values for a kNN classifier and the first observation in our test set of the Folktables Income data set.}
\end{figure}

\newpage
\begin{figure}[H]
    \centering    
    \includegraphics[width=\textwidth]{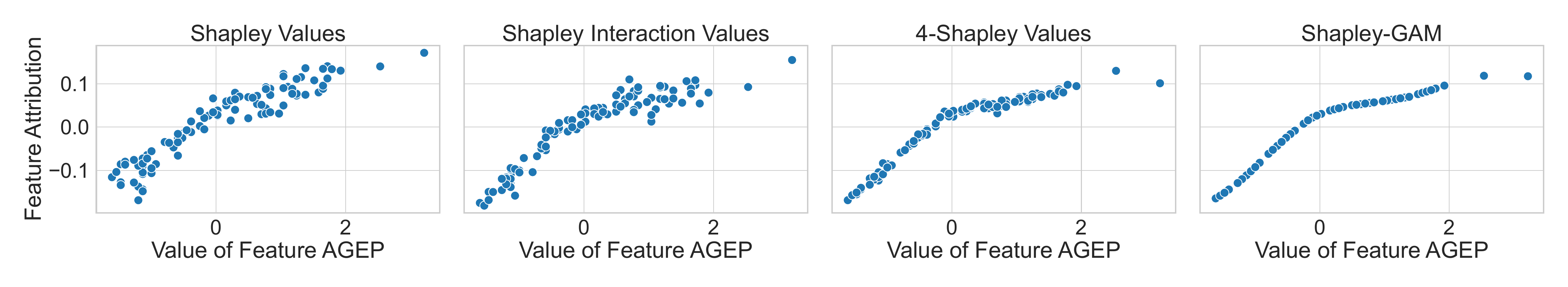}
    \includegraphics[width=\textwidth]{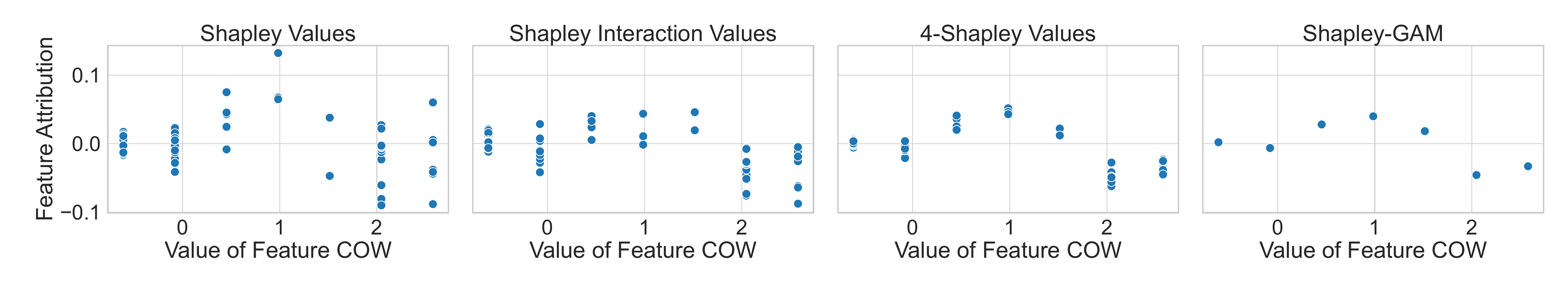}
    \includegraphics[width=\textwidth]{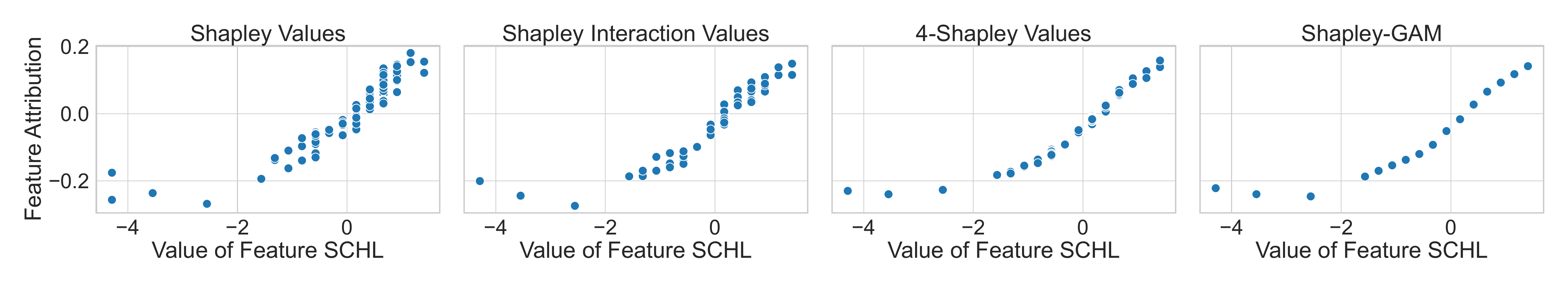}
    \includegraphics[width=\textwidth]{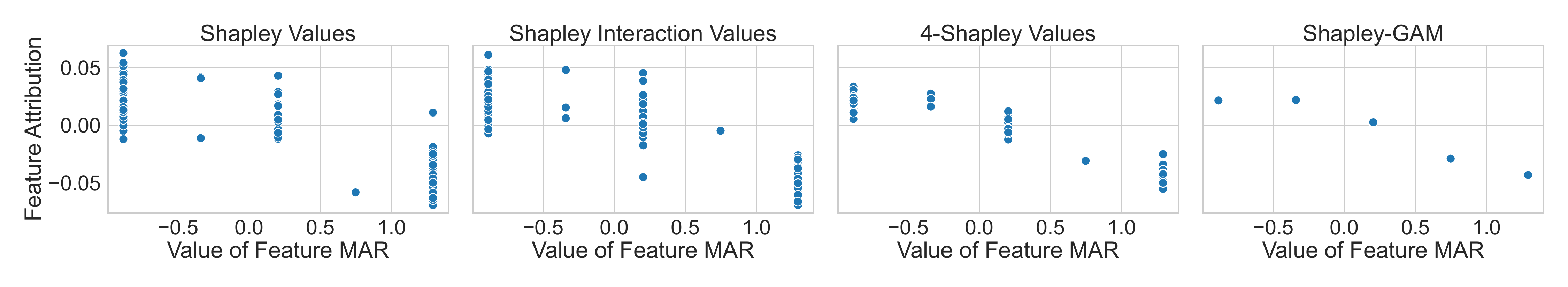}
    \includegraphics[width=\textwidth]{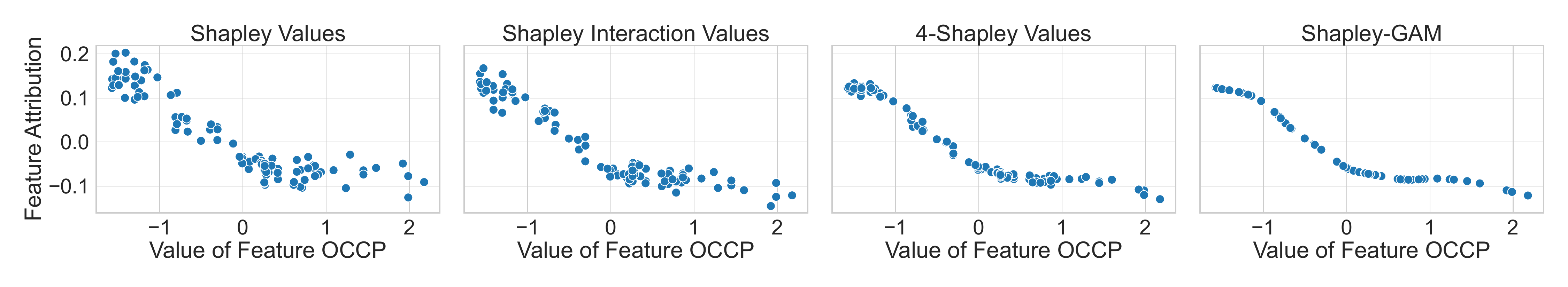}
    \includegraphics[width=\textwidth]{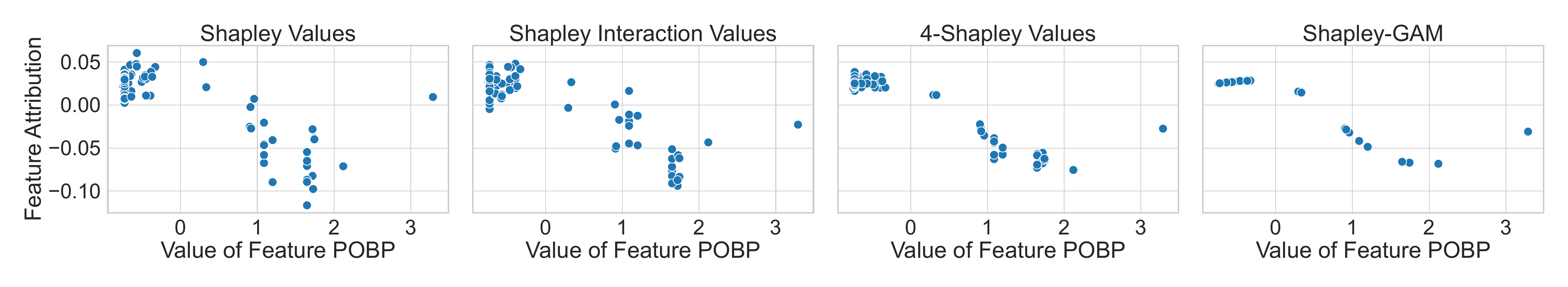}
    \includegraphics[width=\textwidth]{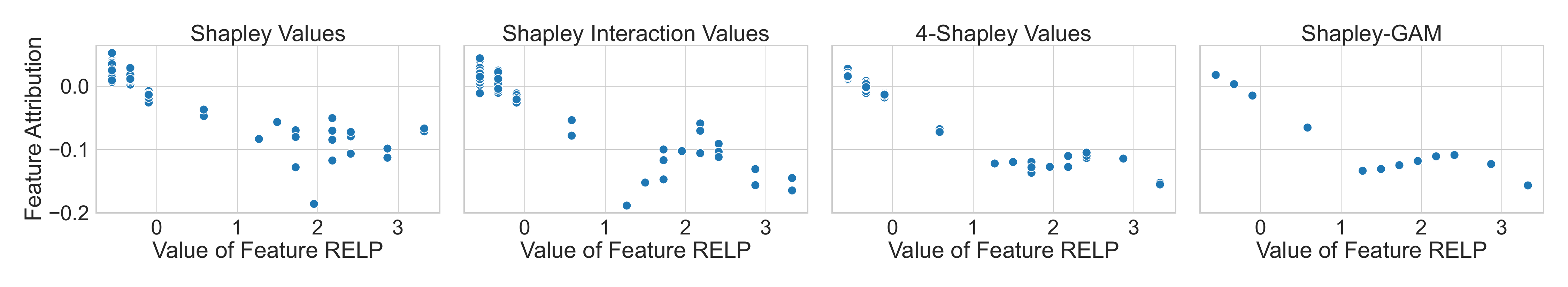}
    \caption{Partial dependence plots for the kNN classifier on the Folktables Income data set (compare Figure \ref{fig:partial_dependece_plot} in the main paper). Depicted are the partial dependence plots of $\Phi_i^n$ for $n=\{1,2,4,10\}$ and 7 different features.}
    \label{apx:knn_dependence}
\end{figure}

\newpage
\subsection{Folktables Travel}

\subsubsection{Glassbox-GAM}
\begin{figure}[H]
    \centering
   \includegraphics[width=\textwidth]{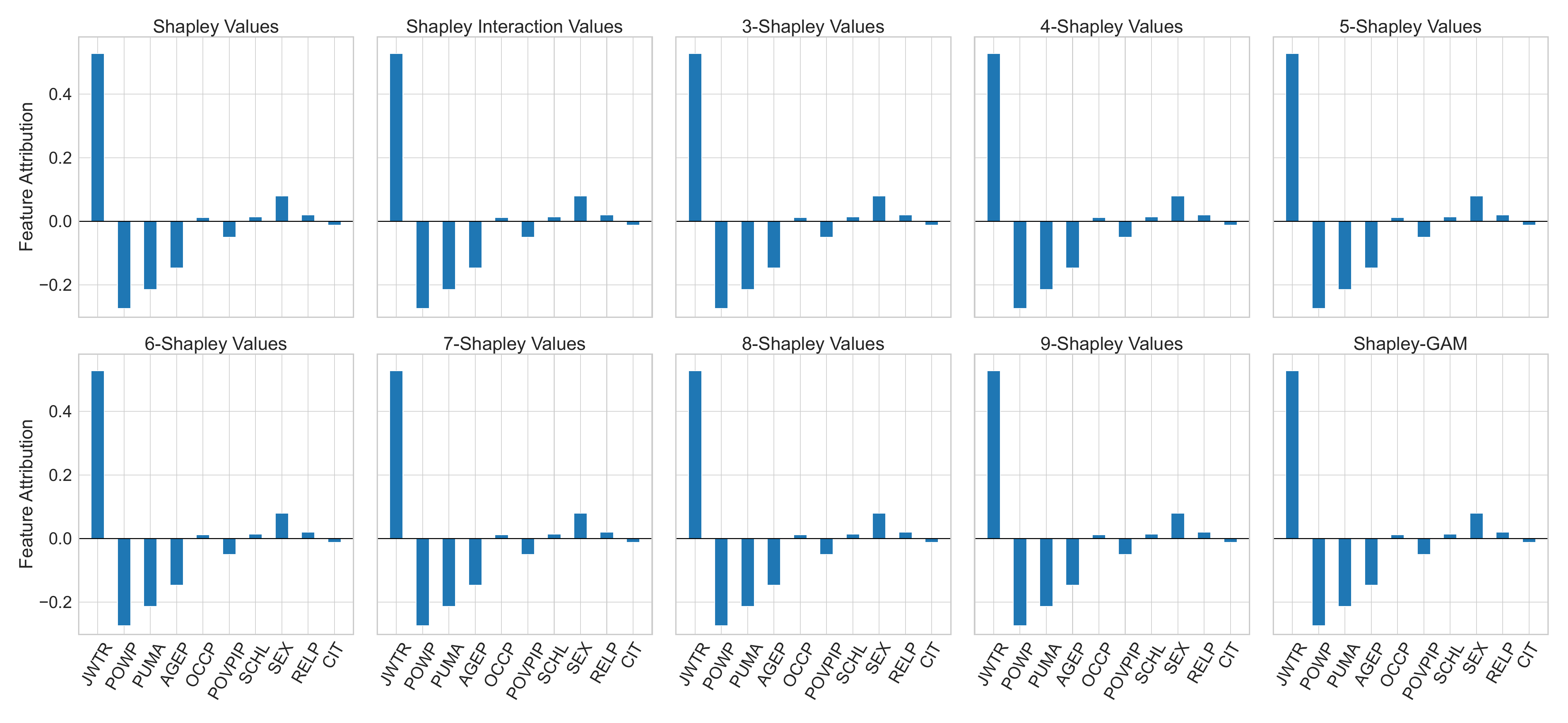}
   \includegraphics[width=0.85\textwidth]{figures/main_paper/legend.pdf}
    \caption{$n$-Shapley Values for a Glassbox-GAM and the first observation in our test set of the Folktables Travel data set.}
\end{figure}

\subsubsection{Gradient Boosted Tree}
\begin{figure}[H]
    \centering
    \includegraphics[width=\textwidth]{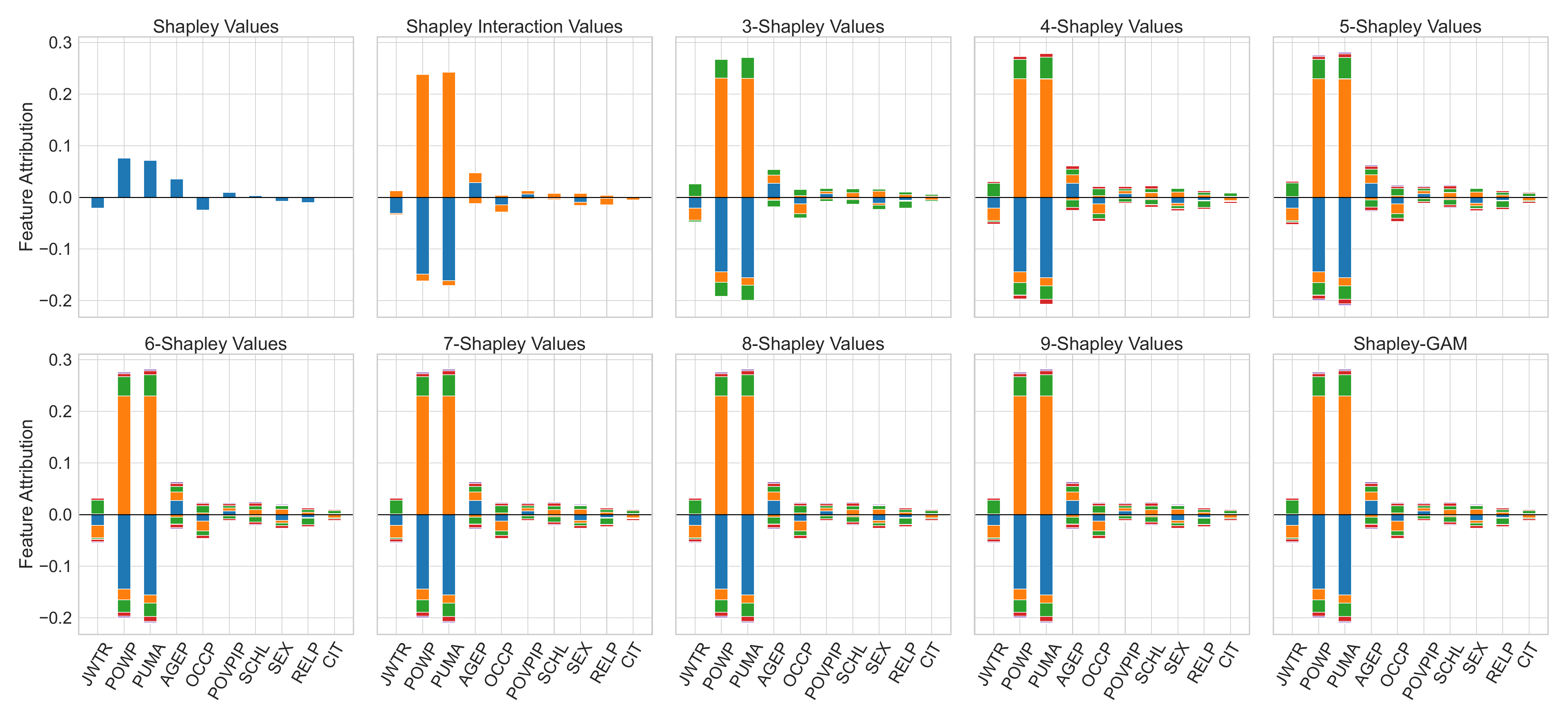}
    \includegraphics[width=0.85\textwidth]{figures/main_paper/legend.pdf}
    \caption{$n$-Shapley Values for a Gradient Boosted Tree and the first observation in our test set of the Folktables Travel data set.}
\end{figure}

\newpage
\subsubsection{Random Forest}
\begin{figure}[H]
    \centering
    \includegraphics[width=\textwidth]{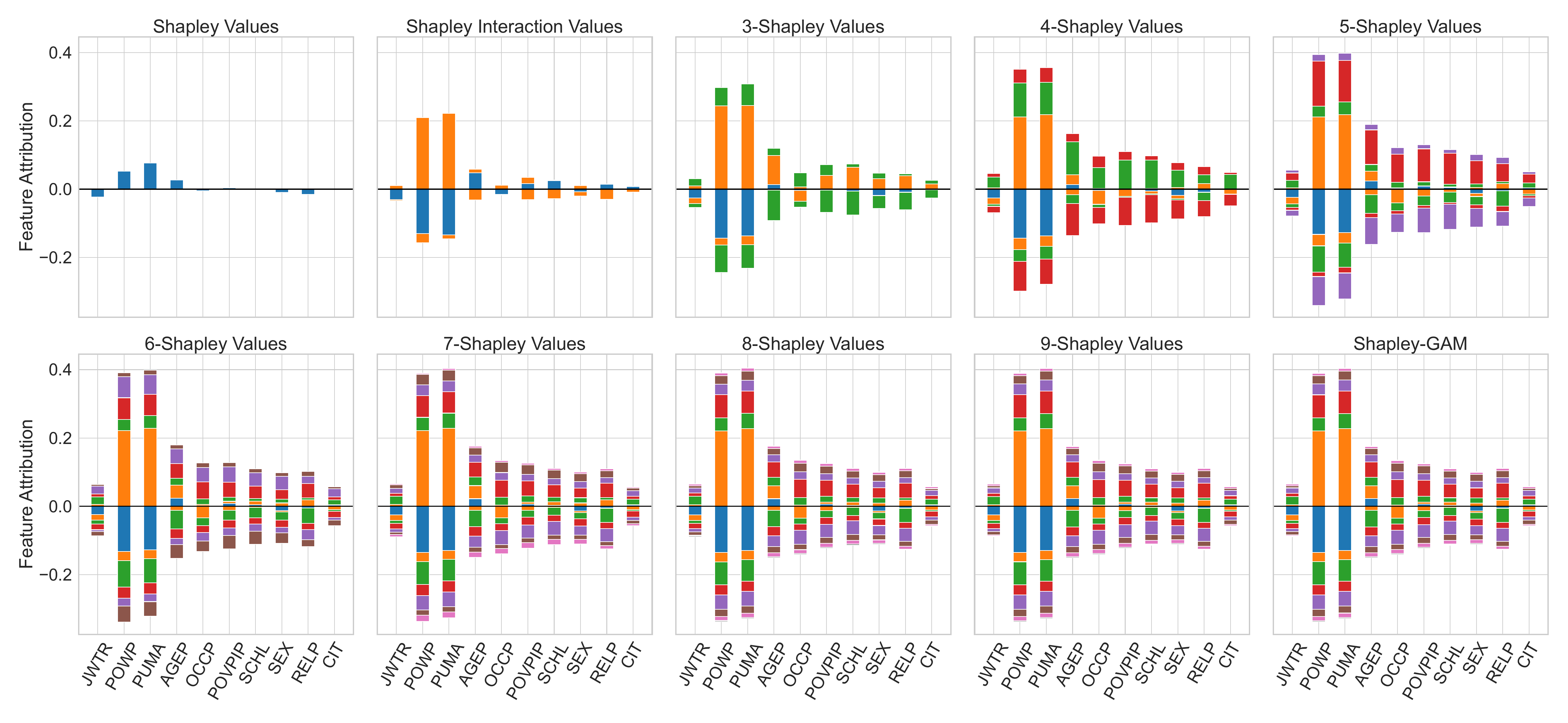}
    \includegraphics[width=0.85\textwidth]{figures/main_paper/legend.pdf}
    \caption{$n$-Shapley Values for a Random Forest and the first observation in our test set of the Folktables Travel data set.}
\end{figure}

\subsubsection{k-Nearest Neighbor}
\begin{figure}[H]
    \centering
    \includegraphics[width=\textwidth]{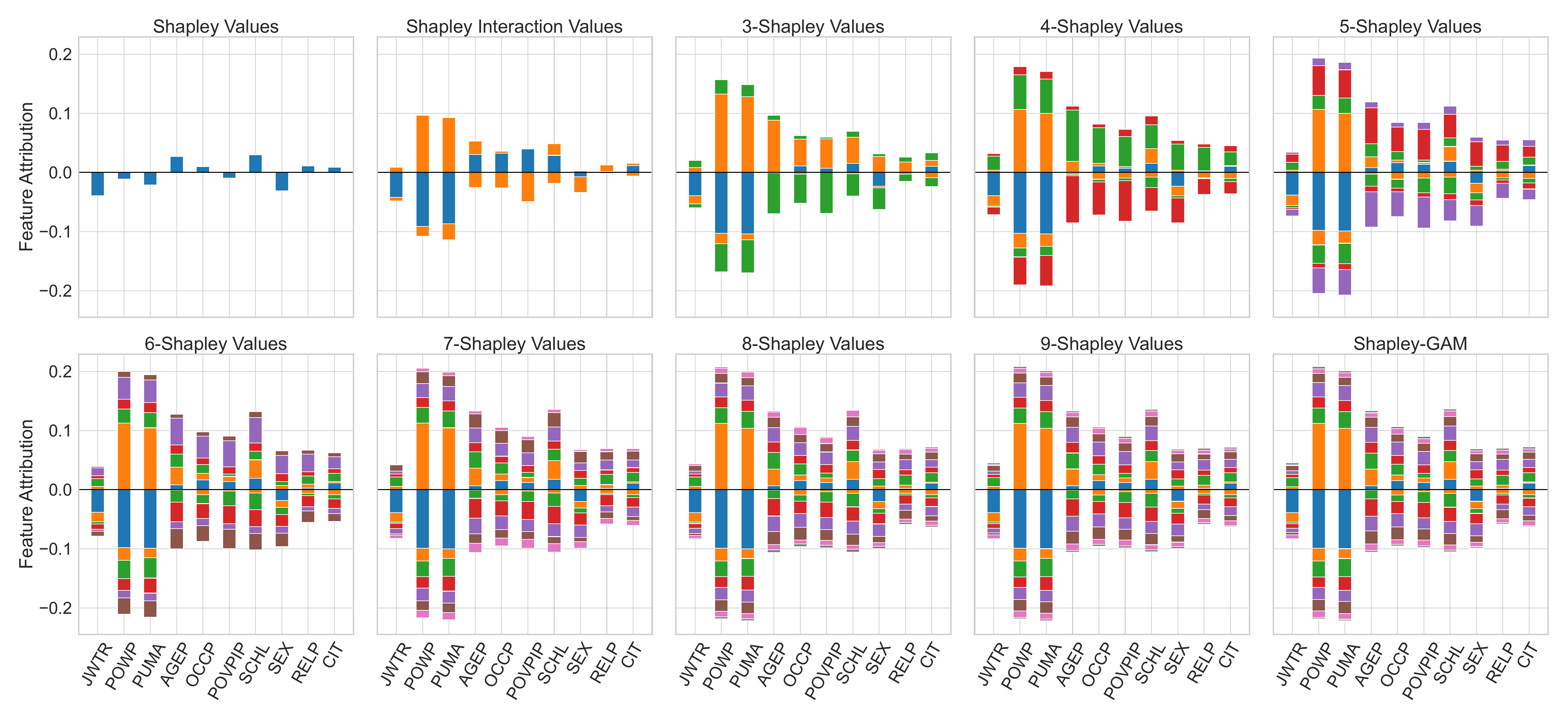}
    \includegraphics[width=0.85\textwidth]{figures/main_paper/legend.pdf}
    \caption{$n$-Shapley Values for a kNN classifier and the first observation in our test set of the Folktables Travel data set.}
\end{figure}

\newpage
\begin{figure}[H]
    \centering    
    \includegraphics[width=\textwidth]{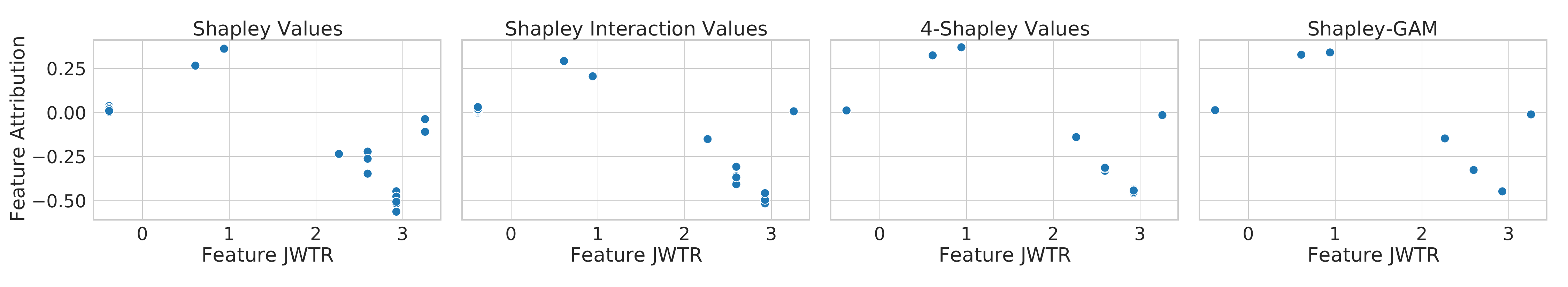}
    \includegraphics[width=\textwidth]{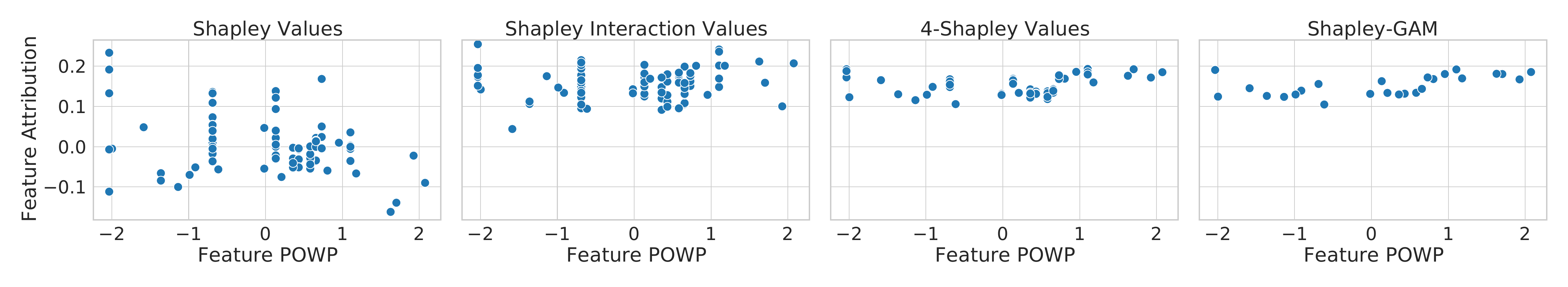}    \includegraphics[width=\textwidth]{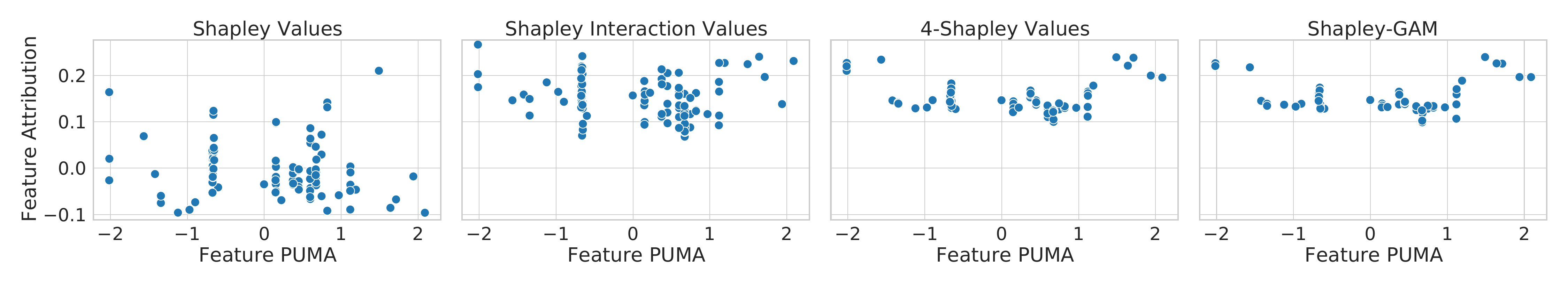}    \includegraphics[width=\textwidth]{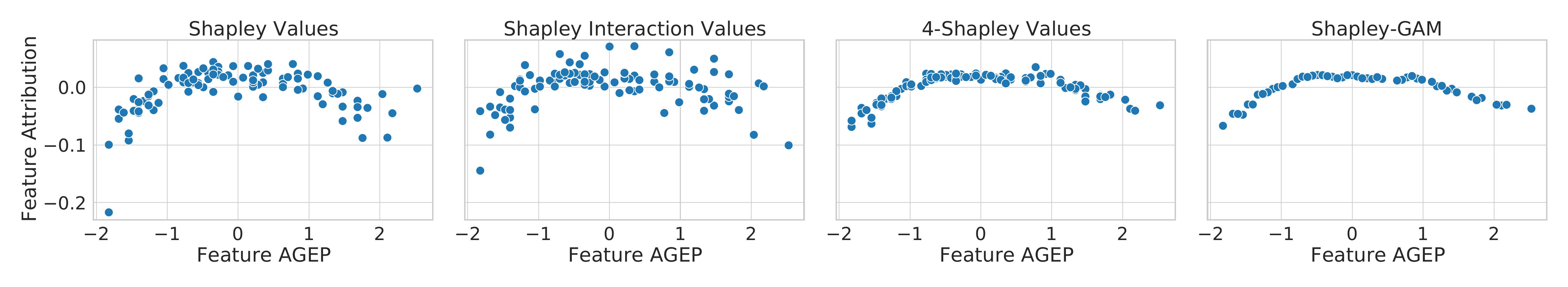}    \includegraphics[width=\textwidth]{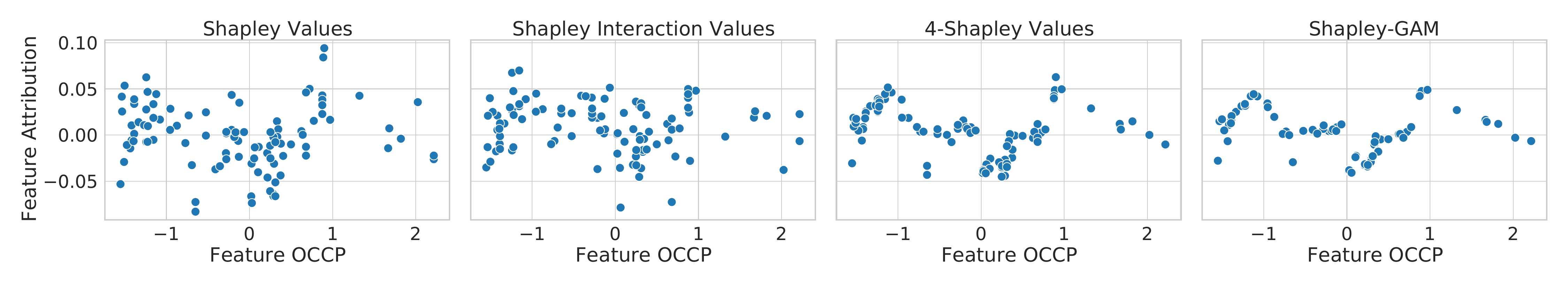}    \includegraphics[width=\textwidth]{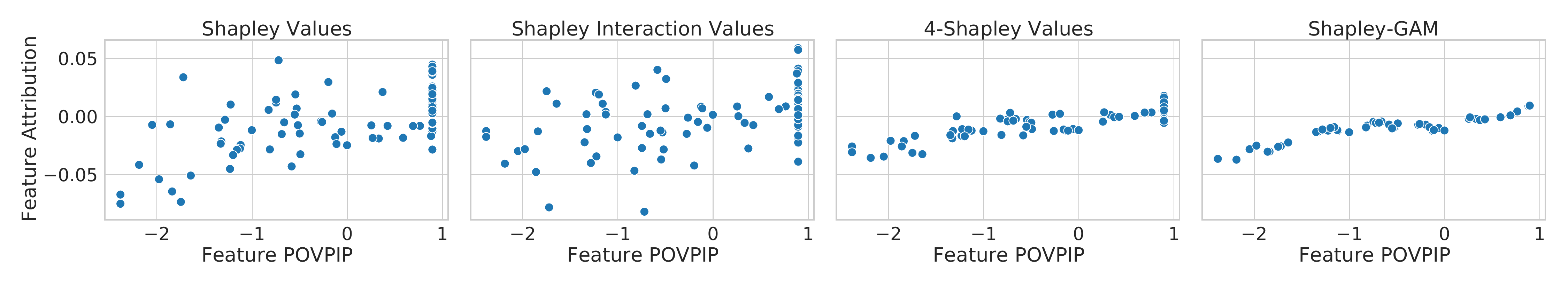}
    \includegraphics[width=\textwidth]{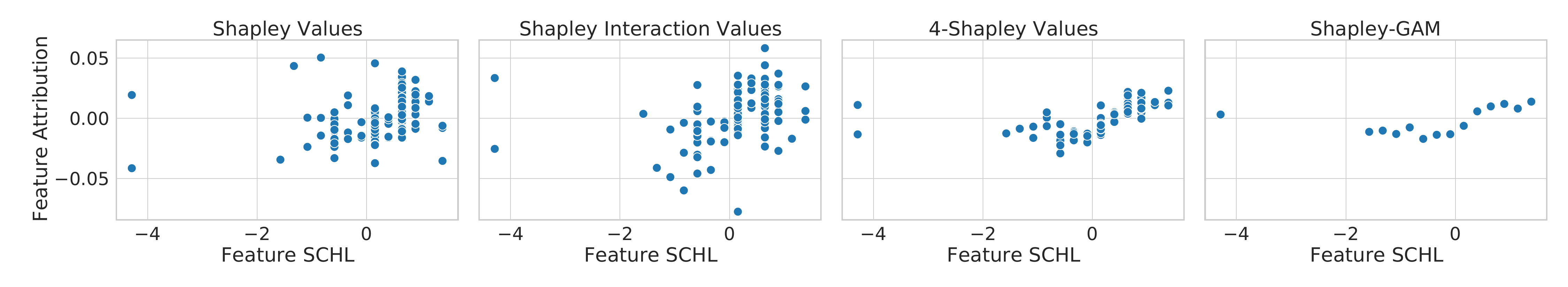}
    \caption{Partial dependence plots for the random forest on the Folktables Travel data set. Depicted are the partial dependence plots of $\Phi_i^n$ for $n=\{1,2,4,10\}$ and 7 different features.}
\end{figure}

\newpage
\subsection{German Credit}

\subsubsection{Glassbox-GAM}
\begin{figure}[H]
    \centering
   \includegraphics[width=\textwidth]{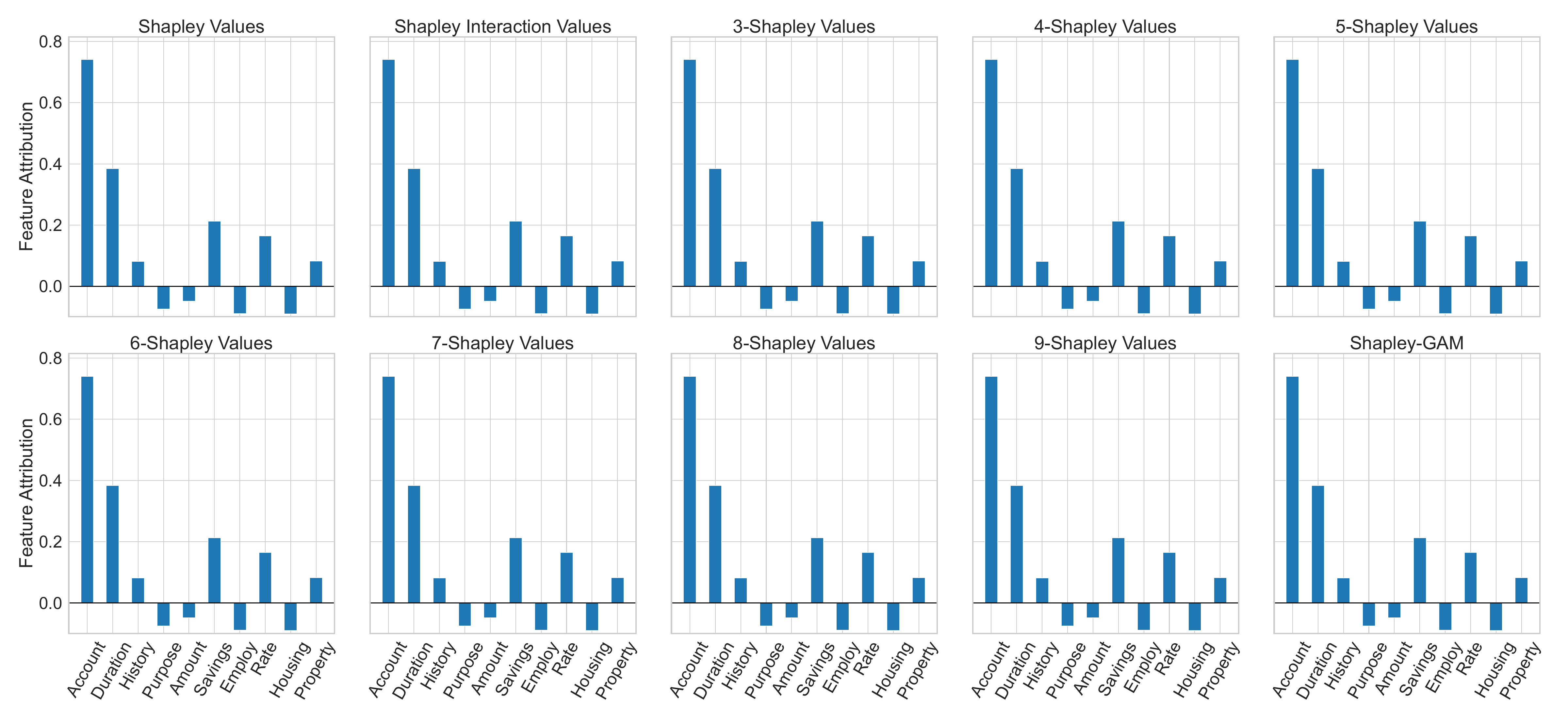}
   \includegraphics[width=0.85\textwidth]{figures/main_paper/legend.pdf}
    \caption{$n$-Shapley Values for a Glassbox-GAM and the first observation in our test set of the German Credit data set.}
\end{figure}

\subsubsection{Gradient Boosted Tree}
\begin{figure}[H]
    \centering
    \includegraphics[width=\textwidth]{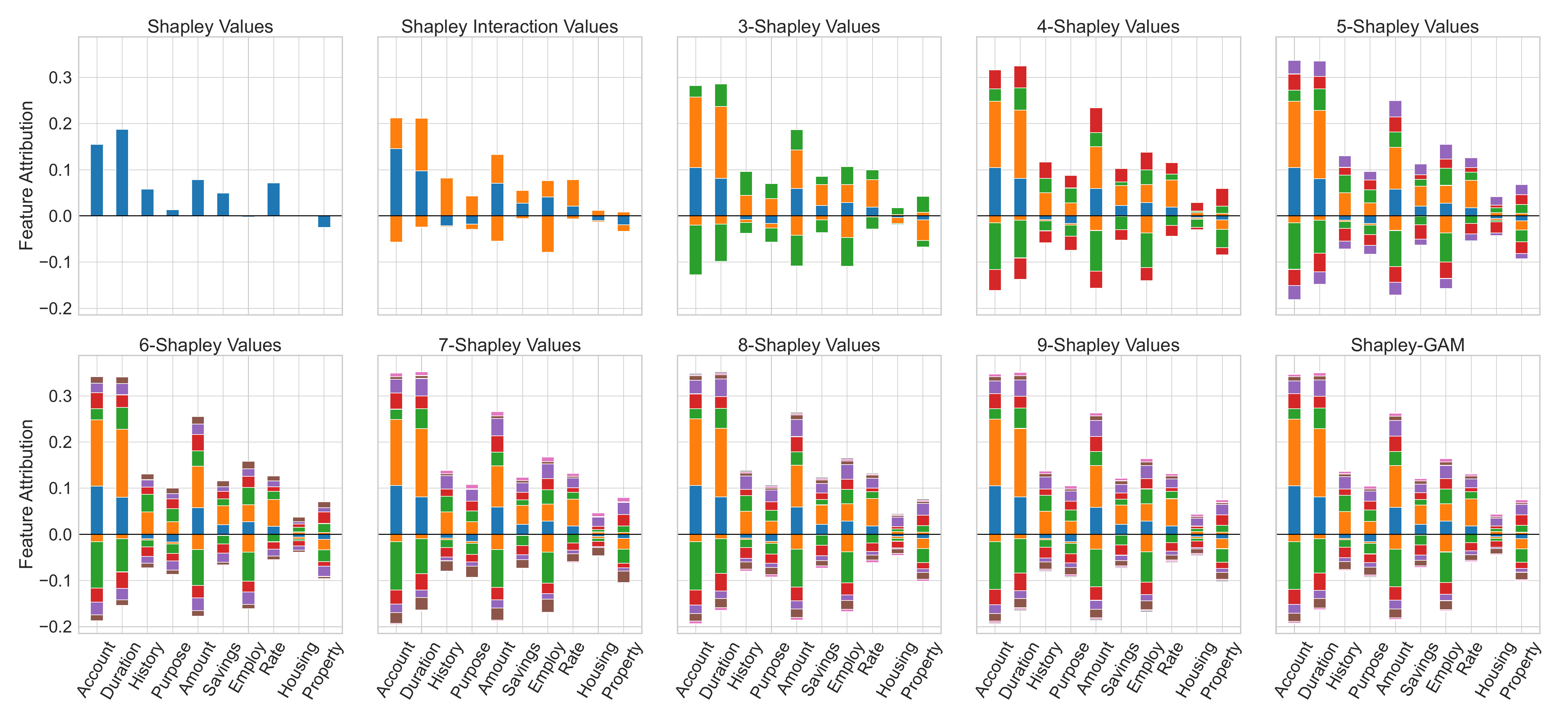}
    \includegraphics[width=0.85\textwidth]{figures/main_paper/legend.pdf}
    \caption{$n$-Shapley Values for a Gradient Boosted Tree and the first observation in our test set of the German Credit data set.}
\end{figure}

\newpage
\subsubsection{Random Forest}
\begin{figure}[H]
    \centering
    \includegraphics[width=\textwidth]{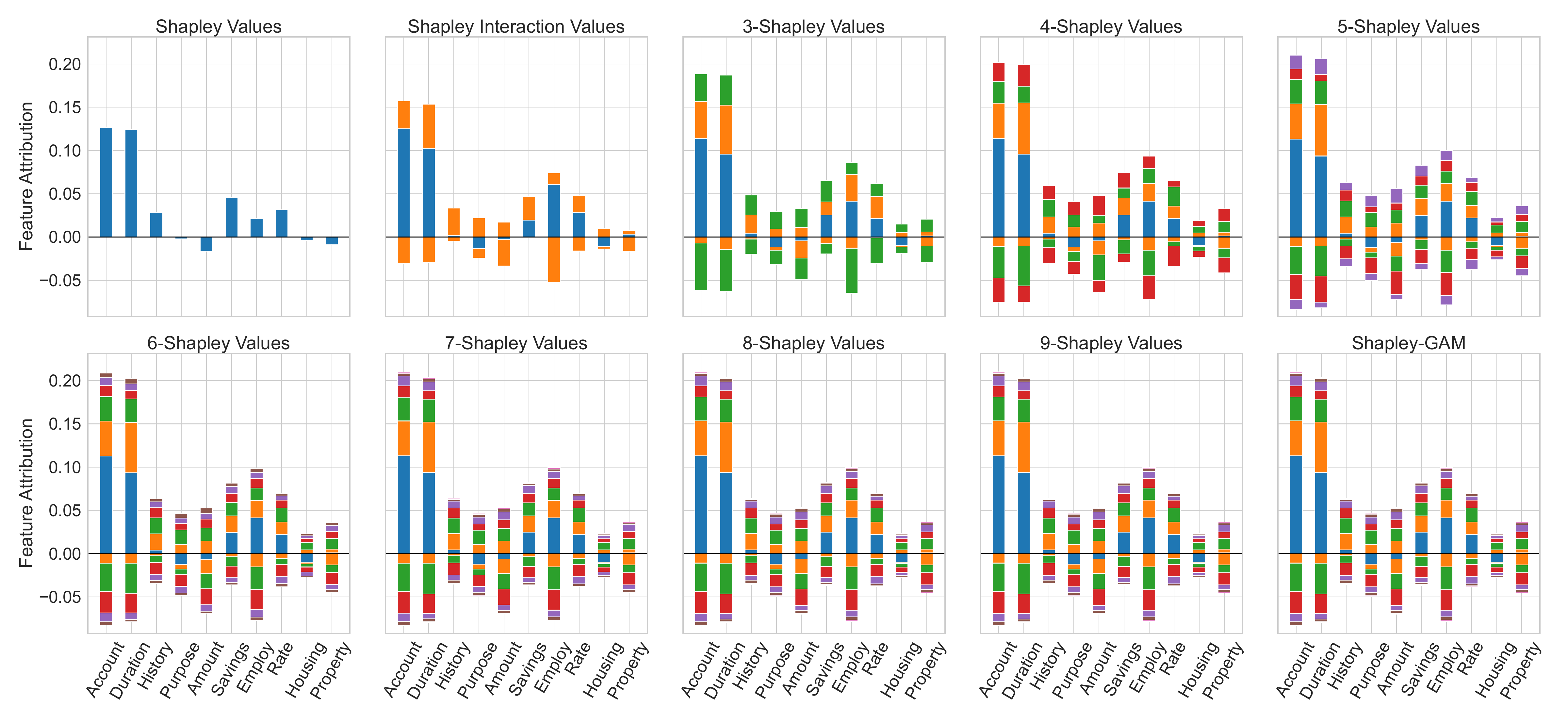}
    \includegraphics[width=0.85\textwidth]{figures/main_paper/legend.pdf}
    \caption{$n$-Shapley Values for a Random Forest and the first observation in our test set of the German Credit data set.}
\end{figure}

\subsubsection{k-Nearest Neighbor}
\begin{figure}[H]
    \centering
    \includegraphics[width=\textwidth]{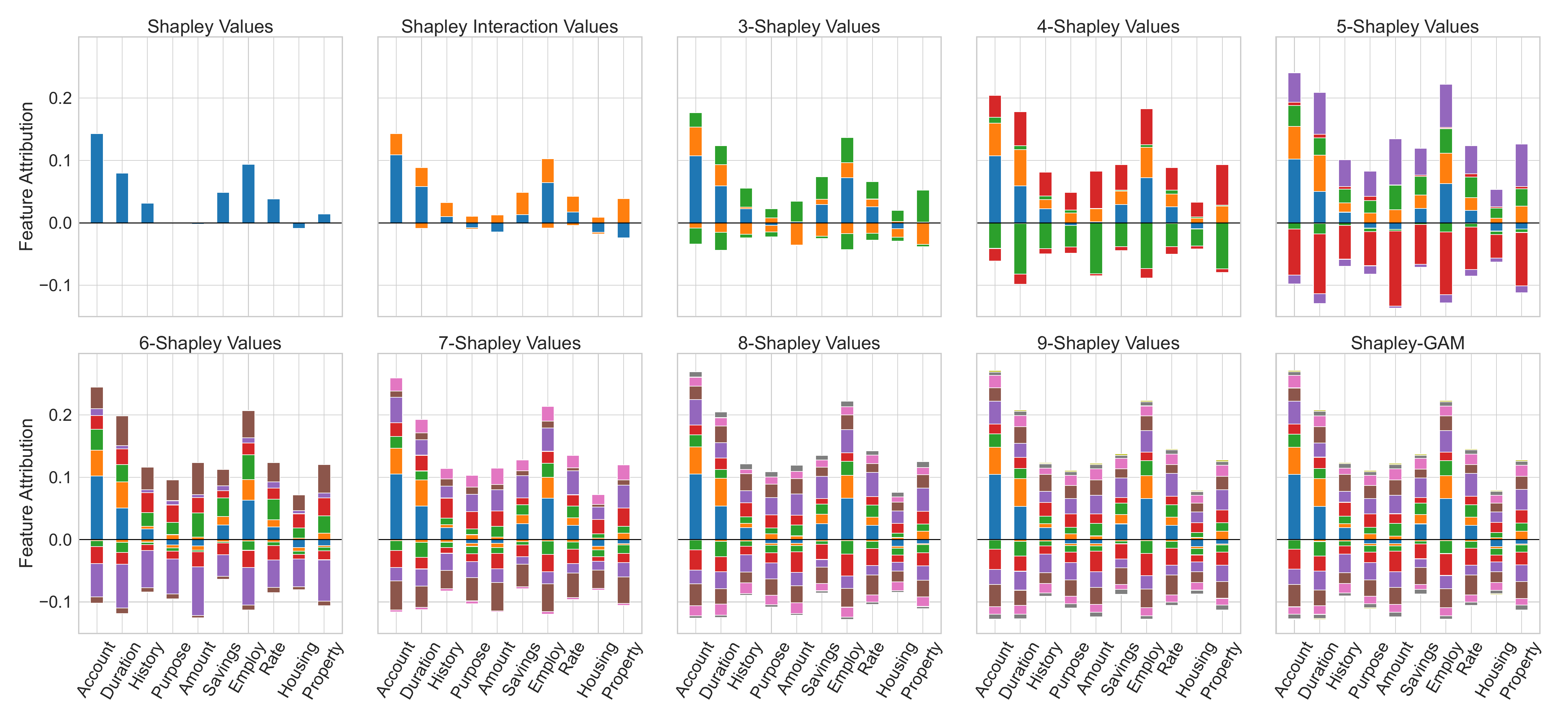}
    \includegraphics[width=0.85\textwidth]{figures/main_paper/legend.pdf}
    \caption{$n$-Shapley Values for a kNN classifier and the first observation in our test set of the German Credit data set.}
\end{figure}

\newpage
\begin{figure}[H]
    \centering    
    \includegraphics[width=\textwidth]{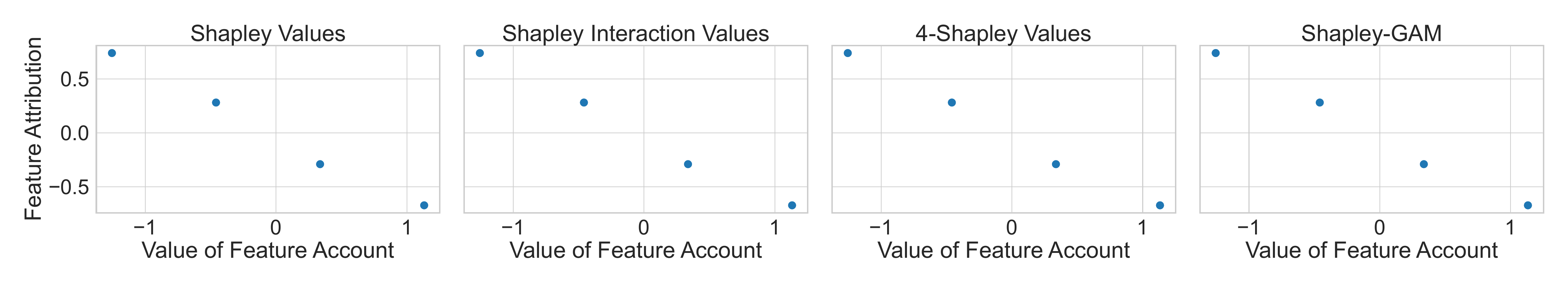}
    \includegraphics[width=\textwidth]{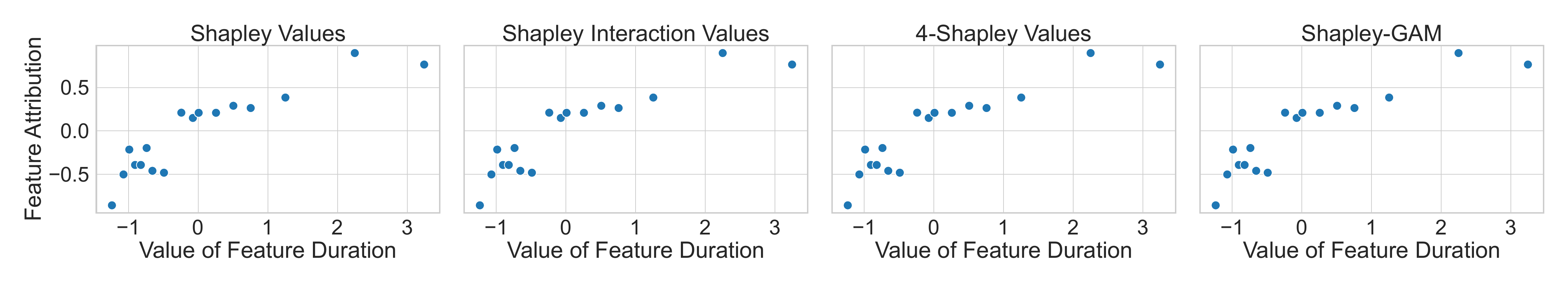}
    \includegraphics[width=\textwidth]{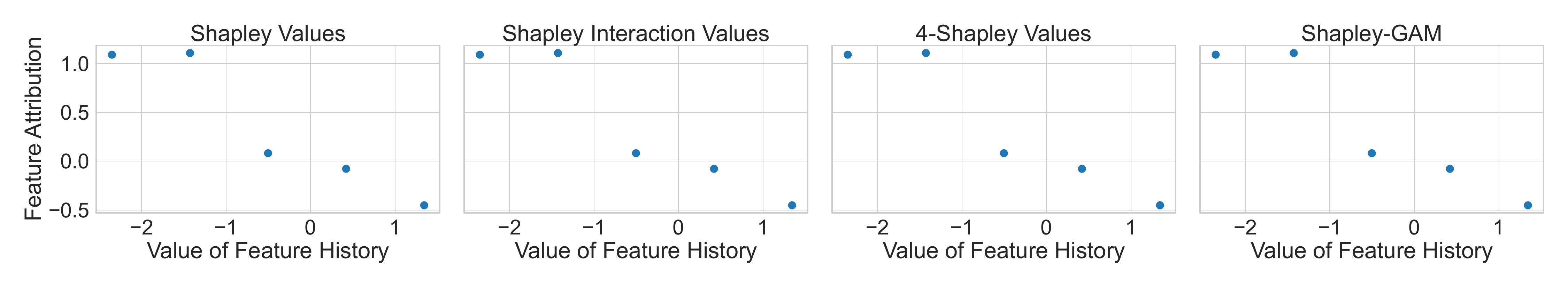}
    \includegraphics[width=\textwidth]{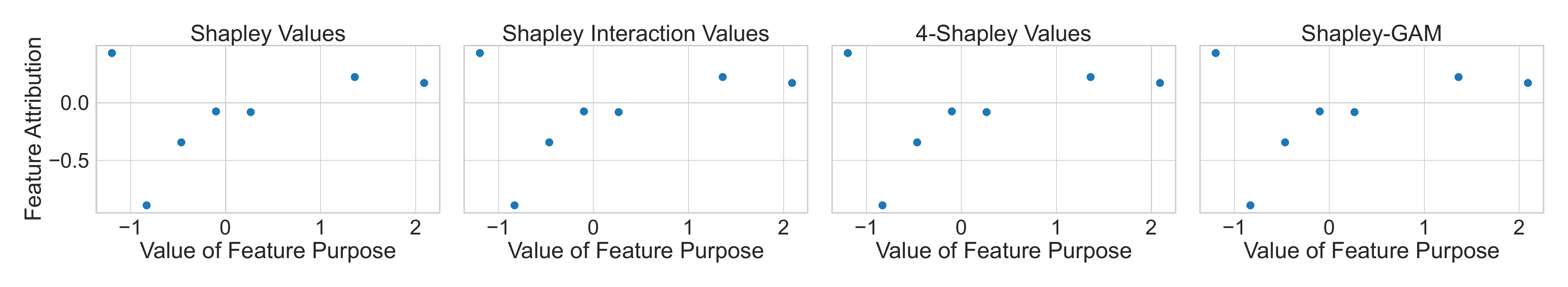}
    \includegraphics[width=\textwidth]{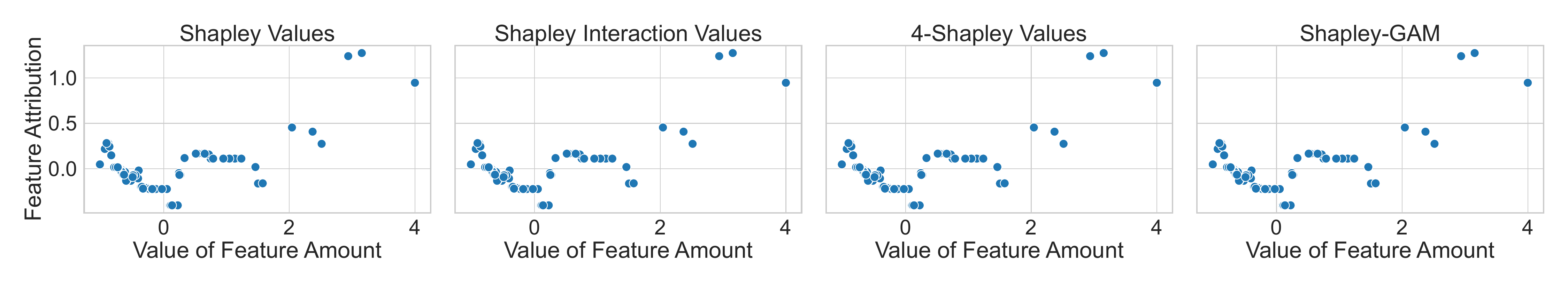}
    \includegraphics[width=\textwidth]{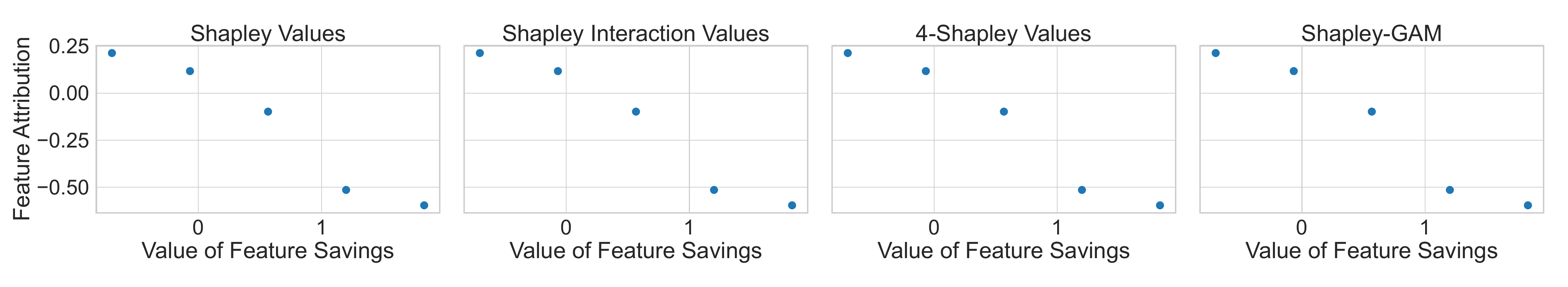}
    \includegraphics[width=\textwidth]{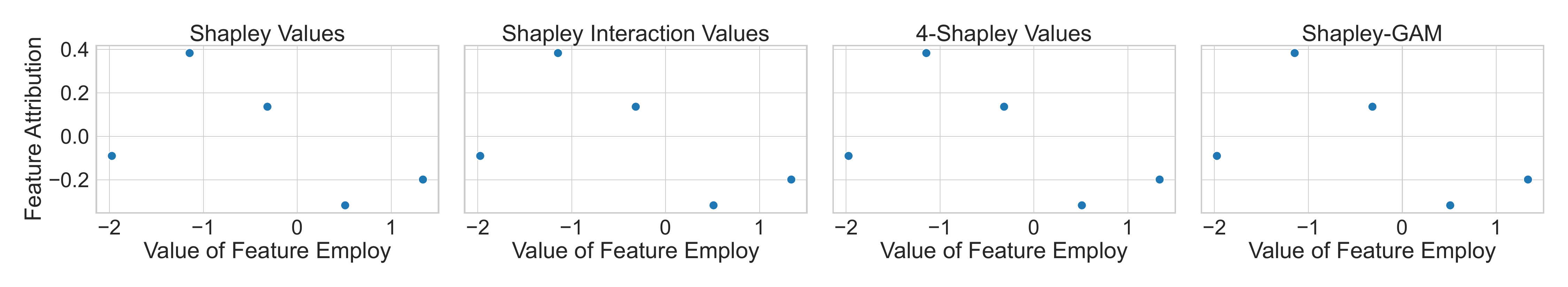}
    \caption{Partial dependence plots for the Glassbox-GAM without interaction terms on the German Credit data set. Depicted are the partial dependence plots of $\Phi_i^n$ for $n=\{1,2,4,10\}$ and 7 different features.}
\end{figure}

\newpage
\subsection{California Housing}

\subsubsection{Glassbox-GAM}
\begin{figure}[H]
    \centering
   \includegraphics[width=0.82\textwidth]{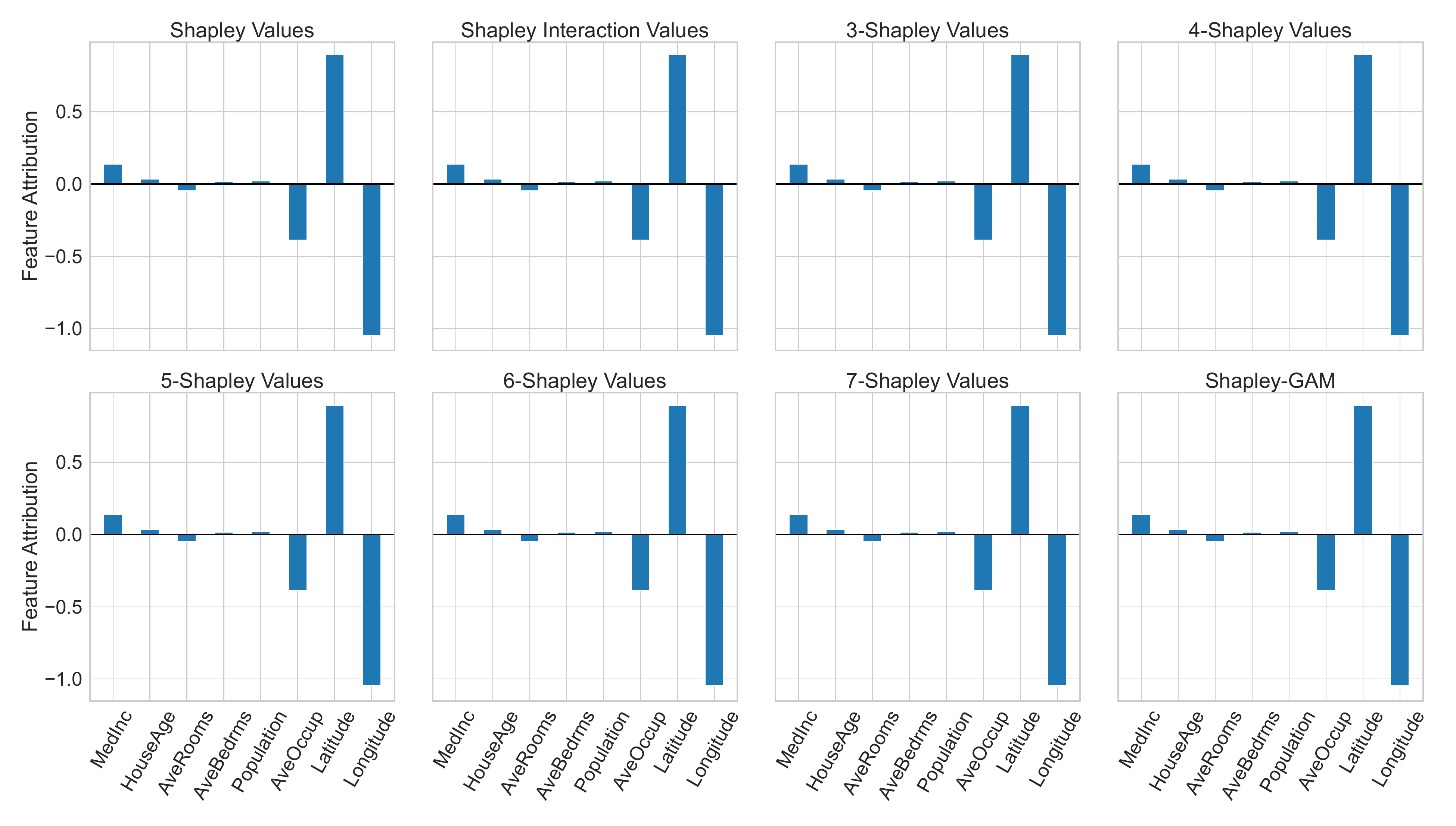}
   \includegraphics[width=0.85\textwidth]{figures/main_paper/legend.pdf}
    \caption{$n$-Shapley Values for a Glassbox-GAM and the first observation in our test set of the California Housing data set.}
\end{figure}

\subsubsection{Gradient Boosted Tree}
\begin{figure}[H]
    \centering
    \includegraphics[width=0.82\textwidth]{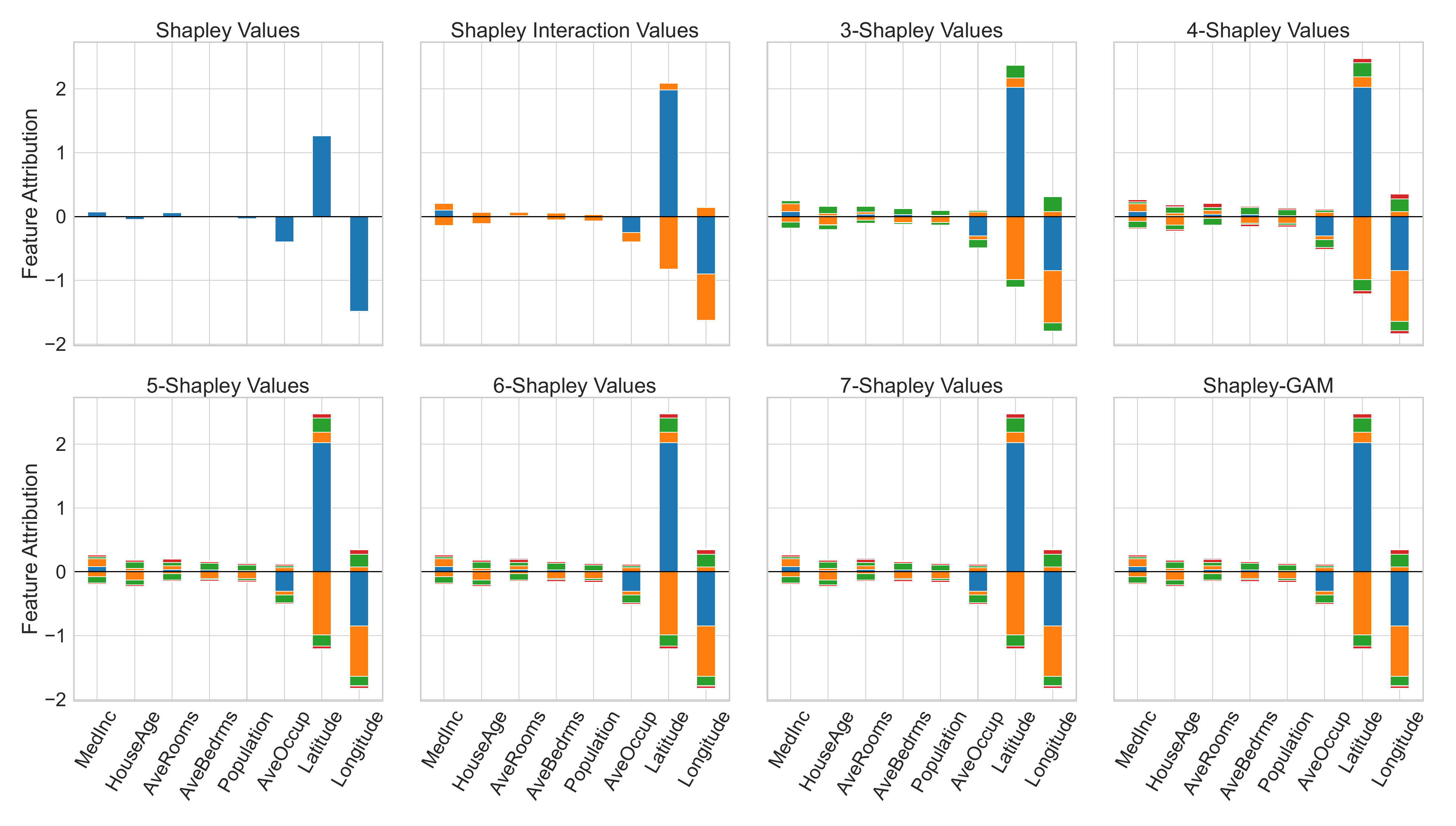}
    \includegraphics[width=0.85\textwidth]{figures/main_paper/legend.pdf}
    \caption{$n$-Shapley Values for a Gradient Boosted Tree and the first observation in our test set of the California Housing data set.}
\end{figure}

\newpage
\subsubsection{Random Forest}
\begin{figure}[H]
    \centering
    \includegraphics[width=0.82\textwidth]{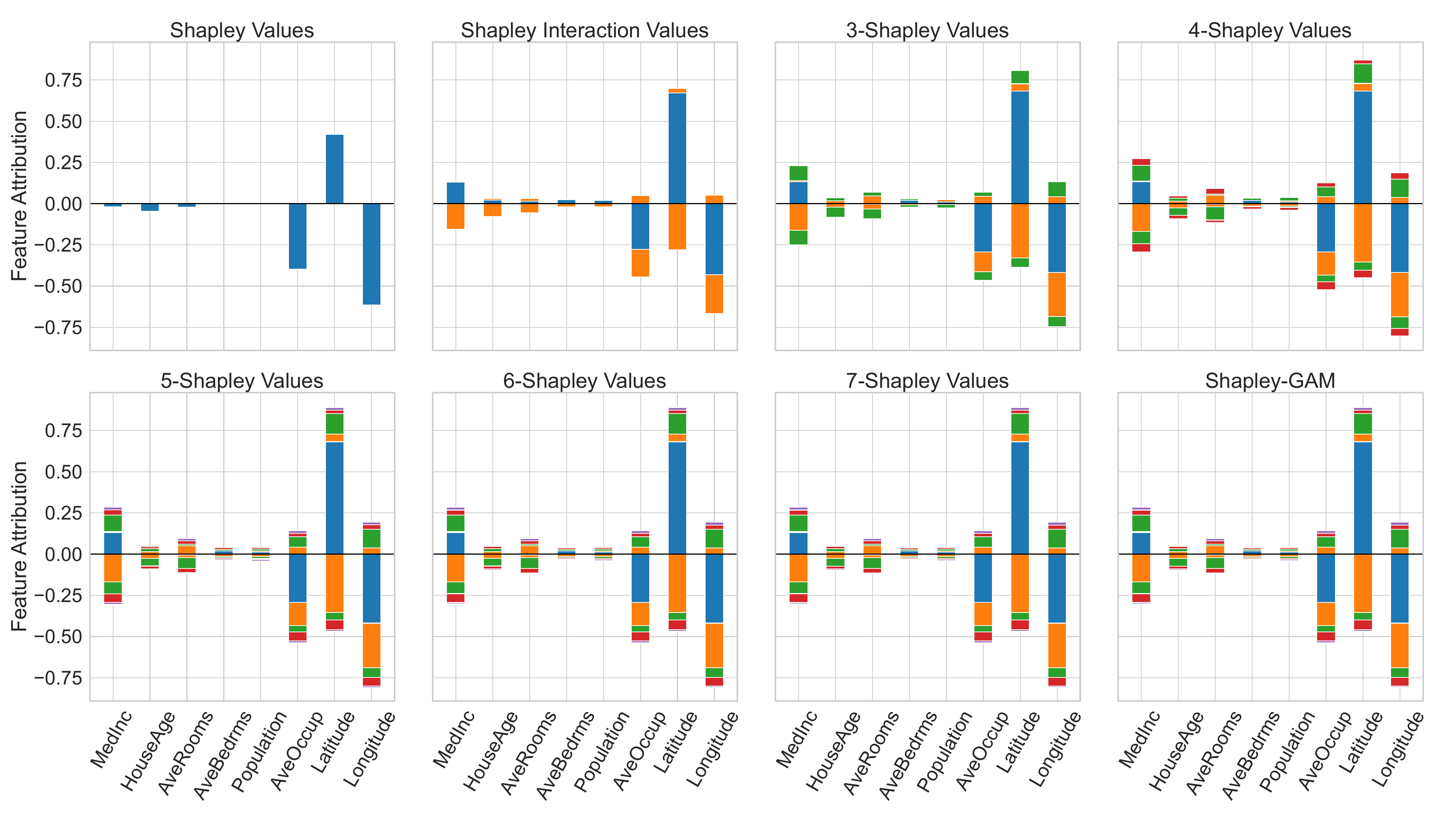}
    \includegraphics[width=0.85\textwidth]{figures/main_paper/legend.pdf}
    \caption{$n$-Shapley Values for a Random Forest and the first observation in our test set of the California Housing data set.}
\end{figure}

\subsubsection{k-Nearest Neighbor}
\begin{figure}[H]
    \centering
    \includegraphics[width=0.82\textwidth]{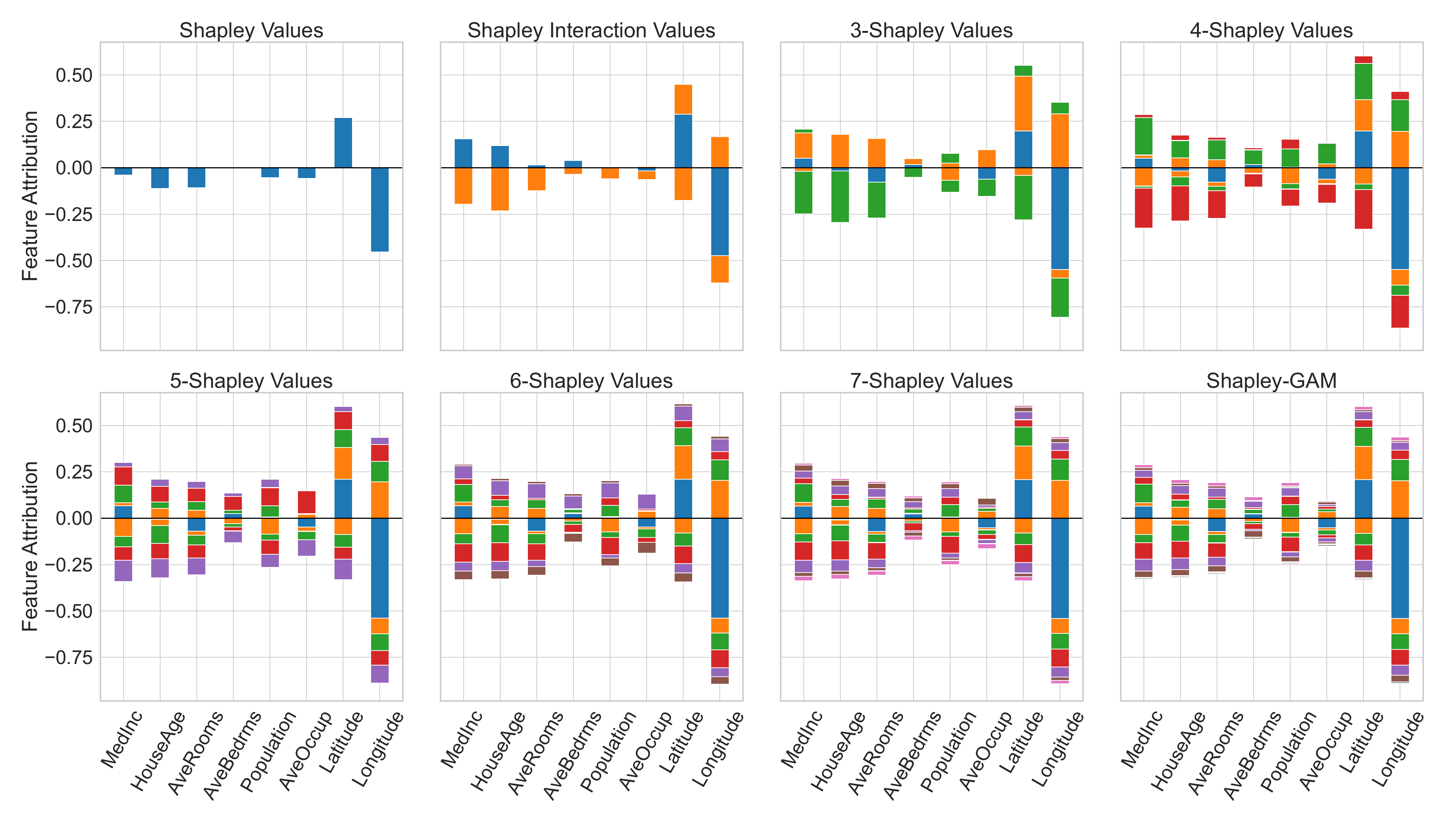}
    \includegraphics[width=0.85\textwidth]{figures/main_paper/legend.pdf}
    \caption{$n$-Shapley Values for a kNN classifier and the first observation in our test set of the California Housing data set.}
\end{figure}

\newpage
\begin{figure}[H]
    \centering    
    \includegraphics[width=\textwidth]{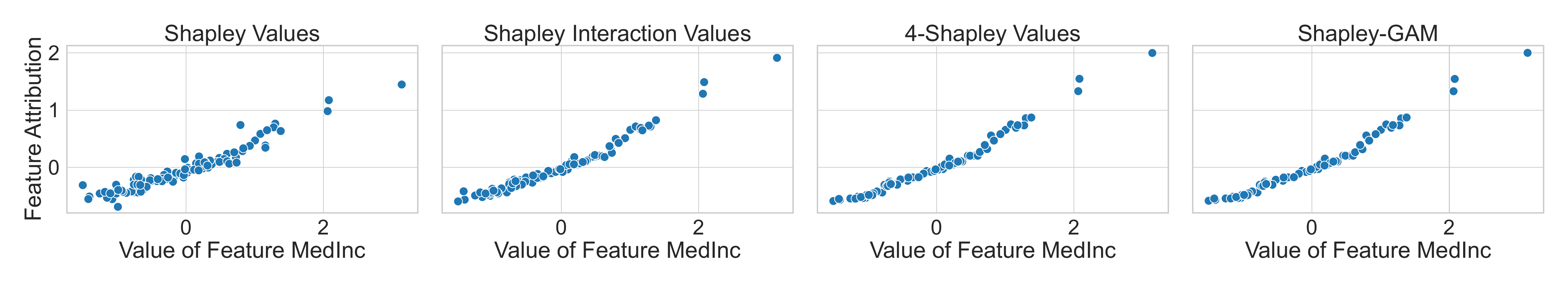}
    \includegraphics[width=\textwidth]{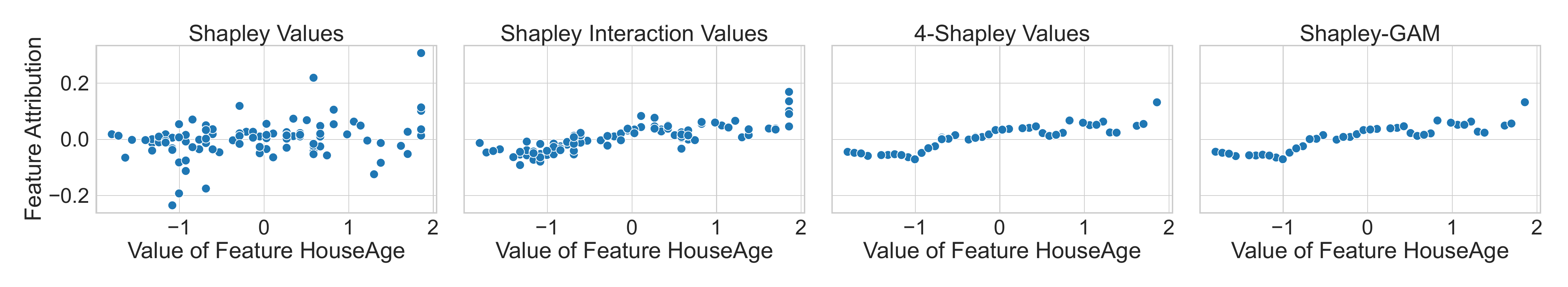}
    \includegraphics[width=\textwidth]{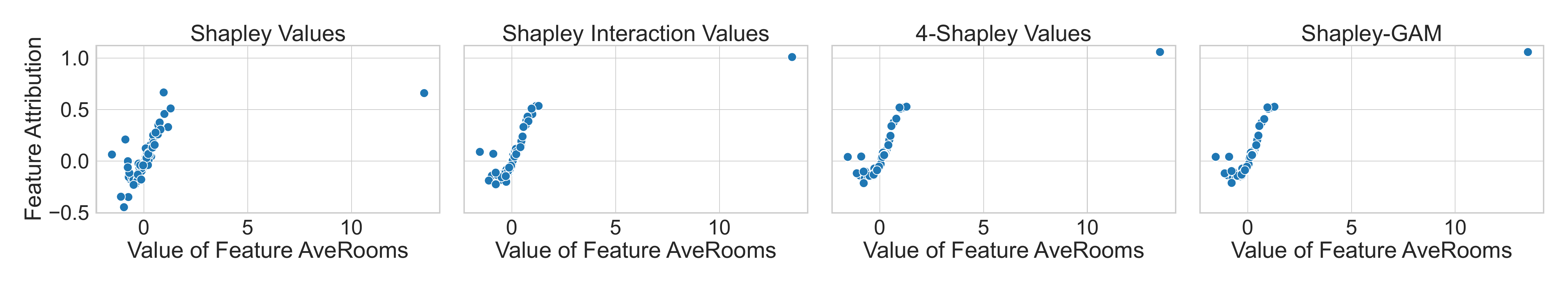}
    \includegraphics[width=\textwidth]{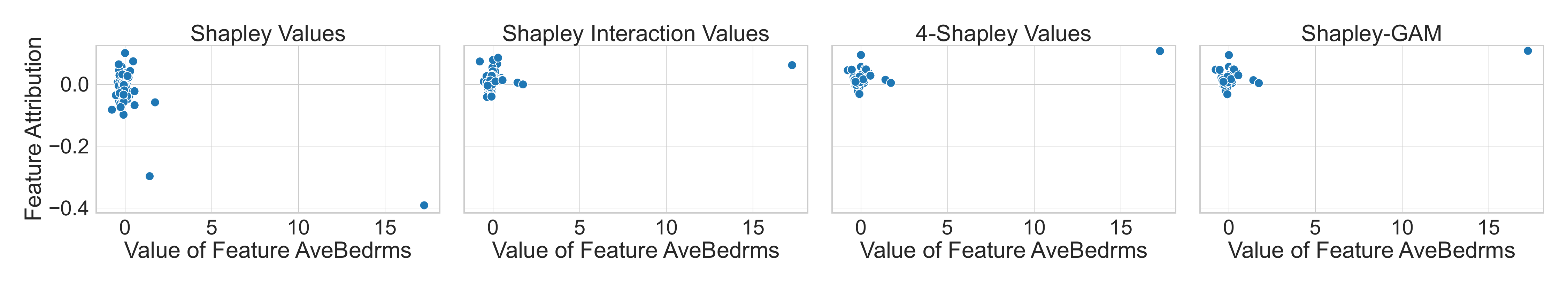}
    \includegraphics[width=\textwidth]{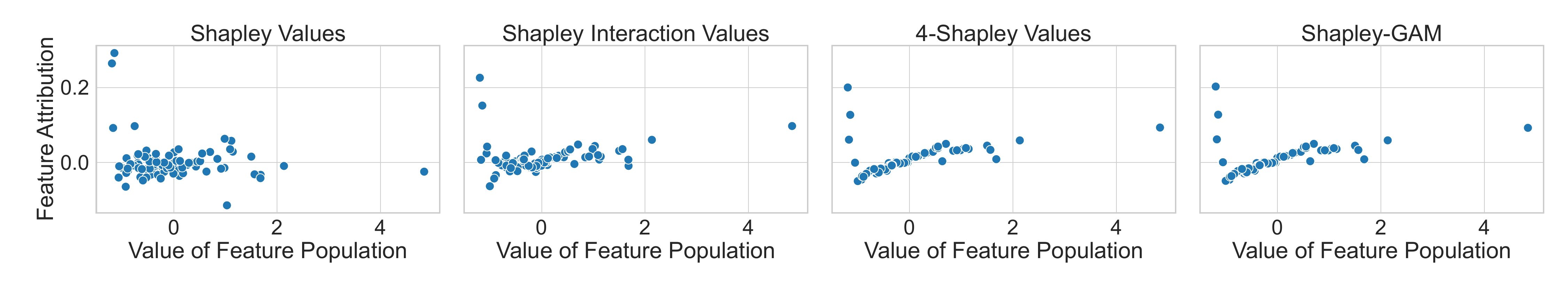}
    \includegraphics[width=\textwidth]{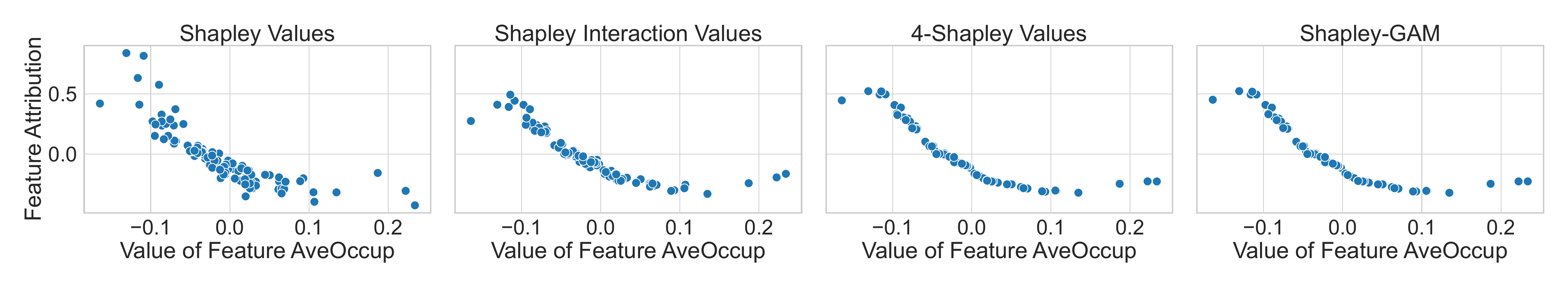}
    \includegraphics[width=\textwidth]{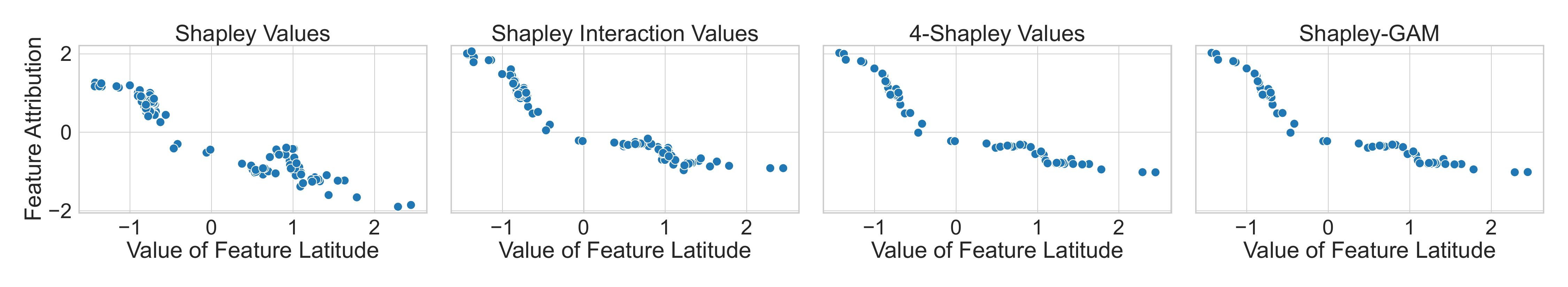}
    \caption{Partial dependence plots for the gradient boosted tree on the California Housing data set. Depicted are the partial dependence plots of $\Phi_i^n$ for $n=\{1,2,4,10\}$ and 7 different features.}
\end{figure}

\newpage
\begin{table}[t]
\tiny
\begin{tabular}{lccc}
\hline
Subset $S$ & N=500 & N=5000 & N=133 549 \\  
\hline\\[-6pt]
(0,) & 0.1128 & 0.1144 & 0.1165 \\
(1,) & 0.0005 & -0.0006 & 0.0022 \\
(2,) & -0.1248 & -0.1117 & -0.1098 \\
(3,) & 0.0227 & 0.0283 & 0.0281 \\
(4,) & -0.0041 & -0.0020 & -0.0018 \\
(5,) & -0.0123 & -0.0189 & -0.0200 \\
(6,) & 0.0845 & 0.1003 & 0.0982 \\
(7,) & 0.2357 & 0.2478 & 0.2505 \\
(8,) & 0.0280 & 0.0329 & 0.0347 \\
(9,) & 0.0197 & 0.0238 & 0.0241 \\
(0, 1) & -0.0023 & -0.0020 & -0.0032 \\
(0, 2) & 0.0059 & 0.0005 & -0.0025 \\
(0, 3) & -0.0146 & -0.0128 & -0.0126 \\
(0, 4) & 0.0089 & 0.0102 & 0.0102 \\
(0, 5) & 0.0038 & 0.0140 & 0.0141 \\
(0, 6) & -0.0244 & -0.0242 & -0.0213 \\
(0, 7) & -0.0452 & -0.0416 & -0.0426 \\
(0, 8) & -0.0032 & -0.0028 & -0.0038 \\
(0, 9) & -0.0043 & -0.0055 & -0.0071 \\
(1, 2) & -0.0012 & -0.0009 & -0.0022 \\
(1, 3) & -0.0004 & 0.0003 & 0.0007 \\
(1, 4) & -0.0006 & 0.0011 & -0.0011 \\
(1, 5) & -0.0060 & -0.0000 & 0.0004 \\
(1, 6) & 0.0027 & 0.0024 & 0.0014 \\
(1, 7) & 0.0093 & 0.0102 & 0.0079 \\
(1, 8) & -0.0017 & 0.0020 & 0.0007 \\
(1, 9) & 0.0048 & 0.0032 & 0.0027 \\
(2, 3) & 0.0029 & -0.0006 & -0.0016 \\
(2, 4) & -0.0419 & -0.0534 & -0.0547 \\
(2, 5) & -0.0128 & -0.0095 & -0.0115 \\
(2, 6) & 0.0389 & 0.0286 & 0.0290 \\
(2, 7) & 0.0752 & 0.0695 & 0.0677 \\
(2, 8) & -0.0031 & -0.0044 & -0.0070 \\
(2, 9) & 0.0151 & 0.0039 & 0.0031 \\
(3, 4) & -0.0112 & -0.0093 & -0.0091 \\
(3, 5) & 0.0006 & 0.0058 & 0.0055 \\
(3, 6) & -0.0068 & -0.0116 & -0.0099 \\
(3, 7) & -0.0286 & -0.0298 & -0.0304 \\
(3, 8) & -0.0135 & -0.0165 & -0.0181 \\
(3, 9) & 0.0038 & -0.0036 & -0.0041 \\
(4, 5) & -0.0016 & 0.0069 & 0.0071 \\
(4, 6) & -0.0279 & -0.0295 & -0.0298 \\
(4, 7) & -0.0100 & -0.0070 & -0.0079 \\
(4, 8) & -0.0019 & -0.0037 & -0.0043 \\
(4, 9) & -0.0091 & -0.0116 & -0.0122 \\
(5, 6) & 0.0026 & 0.0083 & 0.0079 \\
(5, 7) & 0.0084 & 0.0152 & 0.0157 \\
(5, 8) & -0.0000 & 0.0055 & 0.0045 \\
(5, 9) & 0.0015 & 0.0044 & 0.0041 \\
(6, 7) & -0.0551 & -0.0603 & -0.0581 \\
(6, 8) & -0.0132 & -0.0174 & -0.0182 \\
(6, 9) & -0.0053 & -0.0140 & -0.0126 \\
(7, 8) & -0.0125 & -0.0102 & -0.0127 \\
(7, 9) & -0.0102 & -0.0151 & -0.0161 \\
(8, 9) & 0.0052 & 0.0014 & -0.0004 \\
(0, 1, 2) & -0.0058 & -0.0026 & -0.0014 \\
(0, 1, 3) & 0.0018 & 0.0028 & 0.0020 \\
(0, 1, 4) & -0.0000 & -0.0030 & -0.0021 \\
(0, 1, 5) & 0.0070 & -0.0005 & -0.0013 \\
(0, 1, 6) & 0.0060 & 0.0024 & 0.0030 \\
(0, 1, 7) & -0.0039 & -0.0024 & -0.0015 \\
(0, 1, 8) & 0.0073 & 0.0007 & 0.0014 \\
(0, 1, 9) & -0.0003 & -0.0006 & -0.0009 \\
(0, 2, 3) & 0.0038 & 0.0031 & 0.0030 \\
(0, 2, 4) & -0.0274 & -0.0141 & -0.0079 \\
(0, 2, 5) & 0.0088 & 0.0062 & 0.0081 \\
(0, 2, 6) & -0.0042 & 0.0006 & -0.0006 \\
(0, 2, 7) & 0.0233 & 0.0242 & 0.0275 \\
(0, 2, 8) & 0.0043 & 0.0023 & 0.0055 \\
(0, 2, 9) & -0.0298 & -0.0249 & -0.0216 \\
(0, 3, 4) & 0.0149 & 0.0078 & 0.0091 \\
(0, 3, 5) & 0.0019 & -0.0023 & -0.0014 \\
... & ... & ... & ... \\
\hline
\end{tabular}\hfill
\begin{tabular}{lccc}
\hline
Subset $S$ & N=500 & N=5000 & N=133 549 \\  
\hline\\[-6pt]
... & ... & ... & ... \\
(2, 7, 8, 9) & 0.0043 & 0.0009 & 0.0005 \\
(3, 4, 5, 6) & -0.0101 & -0.0143 & -0.0135 \\
(3, 4, 5, 7) & 0.0045 & -0.0030 & -0.0049 \\
(3, 4, 5, 8) & -0.0020 & -0.0053 & -0.0047 \\
(3, 4, 5, 9) & -0.0064 & -0.0049 & -0.0054 \\
(3, 4, 6, 7) & 0.0097 & 0.0076 & 0.0079 \\
(3, 4, 6, 8) & -0.0058 & -0.0058 & -0.0047 \\
(3, 4, 6, 9) & 0.0032 & 0.0018 & 0.0017 \\
(3, 4, 7, 8) & 0.0007 & -0.0011 & -0.0011 \\
(3, 4, 7, 9) & 0.0041 & 0.0003 & 0.0004 \\
(3, 4, 8, 9) & 0.0006 & 0.0013 & 0.0021 \\
(3, 5, 6, 7) & 0.0052 & 0.0059 & 0.0071 \\
(3, 5, 6, 8) & -0.0024 & -0.0011 & -0.0000 \\
(3, 5, 6, 9) & -0.0044 & -0.0023 & -0.0019 \\
(3, 5, 7, 8) & -0.0023 & -0.0014 & -0.0011 \\
(3, 5, 7, 9) & -0.0007 & -0.0031 & -0.0024 \\
(3, 5, 8, 9) & -0.0010 & -0.0007 & -0.0005 \\
(3, 6, 7, 8) & 0.0035 & 0.0027 & 0.0034 \\
(3, 6, 7, 9) & -0.0034 & -0.0052 & -0.0045 \\
(3, 6, 8, 9) & -0.0019 & -0.0011 & -0.0004 \\
(3, 7, 8, 9) & 0.0018 & 0.0014 & 0.0003 \\
(4, 5, 6, 7) & -0.0052 & -0.0020 & -0.0037 \\
(4, 5, 6, 8) & -0.0025 & -0.0001 & 0.0007 \\
(4, 5, 6, 9) & -0.0019 & 0.0005 & 0.0004 \\
(4, 5, 7, 8) & -0.0027 & 0.0009 & 0.0016 \\
(4, 5, 7, 9) & -0.0017 & 0.0005 & 0.0010 \\
(4, 5, 8, 9) & -0.0004 & -0.0004 & 0.0000 \\
(4, 6, 7, 8) & -0.0000 & -0.0003 & 0.0011 \\
(4, 6, 7, 9) & 0.0017 & 0.0005 & 0.0006 \\
(4, 6, 8, 9) & -0.0005 & -0.0005 & 0.0005 \\
(4, 7, 8, 9) & 0.0007 & -0.0000 & -0.0001 \\
(5, 6, 7, 8) & -0.0041 & -0.0024 & -0.0012 \\
(5, 6, 7, 9) & -0.0038 & -0.0046 & -0.0039 \\
(5, 6, 8, 9) & 0.0013 & -0.0009 & -0.0007 \\
(5, 7, 8, 9) & -0.0003 & 0.0003 & 0.0004 \\
(6, 7, 8, 9) & 0.0022 & 0.0005 & 0.0003 \\
(0, 1, 2, 3, 4) & 0.0042 & 0.0010 & 0.0003 \\
(0, 1, 2, 3, 5) & 0.0004 & 0.0010 & 0.0012 \\
(0, 1, 2, 3, 6) & 0.0018 & 0.0004 & 0.0002 \\
(0, 1, 2, 3, 7) & 0.0014 & 0.0006 & 0.0012 \\
(0, 1, 2, 3, 8) & 0.0007 & -0.0004 & -0.0001 \\
(0, 1, 2, 3, 9) & 0.0006 & 0.0012 & 0.0015 \\
(0, 1, 2, 4, 5) & 0.0051 & 0.0013 & 0.0013 \\
(0, 1, 2, 4, 6) & 0.0016 & 0.0011 & 0.0010 \\
(0, 1, 2, 4, 7) & 0.0005 & -0.0011 & -0.0009 \\
(0, 1, 2, 4, 8) & 0.0022 & 0.0005 & 0.0008 \\
(0, 1, 2, 4, 9) & 0.0026 & -0.0002 & -0.0000 \\
(0, 1, 2, 5, 6) & 0.0025 & 0.0023 & 0.0038 \\
(0, 1, 2, 5, 7) & 0.0013 & -0.0001 & 0.0003 \\
(0, 1, 2, 5, 8) & 0.0012 & 0.0011 & 0.0017 \\
(0, 1, 2, 5, 9) & 0.0017 & -0.0008 & -0.0006 \\
(0, 1, 2, 6, 7) & 0.0005 & -0.0010 & -0.0008 \\
(0, 1, 2, 6, 8) & -0.0003 & -0.0008 & -0.0001 \\
(0, 1, 2, 6, 9) & 0.0008 & -0.0001 & -0.0003 \\
(0, 1, 2, 7, 8) & 0.0003 & -0.0010 & -0.0005 \\
(0, 1, 2, 7, 9) & 0.0004 & -0.0008 & -0.0006 \\
(0, 1, 2, 8, 9) & -0.0004 & -0.0006 & -0.0003 \\
(0, 1, 3, 4, 5) & 0.0013 & 0.0002 & -0.0007 \\
(0, 1, 3, 4, 6) & 0.0017 & 0.0009 & -0.0001 \\
(0, 1, 3, 4, 7) & 0.0028 & 0.0010 & 0.0010 \\
(0, 1, 3, 4, 8) & 0.0032 & -0.0004 & -0.0000 \\
(0, 1, 3, 4, 9) & -0.0005 & -0.0006 & -0.0001 \\
(0, 1, 3, 5, 6) & -0.0012 & 0.0000 & 0.0003 \\
(0, 1, 3, 5, 7) & 0.0018 & -0.0006 & -0.0003 \\
(0, 1, 3, 5, 8) & 0.0003 & -0.0001 & 0.0000 \\
(0, 1, 3, 5, 9) & -0.0012 & -0.0000 & 0.0002 \\
(0, 1, 3, 6, 7) & 0.0011 & 0.0002 & 0.0009 \\
(0, 1, 3, 6, 8) & 0.0020 & -0.0004 & -0.0002 \\
(0, 1, 3, 6, 9) & -0.0000 & 0.0007 & 0.0004 \\
(0, 1, 3, 7, 8) & 0.0015 & -0.0003 & 0.0001 \\
(0, 1, 3, 7, 9) & -0.0003 & -0.0010 & -0.0004 \\
... & ... & ... & ... \\
\hline
\end{tabular}\hfill
\begin{tabular}{lccc}
\hline
Subset $S$ & N=500 & N=5000 & N=133 549 \\  
\hline\\[-6pt]
... & ... & ... & ... \\
(1, 2, 5, 6, 7, 8, 9) & 0.0029 & -0.0003 & -0.0009 \\
(1, 3, 4, 5, 6, 7, 8) & 0.0005 & -0.0032 & -0.0035 \\
(1, 3, 4, 5, 6, 7, 9) & 0.0061 & 0.0051 & 0.0049 \\
(1, 3, 4, 5, 6, 8, 9) & 0.0062 & 0.0014 & -0.0009 \\
(1, 3, 4, 5, 7, 8, 9) & 0.0002 & 0.0002 & 0.0009 \\
(1, 3, 4, 6, 7, 8, 9) & 0.0015 & 0.0015 & 0.0008 \\
(1, 3, 5, 6, 7, 8, 9) & 0.0002 & -0.0026 & -0.0004 \\
(1, 4, 5, 6, 7, 8, 9) & 0.0025 & 0.0026 & 0.0016 \\
(2, 3, 4, 5, 6, 7, 8) & -0.0038 & 0.0007 & -0.0002 \\
(2, 3, 4, 5, 6, 7, 9) & 0.0039 & 0.0042 & 0.0036 \\
(2, 3, 4, 5, 6, 8, 9) & 0.0059 & 0.0022 & 0.0013 \\
(2, 3, 4, 5, 7, 8, 9) & -0.0042 & -0.0016 & -0.0010 \\
(2, 3, 4, 6, 7, 8, 9) & -0.0007 & 0.0013 & 0.0008 \\
(2, 3, 5, 6, 7, 8, 9) & -0.0046 & -0.0029 & -0.0015 \\
(2, 4, 5, 6, 7, 8, 9) & 0.0012 & 0.0018 & 0.0008 \\
(3, 4, 5, 6, 7, 8, 9) & 0.0014 & 0.0009 & 0.0011 \\
(0, 1, 2, 3, 4, 5, 6, 7) & -0.0021 & -0.0027 & -0.0019 \\
(0, 1, 2, 3, 4, 5, 6, 8) & 0.0037 & 0.0021 & 0.0014 \\
(0, 1, 2, 3, 4, 5, 6, 9) & 0.0018 & -0.0006 & -0.0015 \\
(0, 1, 2, 3, 4, 5, 7, 8) & 0.0002 & 0.0002 & -0.0001 \\
(0, 1, 2, 3, 4, 5, 7, 9) & -0.0002 & -0.0006 & -0.0012 \\
(0, 1, 2, 3, 4, 5, 8, 9) & 0.0015 & 0.0018 & -0.0001 \\
(0, 1, 2, 3, 4, 6, 7, 8) & 0.0005 & 0.0010 & -0.0003 \\
(0, 1, 2, 3, 4, 6, 7, 9) & -0.0004 & 0.0013 & 0.0003 \\
(0, 1, 2, 3, 4, 6, 8, 9) & 0.0025 & 0.0014 & 0.0005 \\
(0, 1, 2, 3, 4, 7, 8, 9) & -0.0013 & 0.0001 & -0.0003 \\
(0, 1, 2, 3, 5, 6, 7, 8) & 0.0037 & 0.0016 & -0.0005 \\
(0, 1, 2, 3, 5, 6, 7, 9) & 0.0009 & 0.0008 & -0.0009 \\
(0, 1, 2, 3, 5, 6, 8, 9) & 0.0018 & 0.0009 & -0.0002 \\
(0, 1, 2, 3, 5, 7, 8, 9) & 0.0014 & 0.0010 & -0.0002 \\
(0, 1, 2, 3, 6, 7, 8, 9) & 0.0000 & 0.0006 & 0.0001 \\
(0, 1, 2, 4, 5, 6, 7, 8) & 0.0030 & 0.0017 & 0.0002 \\
(0, 1, 2, 4, 5, 6, 7, 9) & -0.0009 & -0.0002 & 0.0000 \\
(0, 1, 2, 4, 5, 6, 8, 9) & 0.0052 & 0.0014 & 0.0004 \\
(0, 1, 2, 4, 5, 7, 8, 9) & -0.0010 & 0.0006 & -0.0001 \\
(0, 1, 2, 4, 6, 7, 8, 9) & -0.0013 & 0.0003 & -0.0000 \\
(0, 1, 2, 5, 6, 7, 8, 9) & 0.0007 & 0.0003 & -0.0004 \\
(0, 1, 3, 4, 5, 6, 7, 8) & -0.0013 & 0.0008 & -0.0006 \\
(0, 1, 3, 4, 5, 6, 7, 9) & 0.0003 & 0.0017 & 0.0006 \\
(0, 1, 3, 4, 5, 6, 8, 9) & 0.0010 & 0.0005 & -0.0001 \\
(0, 1, 3, 4, 5, 7, 8, 9) & -0.0006 & 0.0007 & -0.0000 \\
(0, 1, 3, 4, 6, 7, 8, 9) & -0.0007 & 0.0005 & 0.0002 \\
(0, 1, 3, 5, 6, 7, 8, 9) & -0.0001 & 0.0008 & 0.0002 \\
(0, 1, 4, 5, 6, 7, 8, 9) & -0.0002 & 0.0010 & 0.0001 \\
(0, 2, 3, 4, 5, 6, 7, 8) & 0.0006 & 0.0001 & -0.0007 \\
(0, 2, 3, 4, 5, 6, 7, 9) & -0.0005 & 0.0015 & 0.0003 \\
(0, 2, 3, 4, 5, 6, 8, 9) & 0.0012 & 0.0004 & 0.0001 \\
(0, 2, 3, 4, 5, 7, 8, 9) & -0.0005 & 0.0002 & -0.0001 \\
(0, 2, 3, 4, 6, 7, 8, 9) & -0.0010 & 0.0002 & 0.0002 \\
(0, 2, 3, 5, 6, 7, 8, 9) & 0.0009 & 0.0004 & -0.0000 \\
(0, 2, 4, 5, 6, 7, 8, 9) & -0.0007 & 0.0007 & 0.0001 \\
(0, 3, 4, 5, 6, 7, 8, 9) & -0.0010 & 0.0008 & 0.0006 \\
(1, 2, 3, 4, 5, 6, 7, 8) & -0.0131 & -0.0081 & -0.0069 \\
(1, 2, 3, 4, 5, 6, 7, 9) & -0.0018 & 0.0002 & 0.0013 \\
(1, 2, 3, 4, 5, 6, 8, 9) & -0.0073 & -0.0006 & 0.0015 \\
(1, 2, 3, 4, 5, 7, 8, 9) & 0.0039 & 0.0040 & 0.0042 \\
(1, 2, 3, 4, 6, 7, 8, 9) & 0.0011 & 0.0000 & 0.0014 \\
(1, 2, 3, 5, 6, 7, 8, 9) & 0.0018 & 0.0036 & 0.0014 \\
(1, 2, 4, 5, 6, 7, 8, 9) & -0.0021 & -0.0014 & 0.0005 \\
(1, 3, 4, 5, 6, 7, 8, 9) & -0.0036 & -0.0048 & -0.0048 \\
(2, 3, 4, 5, 6, 7, 8, 9) & -0.0021 & -0.0039 & -0.0038 \\
(0, 1, 2, 3, 4, 5, 6, 7, 8) & -0.0023 & -0.0018 & -0.0002 \\
(0, 1, 2, 3, 4, 5, 6, 7, 9) & -0.0008 & -0.0013 & 0.0003 \\
(0, 1, 2, 3, 4, 5, 6, 8, 9) & -0.0063 & -0.0024 & -0.0003 \\
(0, 1, 2, 3, 4, 5, 7, 8, 9) & 0.0012 & 0.0000 & 0.0010 \\
(0, 1, 2, 3, 4, 6, 7, 8, 9) & 0.0012 & -0.0003 & 0.0002 \\
(0, 1, 2, 3, 5, 6, 7, 8, 9) & -0.0017 & -0.0008 & 0.0005 \\
(0, 1, 2, 4, 5, 6, 7, 8, 9) & 0.0003 & -0.0004 & 0.0004 \\
(0, 1, 3, 4, 5, 6, 7, 8, 9) & -0.0000 & -0.0019 & -0.0009 \\
(0, 2, 3, 4, 5, 6, 7, 8, 9) & -0.0003 & -0.0017 & -0.0010 \\
(1, 2, 3, 4, 5, 6, 7, 8, 9) & -0.0041 & -0.0008 & -0.0022 \\
(0, 1, 2, 3, 4, 5, 6, 7, 8, 9) & 0.0005 & 0.0012 & -0.0002 \\
\hline
\end{tabular}
    \caption{The individual terms of the Shapley-GAM decomposition of a kNN classifier on the Folktables Travel data set. The table depicts a number of selected terms of the full decomposition, estimated with 500, 5000 and  133549 samples per evaluation of the value function. The depicted terms are visualized in Figure \ref{fig:apx_estiation_visualization}. From the table, we see that many relatively small higher-order coefficients are not very precisely estimated for $N=5000$, whereas the overall sums (visualized in Figure \ref{fig:apx_estiation_visualization}) are.}
    \label{tab:estimation_table}
\end{table}

\end{document}